\documentclass[twoside,11pt]{amsart}

\input{style}
\makeatletter
\def\munderbar#1{\underline{\sbox\tw@{$#1$}\dp\tw@\z@\box\tw@}}
\makeatother

\newcommand{\be}{\begin{equation}}
\newcommand{\ee}{\end{equation}}
\newcommand{\bea}{\begin{equation*}\begin{aligned}}
\newcommand{\eea}{\end{aligned}\end{equation*}}
\newcommand{\ds}{\displaystyle}

\newcommand{\R}{\mathbb{R}}

\newcommand{\Max}{\max\limits_}
\newcommand{\Min}{\min\limits_}

\newcommand{\Inf}{\inf\limits_}
\newcommand{\Tr}[1]{\Trace \big[ #1 \big]}

\DeclareMathOperator*{\argmax}{arg\,max}
\DeclareMathOperator*{\argmin}{arg\,min}

\newcommand{\wh}{\widehat}
\newcommand{\mc}{\mathcal}
\newcommand{\mbb}{\mathbb}
\newcommand{\inner}[2]{\big \langle #1, #2 \big \rangle }
\newcommand{\cov}{\Sigma} 
\newcommand{\covsa}{\wh{\cov}}

\newcommand{\p}{\mbb P}
\newcommand{\PP}{\mbb P}

\newcommand{\ie}{{\it i.e.}}

\DeclareMathOperator{\Trace}{Tr}

\DeclareMathOperator{\sign}{sign}

\DeclareMathOperator{\st}{s.t.}

\newcommand{\PSD}{\mathbb{S}_{+}} 
\newcommand{\PD}{\mathbb{S}_{++}} 
\newcommand{\Let}{\triangleq}
\newcommand{\opt}{^\star}

\newcommand{\m}{\mu}
\newcommand{\msa}{\wh \mu}

\newcommand{\half}{\frac{1}{2}}
\newcommand{\dualvar}{\gamma}

\newcommand{\w}{w}

\newcommand{\Pnom}{\wh \p}

\newcommand{\FR}{\mathds{F}}

\usepackage{setspace}
\begin{document}

\title{Coverage-Validity-Aware Algorithmic Recourse}

\author{Ngoc Bui, Duy Nguyen, Man-Chung Yue, Viet Anh Nguyen}
\thanks{The authors are with Yale University (\texttt{ngocbh.pt@gmail.com}), The University of North Carolina, Chapel Hill \texttt{duykhuongnguyen277@gmail.com}), The University of Hong Kong (\texttt{mcyue@hku.hk}), and The Chinese University of Hong Kong (\texttt{nguyen@se.cuhk.edu.hk}).}

\maketitle

\begin{abstract}
  Algorithmic recourse emerges as a prominent technique to promote the explainability, transparency, and ethics of machine learning models. Existing algorithmic recourse approaches often assume an invariant predictive model; however, the predictive model is usually updated upon the arrival of new data. Thus, a recourse that is valid respective to the present model may become \textit{in}valid for the future model. To resolve this issue, we propose a novel framework to generate a model-agnostic recourse that exhibits robustness to model shifts. Our framework first builds a coverage-validity-aware linear surrogate of the \textit{non}linear (black-box) model; then, the recourse is generated with respect to the linear surrogate. We establish a theoretical connection between our coverage-validity-aware linear surrogate and the minimax probability machines (MPM). We then prove that by prescribing different covariance robustness, the proposed framework recovers popular regularizations for MPM, including the $\ell_2$-regularization and class-reweighting. Furthermore, we show that our surrogate pushes the approximate hyperplane intuitively, facilitating not only robust but also interpretable recourses. The numerical results demonstrate the usefulness and robustness of our framework. 
\end{abstract}

\section{Introduction} \label{sec:intro}

The recent prevalence of machine learning (ML) in the automation of consequential decisions related to humans, such as loan approval~\cite{ref:moscato2021benchmark}, job hiring~\cite{ref:cohen2019efficient, ref:schumann2020we}, and criminal justice~\cite{ref:brayne2021technologies}, urges the need of transparent ML systems with explanations and feedback to users~\cite{ref:doshi2017towards, ref:miller2019explanation}. Algorithmic recourse~\cite{ref:ustun2019actionable} is an emerging approach for generating feedback in ML systems. A \textit{recourse} suggests how the input instance should be modified to alter the outcome of a predictive model. Consider a specific scenario in which a financial institution's ML model rejects an individual's loan application. It has now become a legal necessity to provide explanations and recommendations to individuals to improve their situation and obtain a loan in the future~(GDPR, \cite{ref:voigt2017eu}). For example, the recommendation can be ``increase the income to \$5000'' or ``reduce the debt/asset ratio to below 20\%''. These recommendations empower individuals to understand the factors influencing their loan applications and take specific actions to address them. It also promotes transparency and fairness in the decision-making process, incentivizing users to improve their loan eligibility.

Various techniques were proposed to devise algorithmic recourses for a given predictive model; extensive surveys are provided in~\cite{ref:karimi2020survey, ref:stepin2021survey, ref:pawelczyk2021carla, ref:verma2020counterfactual}. \cite{ref:wachter2017counterfactual} introduced the definition of counterfactual explanations and proposed a gradient-based approach to finding the nearest instance that yields a favorable outcome. \cite{ref:ustun2019actionable} proposed a mixed integer programming formulation (AR) that can find recourses for a linear classifier with a flexible design of the actionability constraints. 
In \cite{ref:karimi2021algorithmic, ref:karimi2020algorithmic}, recourses that can accommodate causal relationships between features are investigated through the lens of minimal intervention.
Recent works, including~\cite{ref:russell2019efficient} and~\cite{ref:mothilal2020explaining}, studied the problem of generating a menu of diverse recourses to provide multiple possibilities that users might choose.

While these methods offer actionable recourses with low implementation costs, they face a critical downside regarding robustness in a real-world deployment. The recourse literature identifies three main types of robustness that are desirable: (i) robustness to changes in the input data point, which ensures consistent and comparable solutions for users with similar characteristics or needs~\cite{ref:slack2021counterfactual}, (ii) robustness to changes in attained recourses, which requires that the recommended actions to be stable with respect to minor variations in implementation~\cite{ref:pawelczyk2022probabilistically, ref:maragno2023finding, ref:dominguez2022adversarial, ref:virgolin2023robustness}, (iii) robustness to model shifts, which are often caused by model retraining and updating~\cite{ref:hamman2023robust, ref:zhang2023density, ref:guo2022rocoursenet, ref:nguyen2022robust, ref:pawelczyk2020counterfactual, ref:black2021consistent, ref:upadhyay2021towards}.

Motivated by the fact that ML models are frequently retrained or recalibrated upon the arrival of new data, this paper focuses on the robustness with respect to model shifts. Because of the model shift, a recourse deemed valid for the current model may become \textit{in}valid for the future model. For example, a second-time loan applicant who has been rejected in his first attempt and spent tremendous effort to achieve the recourse suggested by the system could still be rejected (again) simply because of model shifts; see Figure~\ref{fig:recourse} for an illustration. Such unfortunate events can result in inefficiency of the overall ML system, waste of resources and effort of the user, and distrust in ML systems in our society~\cite{ref:rudin2019stop}.

\begin{figure}
\centering
\includegraphics[width=6cm]{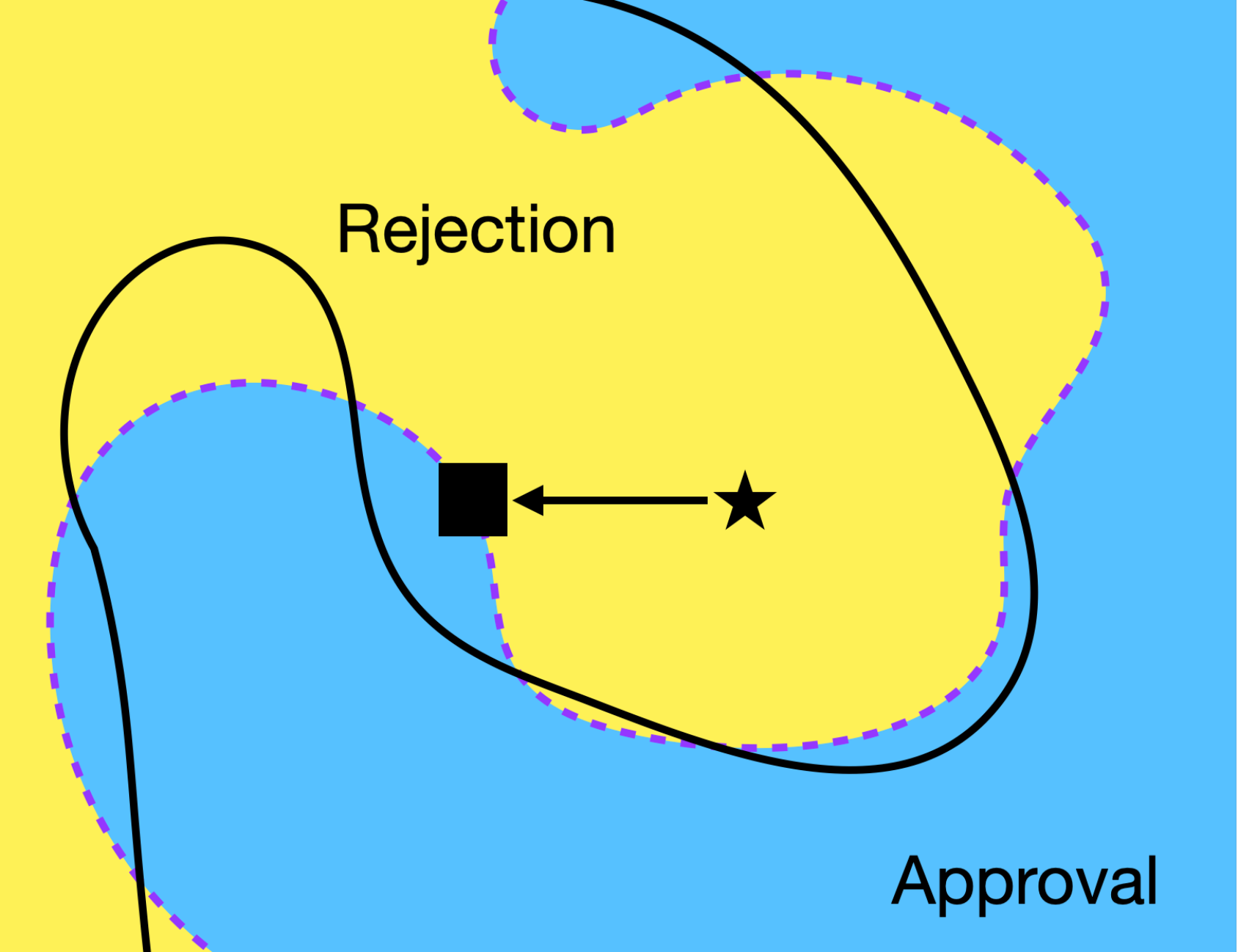}
\caption{An example of recourse failures under model shifts. The present model's boundary is plotted in a dashed curve, separating the input space into the rejection (yellow) and approval (blue) regions. The original input (star) lies in the rejection region of the present model. The recommended recourse (square symbol) is in the approval region of the present model. However, as the boundary shifts to the solid line, the recourse falls into the rejection region under the new boundary.}
\label{fig:recourse}
\end{figure}

Studying this phenomenon, \cite{ref:rawal2020can} first described several types of model shifts related to the correction, temporal, and geospatial shifts from data. They pointed out that even recourses constructed with state-of-the-art algorithms can be vulnerable to shifts in the model's parameters. \cite{ref:pawelczyk2020counterfactual} studied counterfactual explanations under predictive multiplicity and its relation to the difference in how two classifiers treat predicted individuals. \cite{ref:black2021consistent} showed that the constructed recourses might be invalid even for the model retrained with different initial conditions, such as weight initialization and leave-one-out variations in data. 

To generate recourses that are robust to model shifts, \cite{ref:upadhyay2021towards} leveraged robust optimization to propose ROAR - wherein the parameter shifts are captured by an uncertainty set centered around the current parameter value of a linear surrogate. Methods using distributionally robust optimization have also been proposed, which capture the shifts of the parameters using probability distributions~\cite{ref:bui2022counterfactual,ref:nguyen2023distributionally}. This line of research formulates the design of recourse as a min-max problem, and the goal is to find a recourse that minimizes the loss function associated with the worst model parameters over the pre-specified range. Herein, the loss function can be an aggregation of different terms representing the implementation cost and the validity of the solution, among others. The min-max approach to recourse design is usually criticized for its over-conservativeness.
In particular, whichever recourse is chosen, the corresponding worst shift parameter will be realized.
Consequently, although the min-max formulation can deliver a recourse with high validity, the cost of implementing the recourse can be unrealistically high. 
Furthermore, compared with minimization or maximization problems, min-max problems are a more difficult class of computational problems. 
Not only are algorithmic design and analysis much more involved, but computational complexity and ease of implementation are also worse than minimization/maximization problems in general. Finally, one may want to impose additional constraints on a mathematical formulation of the recourse design problem for fairness, explainability, or any other practical, legal, or ethical considerations. These constraints are sometimes of mixed-integer type. Unfortunately, it is highly nontrivial to extend existing theory and algorithms of min-max formulations of the recourse design problem to allow extra constraints, especially ones involving integer variables.

These aforementioned methods all share the linear classifier setting, which is also a popular choice employed by earlier work on algorithmic recourse~\cite{ref:ustun2019actionable, ref:russell2019efficient, ref:rawal2020can}. For non-linear classifiers, a linear surrogate method is used in the preprocessing step to locally approximate the \textit{non}linear decision boundary of the black-box classifiers. The recourse is then generated subject to the linear surrogate instead of the nonlinear model. The most popular choice to construct the surrogate is Local Interpretable Model-Agnostic Explanations (LIME, \cite{ref:ribeiro2016why}), a well-known method to explain ML predictions by fitting a reweighted linear regression model to the perturbed samples around an input instance. Arguably, LIME is the most common linear surrogate for the nonlinear decision boundary of the black-box models~\cite{ref:ustun2019actionable, ref:upadhyay2021towards}. Unfortunately, the LIME surrogate might not be locally faithful to the underlying model due to its weighted sampling scheme~\cite{ref:laugel2018defining, ref:white2019measurable}. 
Furthermore, it is also well-known that perturbation-based surrogates are sensitive to the original input and perturbations~\cite{ref:alvarez2018on, ref:ghorbani2019interpretation, ref:slack2020much, ref:slack2021reliable, ref:agarwal2021towards, ref:laugel2018defining}.
Several approaches are proposed to overcome these limitations. \cite{ref:laugel2018defining} and~\cite{ref:vlassopoulos2020explaining} proposed alternative sampling procedures that generate sample instances in the neighborhood of the closest counterfactual to fit a local surrogate. \cite{ref:white2019measurable} integrated counterfactual explanation to local surrogate models to introduce a novel fidelity measure of an explanation. Later, \cite{ref:garreau2020looking} and \cite{ref:agarwal2021towards} analyzed theoretically the sensitivity\footnote{Throughout, ``robustness'' is used in the algorithmic recourse setting with respect to the model shifts~\cite{ref:rawal2020can}. ``Robustness'' is also used in the literature to indicate the stability of LIME to the sampling distribution. To avoid confusion, in what follows, we use ``sensitivity'' to refer to the aforementioned stability of LIME.} of LIME, especially in the low sampling size regime. \cite{ref:zhao2021baylime} leveraged Bayesian reasoning to improve the consistency in repeated explanations of a single prediction. However, the influence and efficacy of these surrogate models on generating recourse options remain elusive, especially when the situation is further complicated by model shifts.

\subsection{Robust Surrogates for Actionable Recourse Generation} \label{sec:robust}

We propose a new framework for robust actionable recourse design: the core ingredient of our framework is the linear surrogate, which is known to be robust to the model shifts. The upshot is that our robust recourse generation framework does not need to solve a min-max problem. Hence, our method does not suffer from over-conservatism nor unfavorable complexity/scalability as in the (distributionally) robust recourse formulations. Moreover, it is amenable to additional constraints, even mix-integer ones, which allow us to impose fairness, legal, causality, explainability, ethics, or any other constraints easily.

To introduce our method, we assume that the covariate (feature) space is $\mc X  = \R^d$, and we have a binary label space $\mc Y = \{-1, +1\}$. Without any loss of generality, we assume that label $-1$ encodes the unfavorable decision, while $+1$ encodes the favorable one. Given a black-box ML classifier $f$ and input $x_0$ with an unfavorable predicted outcome, we aim to find an actionable recourse recommendation for $x_0$ that has a high probability of being classified into a favorable group, subject to possible shifts of the ML classifier $f$. Such recourse is termed a robust actionable recourse. Figure~\ref{fig:flow} provides a schematic view of our recourse generator consisting of three components:
\begin{figure*}[!t]
    \centering
    \includegraphics[width=\textwidth]{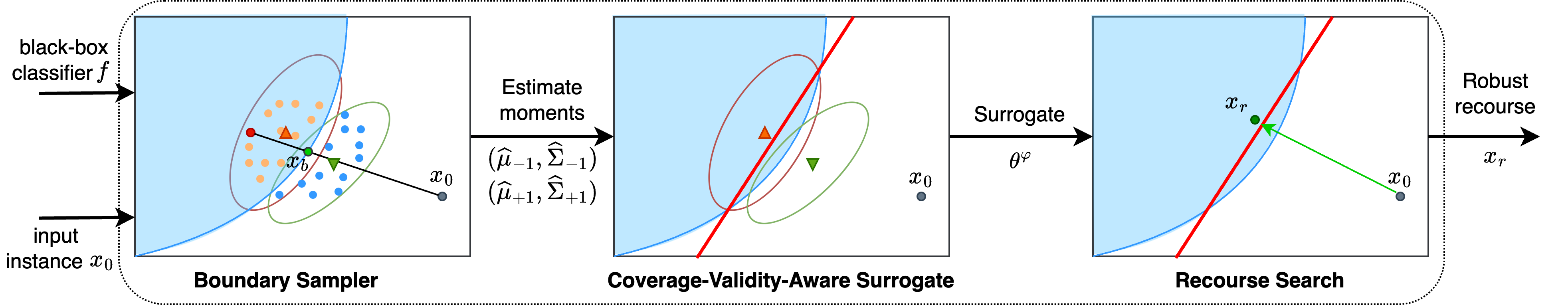}
    \caption{The sampler synthesizes new instances around $x_0$ and queries the predicted labels from the classifier $f$ . The moment information $(\msa_y, \covsa_y)$ estimated from the synthetic pseudo-labeled data (represented by triangles and ellipsoids) serves as inputs to find the (covariance-robust) Coverage-Validity-Aware Surrogate. The surrogate $\theta^\varphi$ (red hyperplane) is the target classifier to generate recourses (red circle).}
    \label{fig:flow}
\end{figure*}

\begin{enumerate}[label=(\roman*), leftmargin = 5mm]
    \item a local sampler: we use a sampling method as in~\cite{ref:laugel2018defining, ref:vlassopoulos2020explaining} to locally approximate the decision boundary. Given an instance $x_0$, we choose $k$ prototypes $x_{1}, \ldots, x_k$ \textit{in the available data} that have smallest $\ell_1$ distances and are predicted to be in the opposite class to $x_0$. For each $x_{i}$, we perform a line search to find a point $x_{b,i}$ that is on the decision boundary and the line segment between $x_{0}$ and $x_{i}$. Among $x_{b,i}$, we choose the nearest point $x_{b}$ to $x_{0}$, and then generate $n_p$ synthetic samples uniformly distributed in the $\ell_2$-ball with radius $r_p$ centered at $x_{b}$. We then query the black-box ML classifier $f$ to obtain the predicted labels of the synthetic samples. This procedure outputs two sets $\mathcal D_{+1}$ and $\mc D_{-1}$ of synthetic samples with predicted labels $+1$ and $-1$, respectively.
    \item a linear surrogate that explicitly balances the coverage-validity trade-off: we propose two performance metrics for the surrogate: coverage and validity, that are motivated by the popular recall-precision metrics in the machine learning literature. We then use the synthesized samples to estimate the moment information $(\msa_y, \covsa_y)$ of the covariate conditional on each predicted class $y$, and then train a coverage-validity-aware linear surrogate to approximate the local decision boundary of the ML model.
    \item a recourse search: in principle, we can integrate our coverage-validity-aware surrogate with any robust recourse search method to find a \textit{robust} recourse. This paper uses two simplest recourse search methods: a simple projection and AR~\cite{ref:ustun2019actionable}, a MIP-based framework, to promote \textit{actionable} recourses.
\end{enumerate}

\begin{figure}[!t]
    \centering
    \includegraphics[width=0.8\linewidth]{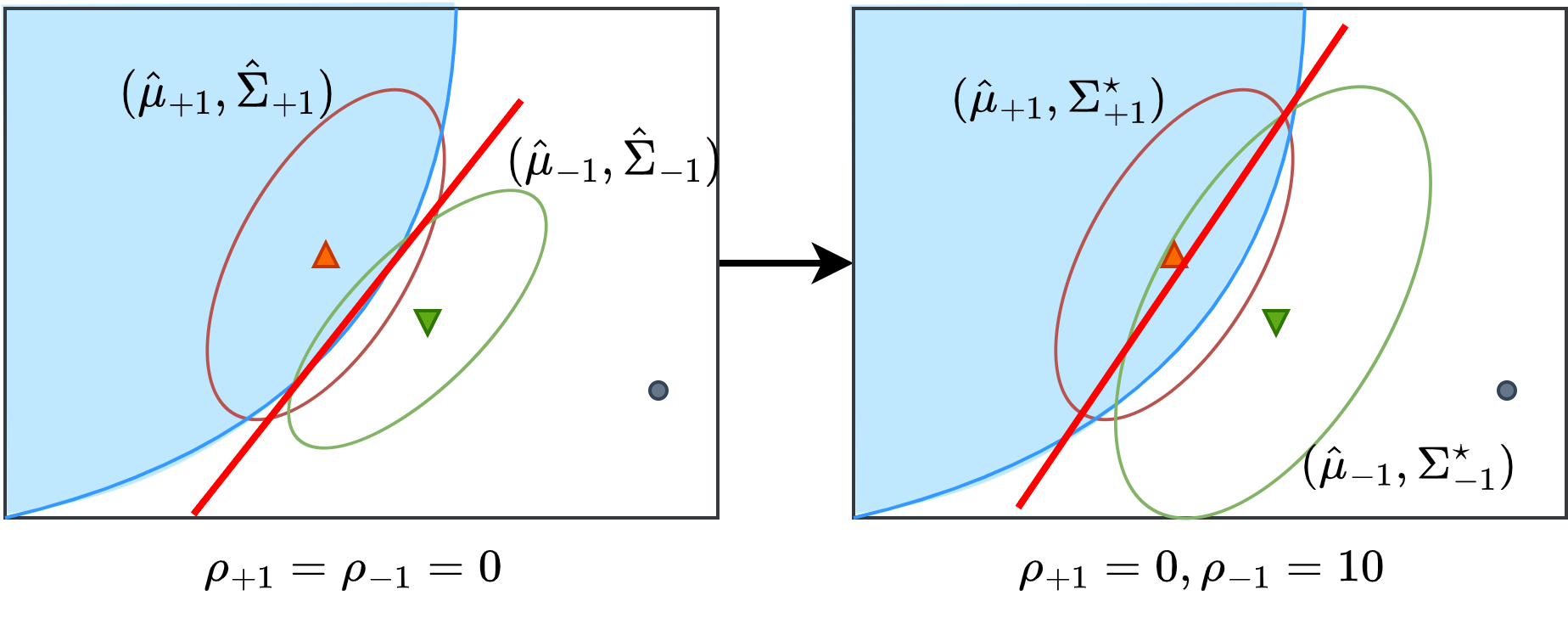}
    \caption{An intuitive explanation of the robustification mechanism. Left: CVAS, right: covariance-robust CVAS with $\rho_{-1} > \rho_{+1}$. By increasing the radius $\rho_{-1}$, the worst-case covariance matrix of the class $-1$ is inflated (bigger green ellipsoid) and shifts the surrogate boundary towards the class $+1$. The projection of the input $x_0$ onto the hyperplane will tend to lie deeper into the favorable region and may become more robust to model shifts.}
    \label{fig:robustify}
\end{figure}

One of the novelties of our framework is the possibility of \textit{shifting} the linear coverage-validity-aware surrogate towards the region of the favorable class, which induces robust recourse with respect to model shifts in a geometrically intuitive manner (see Figure~\ref{fig:robustify} for an illustration). There is a clear distinction between our framework and the existing method of ROAR~\cite{ref:upadhyay2021towards} and DiRRAc~\cite{ref:nguyen2023distributionally}: ROAR and DiRRAc uses a non-robust surrogate in Step (ii) and then formulates a min-max optimization problem in Step (iii) for recourse search. On the contrary, our framework uses a \textit{robust} surrogate in Step (ii) and then employs a simple recourse search in Step (iii). Using our framework, we can leverage mixed-integer formulations in Step (iii) (e.g., AR~\cite{ref:ustun2019actionable}) to generate more realistic or actionable, robust recourses. On the contrary, it is unclear how mixed-integer constraints can be integrated into the ROAR or DiRRAc's min-max formulation.

Robust recourse generation is characterized by a fundamental trade-off between the validity of the recourse $x_r$ under future model shifts and the cost of recourse implementation, measured by the magnitude of the feature difference $x_r - x_0$. Our framework controls this trade-off by altering the ambiguity size parameters $\rho_{+1}$ and $\rho_{-1}$. Traditional pipelines often tune these parameters using cross-validation and neglect the human side of the problem.
Real-world applications of algorithmic recourse, on the other hand, have a direct consequence for human beings or our society and are much more complex.
Thus, there is no single set of optimal parameters that can fit all subjects of the problem. Instead, we recommend intensive benchmarking to probe the trade-off, as we will demonstrate in the numerical experiments in Section~\ref{sec:numerical}, and the parameter selection process should be human-centric.

\subsection{Contributions} 
Motivated by the need to generate \textit{robust} recourses under various practical considerations such as fairness, explainability, ethics, or causality, we propose a novel coverage-validity-aware linear surrogate of the nonlinear boundary of the black-box classifier. 
Based on ideas from distributionally robust optimization, we develop a variant of the covariance-robust coverage-validity-aware surrogate that can promote robustness against model shifts in the recourse generation task. Unlike typical robust recourse generation schemes in the literature, our scheme does not require solving min-max problems, and it allows the incorporation of mixed-integer constraints for modeling various practical considerations.
Our main contributions can be summarized as follows:
\begin{itemize}[leftmargin = 3mm]
\item Coverage-validity-aware surrogate (CVAS). We propose this new linear surrogate approximating the classifier's decision boundary by balancing the coverage and validity quantities. We establish a connection between CVAS and the minimax probability machines proposed by~\cite{ref:lanckriet2001minimax} (Proposition~\ref{prop:MPM}). This connection enables us to identify the surrogate efficiently using second-order cone optimization (Lemma~\ref{lemma:optimal}).
\item Covariance-robust CVAS variants: To hedge against potential model shifts, we propose and analyze several robust variants of the CVAS obtained by perturbing the second-order moment information of the synthetic samples around the local decision boundary. We show that the covariance-robust CVAS admits a general form of reformulation (Proposition~\ref{prop:refor}). Further, the covariance-robust can be motivated in the parametric Gaussian subspace thanks to the equivalence of the solution between the nonparametric and the parametric Gaussian formulations (Proposition~\ref{prop:gauss}). 
\item Covariance-robustness induces regularization. Interestingly, we show that the covariance-robust variants are equivalent to \textit{un}discovered regularizations of the MPM. Specifically, using the Bures distance to prescribe the ambiguity set will recover the $\ell_2$-regularization scheme (Theorem~\ref{thm:bures}). If the ambiguity set is dictated by the Fisher-Rao or LogDet distance, we recover the class reweighting schemes (Theorem~\ref{thm:fr} and~\ref{thm:logdet}).
\item Robust recourse generation: We show that our covariance-robust CVAS can be integrated into the robust recourse generation process. Among multiple variants of the surrogate in this paper, we show that the Fisher-Rao or LogDet surrogates exhibit a more intuitive and interpretable boundary shift in a certain asymptotic sense (Proposition~\ref{prop:tradeoff}). This shift aligns better with the coverage-validity trade-off we desire to observe for the recourse generation task; thus, the Fisher-Rao or LogDet CVAS are reasonable surrogates for the recourse generation task. Formulations of the recourse generation problem using CVAS are provided in Section~\ref{sec:numerical}.
\end{itemize}

This paper unfolds as follows. Section~\ref{sec:trade-off} and~\ref{sec:modelshift} dive deeper into the construction of the CVAS and its robustification to hedge against model shifts. Section~\ref{sec:examples} presents and compares several variants of covariance-robust surrogates using the Bures, Fisher-Rao, and LogDet distance on the space of covariance matrices. Section~\ref{sec:asymptotic} studies the asymptotic surrogate as one of the radii grows to infinity, along with the implications on the robust recourse generation task. In Section~\ref{sec:numerical}, we demonstrate empirically that our proposed surrogates provide a competitive approximation of the local decision boundary and improve the robustness of the recourse subject to model shifts. All proofs are relegated to the appendix.

\section{Coverage-Validity-Aware Linear Surrogate} \label{sec:trade-off}

We present the intuition supporting our construction of the coverage-validity-aware surrogate (CVAS). We aim to construct a linear surrogate that approximates the local decision boundary of the black-box model and enables the robust generation of recourses in the latter phase. A linear surrogate is parametrized by $\theta = (w, b) \in \R^{d+1}$, $w \neq 0$ with a decision rule 
\[
    \mc C_\theta(x) = \sign(w^\top x - b),
\] 
where $w$ is the slope and $b$ is the intercept.

To assess the performance of $\mathcal C_\theta$, let us now briefly review two fundamental classification metrics in statistical machine learning: recall and precision. The recall of a classification system is the fraction of true positives over the total number of samples that belong to the positive class. At the same time, precision is the fraction of true positives over the total number of samples \emph{labeled} as positive. In the context of algorithmic recourse, the true labels are the labels given by the black-box model, while the predicted labels are the labels predicted by the surrogate $\mathcal C_\theta$. These two metrics are sample-based: they are computed by counting the number of samples in each block of the confusion matrix, including the true positive, false positive, and false negative. Thus, finding a surrogate that maximizes a joint measure of recall and precision requires solving a discrete optimization problem and is unscalable.

We overcome this challenge by taking a more geometric approach.
Specifically, to construct a surrogate, we propose two approximations of the precision and recall that are no longer sample-based. These approximations rely instead on the aggregated moment information of the samples and thus alleviate the discrete nature of the measurement process. The moment-based approximations also facilitate the robustification in subsequent steps. Given two groups of synthesized samples $\mathcal D_y$ of respective predicted label $y \in \mc Y$, we estimate the moment information  $(\msa_{y}, \covsa_{y})$ on the feature space for each group, for which a reasonable estimator is the unbiased sample average $\msa_y = (|\mc D_y|)^{-1} \sum_{x \in \mc D_y} x$ and $\covsa_y = (|\mc D_y| - 1)^{-1} \sum_{x \in \mc D_y} (x - \msa_y) (x - \msa_y)^\top$ for $y \in \mc Y$.

Figure~\ref{fig:cases} illustrates four possibilities of the relative position of the surrogate $\mathcal C_\theta$ to the data clusters of the two classes. Each data cluster is visualized by an ellipsoid constructed from the mean $\msa_y$ and the covariance matrix $\covsa_y$. The red line depicts the set $\mbb H_\theta = \{x \in \R^d: w^\top x - b = 0\}$, termed the decision boundary surrogate: it contains all inputs that lie on the surrogate hyperplane $w^\top x = b$. The arrow associated with the line shows the direction toward the positively-predicted halfspace. The halfspace $\mbb H_{\theta, +1}=  \{x \in \R^d: \mc C_\theta(x) = + 1\}$ on the same side with the arrow contains all inputs which will be predicted positively by the surrogate. Similarly, $\mbb H_{\theta, -1} =  \{x \in \R^d: \mc C_\theta(x) = - 1\}$ is the negatively-predicted halfspace for the surrogate. If we use $\theta$ as a surrogate, any point $x \in \mbb H_{\theta, +1}$ will be considered a feasible recourse under the surrogate's prediction. 
\begin{figure*}[!t]
    \centering
    \includegraphics[width=0.8\linewidth]{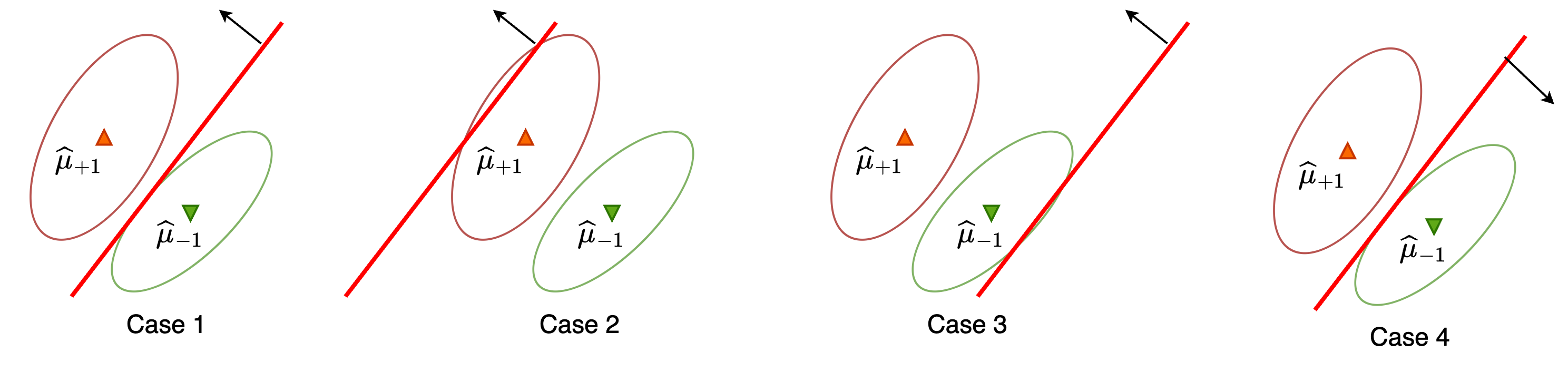}
    \caption{Four possibilities of the relative position of the surrogate to the data clusters. The black arrow indicates the normal vector of the hyperplane, pointing toward the region associated with the positive label. In case 1, the two mean vectors are classified correctly by the hyperplane. In cases 2 and 3, only one of the mean vectors is classified correctly. Case 4 is trivial since no mean vectors are classified correctly, leading to both trivial coverage and validity $\mathrm{Co}_{\covsa_{+1}} = \mathrm{Va}_{\covsa_{-1}} = 0$.}
    \label{fig:cases}
\end{figure*}

Next, we define and quantify two new moment-based metrics for the surrogate $\mathcal C_\theta$.

\textbf{Coverage.} The coverage is our moment-based approximation of the recall: it represents the overlap between the cluster of positive samples $\mathcal D_{+1}$ and the positively-predicted halfspace $\mbb H_{\theta, +1}$. To measure the coverage, we use the Mahalanobis distance from the mean vector of the positive cluster $\msa_{+1}$ to the decision boundary $\mbb H_{\theta}$. Specifically, the coverage is defined as
\begin{align} \label{eq:coverage}
    \mathrm{Co}_{\covsa_{+1}}(\theta) \Let 
            \inf_{x \in \mbb H_{\theta, -1}}~\sqrt{(\msa_{+1} - x)^\top \covsa_{+1}^{-1} (\msa_{+1} - x)},
\end{align}
where the Mahalanobis distance is weighted using the data-driven covariance matrix $\covsa_{+1}$. It is worth noting that we utilize the set $\mbb H_{\theta, -1}$ instead of $\mbb H_{\theta}$ to penalize the case where the surrogate incorrectly predicts the positive center $\msa_{+1}$ (case 2 and 4 of Figure~\ref{fig:cases}). In cases 1 and 3 of Figure~\ref{fig:cases}, the Mahalanobis distance from the mean of the positive data cluster to the decision boundary $\mbb H_\theta$ coincides with that to the set $\mbb H_{\theta, -1}$ (see Lemma~\ref{lemma:maha-proj} for proof).  A high coverage suggests that the positively predicted halfspace $\mbb H_{\theta,+1}$  can cover a significant portion of the feasible recourses of the black-box model. 

\textbf{Validity.} Validity is our moment-based approximation of the precision: it represents the overlapping between the cluster of negative samples $\mathcal D_{-1}$ and the positively-predicted halfspace $\mbb H_{\theta, +1}$. Specifically, the validity is defined as
\begin{align} \label{eq:validity}
    \mathrm{Va}_{\covsa_{-1}}(\theta) \Let 
            \inf_{x \in \mbb H_{\theta, +1}}~\sqrt{(\msa_{-1} - x)^\top \covsa_{-1}^{-1} (\msa_{-1} - x)},
\end{align}
where $\covsa_{-1}$ is the estimated covariance matrix of the negative samples. A small value of $\mathrm{Va}_{\covsa_{-1}}(\theta)$ indicates that the positively-predicted halfspace $\mbb H_{\theta, +1}$ is close to the cluster of negative samples, implying that the surrogate $\mathcal C_\theta$ may have a high false discovery rate, or equivalently, a low precision. In the opposite direction, a high value $\mathrm{Va}_{\covsa_{-1}}(\theta)$ implies that $\mathcal C_\theta$ may have a high precision.

Of course, the coverage and validity are meaningless if they are used in isolation. We can also relate the extreme behavior of the coverage-validity trade-off to the cases of Figure~\ref{fig:cases}. The surrogate plotted in case 2 has $\mathrm{Co}_{\covsa_{+1}}(\theta) = 0$ because $\msa_{+1} \in \mbb H_{\theta, -1}$, but it has high validity. Similarly, the surrogate in case 3 has  $\mathrm{Va}_{\covsa_{-1}}(\theta)  = 0$ because $\msa_{-1} \in \mbb H_{\theta, +1}$, but it has high coverage. Coverage and validity are thus conflicting criteria, and a surrogate with nontrivial coverage-validity trade-off should separate the mean vectors, as in case 1 of Figure~\ref{fig:cases}. Under the assumption that $\msa_{-1}\neq \msa_{+1}$, the separating hyperplane theorem asserts that a linear hyperplane that can separate the mean vectors exists. In view of the above discussion, we propose the following maximin problem formulation for the linear surrogate:
\be \label{eq:surrogate}
    \Max{\theta \in \Theta}~\min\left\{\mathrm{Co}_{\covsa_{+1}}(\theta),~ \mathrm{Va}_{\covsa_{-1}}(\theta)\right\}.
\ee
where $\Theta$ is the set of admissible surrogates defined by
\be \label{eq:Theta}
    \Theta \Let \{ \theta = (w, b) \in \R^{d+1}: w \neq 0\},
\ee
in which the constraint $w \neq 0$ eliminates trivial surrogates. Note that the inner minimization aims to balance the coverage and validity of the surrogate; this design of the maximin objective function is reasonable since both metrics are intended to be of similar magnitude. Alternatively, one may consider a reweighted objective function by injecting a positive weight parameter $\alpha$ and solve
\be \notag
    \Max{\theta \in \Theta}~\min\left\{\alpha \mathrm{Co}_{\covsa_{+1}}(\theta),~ \mathrm{Va}_{\covsa_{-1}}(\theta)\right\}.
\ee
Nevertheless, the above reweighted objective function can be recovered from~\eqref{eq:surrogate} by inflating $\covsa_{+1}$ with a corresponding coefficient $\alpha^2$. We, thus, do not proceed with the reweighted formulation in the remainder of this paper.

\subsection{Relationship with Minimax Probability Machines}
\label{sec:MPM}

Minimax Probability Machines (MPM) is a binary classification framework pioneered by~\cite{ref:lanckriet2001minimax} and extended to Quadratic MPM in~\cite{ref:lanckriet2003robust}. For each class $y \in \mc Y$, MPM does not assume the specific parametric form of the (conditional) distribution $\Pnom_y$ of $X | f(X) = y$. Instead, MPM assumes that we can identify $\Pnom_y$ only up to the first two moments, \ie, $\Pnom_y$ has mean vector $\msa_y \in \R^d$ and covariance matrix $\covsa_y \in \PSD^d$, denoted $\Pnom_y \sim (\msa_y, \covsa_y)$. These moments can be estimated from the samples synthesized from the local sampler, such as the unbiased sample average estimators, as suggested. MPM aims to find a (non-trivial) linear surrogate that minimizes the maximum misclassification rate among classes. MPM solves the min-max optimization problem
    \be \label{eq:mpm}
         \Min{\theta \in \Theta}~\Max{y \in \mc Y}~\Max{\Pnom_y\sim (\msa_y, \covsa_y)}\Pnom_y(\mc C_\theta(X) \neq y), 
    \ee
where the set $\Theta$ is defined as in~\eqref{eq:Theta}.

\begin{proposition}[Equivalence characterization] \label{prop:MPM}
    The CVAS obtained by solving~\eqref{eq:surrogate} coincides with the MPM obtained by solving~\eqref{eq:mpm}.
\end{proposition}
Proposition~\ref{prop:MPM} establishes a link between our CVAS and the MPM~\cite{ref:lanckriet2001minimax}. It shows that our surrogate also minimizes the misclassification rate, similar to MPM, by maximizing the lower bound of the coverage and validity quantities. Maximizing the coverage corresponds to minimizing the misclassification rate within the positive class. Meanwhile, the maximization of validity is equivalent to minimizing the misclassification probability within the negative class. The proof of is inspired by \cite[Section~2]{ref:lanckriet2001minimax} and relies on a technical result from~\cite{ref:bertsimas2000moment}. Nevertheless, the analysis in \cite[Section~2]{ref:lanckriet2001minimax} is not phrased in terms of coverage and validity.

\subsection{Solution Procedure} \label{sec:solution}

The main instrument to find the coverage-validity-aware surrogate is the result from~\cite[\S2]{ref:lanckriet2001minimax}, which provides the optimal solution for the MPM problem. Then, we can leverage the equivalence result from Proposition~\ref{prop:MPM} to find the surrogate. Define the set of feasible slopes $\mc W = \{ w \in \R^d \backslash\{0\} : \sum_{y \in \mc Y} y w^\top \msa_y  = 1\}$,
which is a hyperplane in $\R^d$. The reason for the extra constraint $\sum_{y \in \mc Y} y w^\top \msa_y  = 1$ is that for any $w\in\mathbb{R}^d\setminus\{0\}$, $b\in\mathbb{R}$ and $t \neq 0$, the two linear surrogates $\theta = (w, b)$ and $t \theta = (tw, tb)$ are equivalent to each other in the sense that they define the same hyperplane and yield the same objective value for problem~\eqref{eq:surrogate}. The next lemma offers an efficient procedure to compute the surrogate. 
\begin{lemma}[{Optimal solution, adapted from~\cite[\S2]{ref:lanckriet2001minimax}}] \label{lemma:optimal}
    Let $\wh w$ be an optimal solution to the second-order cone program
    \be \label{eq:mpm-equi}
        \Min{w \in \mc W}~\sum\nolimits_{y \in \mc Y}\sqrt{w^\top \covsa_y w},
    \ee
    and let $\wh \kappa = \big( \sum\nolimits_{y \in \mc Y}\sqrt{\wh w^\top \covsa_y \wh w} \big)^{-1}$, and $\wh b = \wh w^\top \msa_{+1} - \wh \kappa \sqrt{ \wh w^\top \covsa_{+1} \wh w}$.
    Then $\wh \theta = (\wh w, \wh b)$ solves the coverage-validity-aware surrogate problem~\eqref{eq:surrogate}.
\end{lemma}    

In this paper, we refer to the second-order cone program~\eqref{eq:mpm-equi} as the \textit{nominal} problem.

\section{Coverage-Validity-Aware Linear Surrogate under Model Shifts}\label{sec:modelshift}

The CVAS presented in the previous section is adapted to the current black-box model $f$. Each synthetic sample is associated with a pseudo-label generated from $f$, which helps define the data clusters for each class and guides the surrogate problem~\eqref{eq:surrogate}. A typical recourse generation framework often assumes that the model $f$ does not change over time; nevertheless, this assumption is usually violated in practice. To model the potential shifts of the model $f$, one can model the decision boundary of $f$ and how this boundary may alter temporally. Unfortunately, representing a nonlinear decision boundary and its possible shifts poses a significant technical challenge. Further, because our framework aims to find an approximate linear surrogate $\mc C_\theta$ represented by a $(d+1)$ dimensional parameter, tracing the nonlinear shifts of the decision boundary of $f$ is a likely overload of unnecessary information. 

We propose to model the potential shifts of the black-box model $f$ using the shifts of the (synthesized) data clusters. Specifically, we do not model how individual data may shift; instead, we model how the aggregated second-order moment information of the data clusters may change. To do this, we suppose that the covariance matrix of the $y$-class data cluster may shift to a value $\cov_y$, which can be different from the nominal value $\covsa_y$. These shifts of $\cov_y$ away from $\covsa_y$ represent the second-order moment shift of the future ML model away from the current model $f$. The shift amount is quantified using the function $\varphi$, which measures the dissimilarity between covariance matrices. The function $\varphi$ should be a divergence: it is non-negative and $\varphi(\cov_y \parallel \covsa_y) = 0$ if and only if $\cov_y = \covsa_y$. We also need to specify the shift budgets $\rho_y$, which dictate the maximal amount of covariance shifts possible for the $y$-class data cluster.

Motivated by the success of (distributionally) robust optimization in machine learning and ethical AI~\cite{ref:taskesen2021sequential, vu2022distributionally}, we introduce the covariance-robust CVAS, which is defined as the solution of the adversarial trade-off balancing problem
\be \label{eq:dro}
    \Max{\theta \in \Theta}~\Min{\cov_y \in \PSD^d: \varphi(\cov_y \parallel \covsa_y) \le \rho_y~\forall y}~\min\left\{~\mathrm{Co}_{\cov_{+1}}(\theta), \mathrm{Va}_{\cov_{-1}}(\theta)\right\},
\ee
where we make explicit the dependence of the coverage and validity measures on the covariance weighting matrices $\cov_{+1}$ and $\cov_{-1}$, respectively. The covariance-robust CVAS~\eqref{eq:dro} alleviates the model shifts problem by injecting an additional layer of conservativeness to the surrogate: the surrogate now maximizes the worst-case values of the coverage and validity measures, subject to all possible perturbations of the covariance matrix of each data cluster. One can view problem~\eqref{eq:dro} as a zero-sum game between the surrogate generator and a fictitious adversarial opponent: the surrogate generator aims to find a parameter $\theta$ that can balance the coverage and validity trade-off. In contrast, the opponent can shift the covariance matrices of the data clusters to reduce the coverage and validity of the surrogate. 

There are now two main questions about the new family of covariance-robust CVASes: (i) How could we solve problem~\eqref{eq:dro} efficiently? (ii) how could we choose the divergence $\varphi$? In Section~\ref{sec:dro-mpm}, we establish a general result on the solution method of problem~\eqref{eq:dro} for a generic choice of $\varphi$. Then, Section~\ref{sec:Gauss} motivates to choose $\varphi$ using statistical divergence between Gaussian distributions.

\subsection{Equivalence and Solution Procedure} \label{sec:dro-mpm}

Proposition~\ref{prop:MPM} establishes an equivalence between the CVAS and the MPM. It is reasonable to expect that the (covariance-robust) CVAS should be equivalent to a certain variant of MPM under probability misspecification. We will establish this equivalence in this section. We first define the ambiguity set
\begin{equation*}
    \mbb U_{y}^\varphi(\Pnom_y) = \{ 
            \PP_y :~\PP_y \sim (\msa_y, \cov_y) \text{ for some }\cov_y \in \PSD^d \text{ with }\varphi(\cov_y \parallel \covsa_y) \le \rho_y
        \}
\end{equation*}
for each conditional distribution respective to the $y$-label. Each $\mbb U_{y}^\varphi(\Pnom_y)$ is a non-parametric set: it contains all distributions supported on $\R^d$ with mean $\msa_y$ and covariance matrix $\cov_y$, whereas the $\cov_y$ resides in the neighborhood of radius $\rho_y$ from the nominal values $\covsa_y$. Consider now the MPM under probability misspecification, which is a distributionally robust optimization problem:
\be \label{eq:dro_mpm}
    \Min{\theta \in \Theta}~\Max{y \in \mc Y}~\Max{\PP_y \in \mathbb{U}^\varphi_y(\Pnom_y)} \PP_y(\mc C_\theta(X) \neq y).
\ee
The goal of problem~\eqref{eq:dro_mpm} is to find a linear classifier that minimizes the worst-case maximal misclassification rate among classes, where the conditional data generating distributions are confined to the ambiguity sets  $\mbb U_{y}^\varphi(\Pnom_y)$. Problem~\eqref{eq:dro_mpm} is not new: in fact, it was first proposed by~\cite{ref:lanckriet2003robust} but with $\varphi$ being the squared Frobenius distance $\varphi(\cov_y \parallel \covsa_y) = \Tr{(\cov_y - \covsa_y)^2}$. 

Now, we will establish that the optimal solution of problem~\eqref{eq:dro} is equivalent to the optimal solution of problem~\eqref{eq:dro_mpm}. This equivalence is established for any \textit{arbitrary} divergence function $\varphi$.  Toward this goal, for any $y \in \mc Y$, define $\tau_y^\varphi: \R^d \to \R$ as
\be \label{eq:tau}
    \tau_y^\varphi(w) \Let \Max{\cov_y \in \PSD^d: \varphi(\cov_y \parallel \covsa_y) \le \rho_y} \sqrt{w^\top \cov_y w}.
\ee
We now provide a generalized equivalence result and a reformulation of the CVAS under model shifts problem~\eqref{eq:dro}.

\begin{proposition}[Robust surrogate under model shifts] \label{prop:refor} Let $\varphi$ be an arbitrary divergence on the space of covariance matrices. The CVAS under model shifts obtained by solving~\eqref{eq:dro} coincides with the MPM under probability misspecification obtained by solving~\eqref{eq:dro_mpm}. Moreover, let $w^\varphi$ be the optimal solution to the problem
    \be \label{eq:dro-equi} \notag
        \Min{w \in \mc W}~ \sum\nolimits_{y \in \mc Y} \tau_y^\varphi(w),
    \ee
and let $\kappa^\varphi = \big( \ds \sum\nolimits_{y \in \mc Y} \tau_y^\varphi(w^\varphi) \big)^{-1}$, and $b^\varphi = (w^\varphi)^\top \msa_{y} - y \kappa^\varphi \tau_{y}^\varphi(w^\varphi)$ for any $y\in \mc Y$. Then $\theta^\varphi = (w^\varphi, b^\varphi)$ solves the covariance-robust CVAS problem~\eqref{eq:dro}.
\end{proposition}

This equivalence result provides a generic solution procedure to find the covariance-robust CVAS, provided that we can evaluate the function $\tau_y^\varphi$ in~\eqref{eq:tau}.

\subsection{Equivalence under Gaussian Assumptions} \label{sec:Gauss}
    
When $\varphi$ is chosen as the squared Frobenius distance, we recover the Quadratic MPM first studied in~\cite{ref:lanckriet2003robust}. As we will delineate in Section~\ref{sec:quad}, the Quadratic MPM will give rise to a Quadratic surrogate. Nevertheless, the squared Frobenius distance is not statistically meaningful: it does not coincide with any distance between probability distributions with the corresponding covariance information. On the upside, the connection between the MPM framework and our surrogate can be exploited to design many versions of the surrogate by simply choosing different divergences $\varphi$. One thus can consider more principled approaches of choosing $\varphi$ by leveraging the distributional properties of the MPM formulation.

This section shows that the MPM under probability misspecification problem~\eqref{eq:dro_mpm} is \textit{in}variant with the Gaussian information. Put differently, the solution to problem~\eqref{eq:dro_mpm} does not change if a parametric Gaussian assumption is imposed on the conditional distributions. To see this, define a parametric ambiguity set constructed on the space of Gaussian distributions of the form
\[
    \mc U_{y}^{\mc N}(\Pnom_y) =
    \left\{ 
        \PP_y \in \mc P(\mc X):~
        \PP_y \sim \mc N(\msa_y, \cov_y),~
        \varphi(\cov_y \parallel \covsa_y) \le \rho_y
    \right\},
\]
wherein any distribution in this set is Gaussian. Consider the Gaussian-robust MPM problem
\be \label{eq:dro-Gauss}
    \Min{\theta \in \Theta}~\Max{y \in \mc Y}~\Max{\PP_y \in \mc U_{y}^{\mc N}(\Pnom_y)}~\PP_y(\mc C_\theta(X) \neq y).
\ee
The only difference between problem~\eqref{eq:dro-Gauss} and problem~\eqref{eq:dro_mpm} is the Gaussian specification of the ambiguity sets. Nevertheless, the following result asserts that the solutions to the two problems coincide. 

\begin{proposition}[Gaussian equivalence] \label{prop:gauss}
    The optimizer $\theta^\varphi = (w^\varphi, b^\varphi)$ in Proposition~\ref{prop:refor} also solves the Gaussian-robust MPM problem~\eqref{eq:dro-Gauss}.
\end{proposition}
Proposition~\ref{prop:gauss} implies that it suffices to restrict the MPM problem to Gaussian misspecification. It thus justifies using divergences $\varphi$ induced by some dissimilarity measures between Gaussian distributions. This leads to a family of surrogates, as we will explore in the next section.


\section{Examples of Robust Surrogates} \label{sec:examples}
We discuss in this section several versions of the CVASes under model shifts, including the Quadratic CVAS in Section~\ref{sec:quad}, the Bures CVAS in Section~\ref{sec:bures}, the Fisher-Rao CVAS in Section~\ref{sec:fr}, and the LogDet CVAS in Section~\ref{sec:logdet}.

\subsection{Quadratic Surrogate} \label{sec:quad}

First, consider the perturbation of the covariance matrix using the quadratic divergence.

\begin{definition}[Quadratic divergence] \label{def:quad}
    Given two positive semidefinite matrices $\cov$, $\covsa \in \PSD^d$, the quadratic divergence between them is $\mathds Q(\cov \parallel \covsa) = \Tr{(\cov - \covsa)^2 }$.
\end{definition}
The divergence $\mathds Q$ is the \textit{squared} Frobenius norm of $\cov - \covsa$; thus $\mathds Q$ is non-negative and vanishes to zero if and only if $\cov = \covsa$, so it is a divergence on $\PSD^d$. The Quadratic CVAS has the below form. 

\begin{theorem}[Quadratic surrogate] \label{thm:quadratic}
    Suppose that $\varphi \equiv \mathds Q$. Let $w^{\mathds Q}$ be a solution to the problem
    \be \label{eq:quad}
        \Min{w \in \mc W} ~ \ds \sum\nolimits_{y \in \mc Y} \sqrt{w^\top (\covsa_y + \sqrt{\rho_y} I) w}.
    \ee 
    Then $\theta^{\mathds Q} = (w^{\mathds Q}, b^{\mathds Q})$ solves the surrogate problem under model shifts~\eqref{eq:dro}, with
    \begin{align*}
        \kappa^{\mathds Q} &= \left(\sum\nolimits_{y \in \mc Y} \sqrt{(w^{\mathds Q})^\top (\covsa_{y} + \sqrt{\rho_{y}} I) w^{\mathds Q}}\right)^{-1}, \\
        \text{and} \qquad 
        b^{\mathds Q} &= (w^{\mathds Q})^\top \msa_{+1} - \kappa^{\mathds Q} \sqrt{(w^{\mathds Q})^\top (\covsa_{+1} + \sqrt{\rho_{+1}} I) w^{\mathds Q}}.
    \end{align*}
\end{theorem}  
Theorem~\ref{thm:quadratic} is obtained by exploiting the result from~\cite{ref:lanckriet2003robust}, which asserts the optimal form of the Quadratic MPM. We present Theorem~\ref{thm:quadratic} mainly for comparison against our new statistical variants. Problem~\eqref{eq:quad} can be considered as a regularization of the nominal problem~\eqref{eq:mpm-equi}: each matrix $\covsa_y$ is added with a diagonal matrix $\sqrt{\rho_y} I$, making the matrix better conditioned. This is equivalently known as inverse regularization, which ensures invertibility when $\covsa_y$ is low-rank and the radiu $\rho_y$ is strictly positive.

\subsection{Bures Surrogate} \label{sec:bures}

We now explore a variant of a statistically motivated surrogate using Bures divergence as $\varphi$.

\begin{definition}[Bures divergence]
    Given two positive semi-definite matrices $\cov$, $\covsa \in \PSD^d$, the Bures divergence between them is $\mathds B(\cov \parallel \covsa) = \Tr{\cov + \covsa - 2 (\covsa^\half \cov \covsa^\half)^\half }$.
\end{definition}

It can be shown that $\mathds B$ is symmetric and non-negative, and it vanishes to zero if and only if $\cov = \covsa$. As such, $\mathds B$ is a divergence on the space of positive semidefinite matrices. Additionally, it is equivalent to the \textit{squared} type-2 Wasserstein distance between two Gaussian distributions with the same mean vector and covariance matrices $\cov$ and $\covsa$~\cite{ref:olkin1982distance, ref:givens1984class, ref:gelbrich1990formula}. Next, we state the form of the Bures surrogate.

\begin{theorem}[Bures surrogate] \label{thm:bures}
    Suppose that $\varphi \equiv \mathds B$. Let $w^{\mathds B}$ be the solution to the problem
    \be \label{eq:bures}
     \Min{w \in \mc W}~ \ds \sum\nolimits_{y \in \mc Y}\sqrt{w^\top \covsa_y w} + \big(\sum\nolimits_{y \in \mc Y} \rho_y \big) \|w\|_2.
    \ee 
    Then $\theta^{\mathds B} = (w^{\mathds B}, b^{\mathds B})$ solves the surrogate problem under model shifts~\eqref{eq:dro}, where
    \begin{align*}
        \kappa^{\mathds B} &= \left(\sum\nolimits_{y \in \mc Y}\sqrt{(w^{\mathds B})^\top \covsa_y w^{\mathds B}} + \big(\sum\nolimits_{y \in \mc Y} \rho_y \big) \|w^{\mathds B}\|_2 \right)^{-1}, \\  \text{and} \qquad 
    b^{\mathds B} &= (w^{\mathds B})^\top \msa_{+1} - \kappa^{\mathds B} (\sqrt{(w^{\mathds B})^\top \covsa_{+1} w^{\mathds B}} + \rho_{+1} \| w^{\mathds B} \|_2).
    \end{align*}
\end{theorem}

Theorem~\ref{thm:bures} unveils a fundamental connection between robustness and regularization: if we perturb the covariance matrices using the Bures divergence, the resulting optimization problem~\eqref{eq:bures} is an $l_2$-regularization of the problem~\eqref{eq:mpm-equi}. This connection aligns with previous observations showing the equivalence between regularization schemes and optimal transport robustness~\cite{ref:shafieezadeh2017regularization, ref:blanchet2016robust, ref:kuhn2019tutorial}. To prove Theorem~\ref{thm:bures}, we provide in Proposition~\ref{prop:bures} a result that asserts the analytical form of $\tau_y^{\mathds B}(w)$. The proof of Theorem~\ref{thm:bures} follows by combining Propositions~\ref{prop:refor} and~\ref{prop:bures}. 

\begin{proposition}[Bures divergence] \label{prop:bures}
    If $\varphi \equiv \mathds B$, then $\tau_y^{\mathds B}(w) = \rho_y \| w \|_2 + \sqrt{w^\top \covsa_y w}$ for all $ y \in \mc Y$.
\end{proposition}

\subsection{Fisher-Rao Surrogate} \label{sec:fr}

The second new variant of a statistically-motivated surrogate is obtained by taking $\varphi$ as the Fisher-Rao divergence.

\begin{definition}[Fisher-Rao distance]\label{def:FR}
Given two positive definite matrices $\cov$, $\covsa \in \PD^d$, the Fisher-Rao distance between them is
$\FR(\cov , \covsa) = \| \log (\covsa^{-\frac{1}{2}} \cov \covsa^{-\frac{1}{2}}) \|_F$,
where $\log(\cdot)$ is the matrix logarithm.
\end{definition}
The Fisher-Rao distance enjoys many desirable properties. In particular, it is invariant to inversion and congruence, \ie, for any $\cov,~\covsa\in \PD^d$ and invertible $A\in \R^{d\times d}$, $\FR(\cov, \covsa) = \FR(\cov^{-1}, \covsa^{-1}) = \FR(A \cov A^\top , A \covsa A^\top) $. Such invariances are statistically meaningful because they imply that the results remain unchanged if we reparametrize the problem with an inverse covariance matrix (instead of the covariance matrix) or if we apply a change of basis to the data space $\mc X$. It is shown that $\FR$ is the unique Riemannian distance (up to scaling) on the cone $\PD^d$ with such invariances~\cite{savage1982space}. Next, we assert the form of the Fisher-Rao surrogate.

\begin{theorem}[Fisher-Rao surrogate] \label{thm:fr}
    Suppose that $\varphi \equiv \FR$. Let $w^{\FR}$ be the solution of the following problem
    \be \label{eq:FR}
        \Min{w \in \mc W}~\sum\nolimits_{y \in \mc Y} \exp\left(\frac{\rho_y}{2}\right) \sqrt{w^\top \covsa_y w}.
    \ee
    Then $\theta^{\FR} = (w^{\FR}, b^{\FR})$ solves the surrogate problem under model shifts~\eqref{eq:dro}, where 
    \begin{align*}
        \kappa^{\FR} &\!=\!\left(\sum_{y \in \mc Y} \exp\left(\frac{\rho_y}{2}\right) \sqrt{(w^{\FR})^\top \covsa_y w^{\FR}}\right)^{-1} \\
        \text{and} \quad b^{\FR} &\!=\!(w^{\FR})^\top \msa_{+1} - \kappa^{\FR} \exp\left( \frac{\rho_{+1}}{2} \right)\sqrt{(w^{\FR})^\top \covsa_{+1} w^{\FR}} = (w^{\FR})^\top \msa_{-1} + \kappa^{\FR} \exp\left( \frac{\rho_{-1}}{2} \right)\sqrt{(w^{\FR})^\top \covsa_{-1} w^{\FR}}.
    \end{align*}
\end{theorem}
Theorem~\ref{thm:fr} divulges another foundational connection between robustness and regularization: if we construct the ambiguity sets for the covariance matrices using the Fisher-Rao distance, the resulting optimization problem~\eqref{eq:FR} is a \textit{reweighted} version of the nominal problem~\eqref{eq:mpm-equi}. Each term $(w^\top \covsa_y w)^\half$ is assigned a weight $\exp(\rho_y/2)$, which is proportional to the radius $\rho_y$. This connection aligns with previous observations highlighting the equivalence between reweighting schemes and distributional robustness~\cite{ref:ben2013robust, ref:bayraksan2015data, ref:namkoong2017variance, ref:hashimoto2018fairness}. To prove Theorem~\ref{thm:fr}, we derive an analytical expression of $\tau_y^{\FR}(w)$ in Proposition~\ref{prop:FR}. The proof of Theorem~\ref{thm:fr} follows by combining Proposition~\ref{prop:refor} and Proposition~\ref{prop:FR}.

\begin{proposition}[Fisher-Rao distance] \label{prop:FR}
   If $\varphi \equiv \FR$, then 
    $\tau_y^{\FR}(w) = \exp\left(\frac{\rho_y}{2}\right) (w^\top \covsa w)^\half$ for all $ y \in \mc Y$.
\end{proposition}

\subsection{LogDet Surrogate} \label{sec:logdet}


Last, we consider the case when $\varphi$ is the Log-Determinant (LogDet) divergence, which is a distance-like measure closely related to information theory, see \cite{nielsen2013matrix, yue2020matrix}.

\begin{definition}[LogDet divergence]
    Given two positive definite matrices $\cov$, $\covsa \in \PD^d$, the log-determinant divergence between them is
$\mathds D(\cov \parallel \covsa) = \Tr{\cov \covsa^{-1}} - \log\det (\cov \covsa^{-1}) - d$.
\end{definition}

It can be shown that $\mathds D$ is a divergence because it is non-negative and vanishes to zero if and only if $\cov = \covsa$. However, $\mathds D$ is not symmetric, and in general, we have $\mathds D(\cov \parallel \covsa) \neq \mathds D(\covsa \parallel \cov)$. The LogDet divergence $\mathds D$ is related to the relative entropy: it is equal to the Kullback-Leibler divergence between two Gaussian distributions with the same mean vector and covariance matrices $\cov$ and $\covsa$. We now provide the form of the LogDet surrogate problem. 
\begin{theorem}[LogDet surrogate] \label{thm:logdet}
    Suppose that $\varphi \equiv \mathds D$. Let $w^{\mathds D}$ be the optimal solution of the following second-order cone problem
    \be \label{eq:logdet}
        \Min{w \in \mc W}~\sum_{y \in \mc Y} \sqrt{c_y} \sqrt{w^\top \covsa_y w},
    \ee
    where $c_y = -W_{-1}(-\exp(-\rho_y-1))$ and $W_{-1}$ is the Lambert-W function for the branch $-1$. Then $\theta^{\mathds D} = (w^{\mathds D}, b^{\mathds D})$ solves the surrogate problem under model shifts~\eqref{eq:dro}, where 
    \begin{align*}
        \kappa^{\mathds D} &= \ds \left(\sum_{y \in \mc Y} \sqrt{c_y} \sqrt{(w^{\mathds D})^\top \covsa_y w^{\mathds D}}\right)^{-1} \\
       \text{and} \quad b^{\mathds D} &= (w^{\mathds D})^\top \msa_{+1} - \kappa^{\mathds D} \sqrt{c_{+1}} \sqrt{(w^{\mathds D})^\top \covsa_{+1} w^{\mathds D}} = (w^{\mathds D})^\top \msa_{-1} + \kappa^{\mathds D} \sqrt{c_{-1}} \sqrt{(w^{\mathds D})^\top \covsa_{-1} w^{\mathds D}}.
    \end{align*}
\end{theorem}
Theorem~\ref{thm:logdet} shows that the LogDet divergence induces a similar reweighting scheme as the Fisher-Rao MPM. The proof of Theorem~\ref{thm:logdet} follows by combining Proposition~\ref{prop:logdet} below, which provides the analytical form of $\tau_y^{\mathds D}(w)$, with Proposition~\ref{prop:refor}.
\begin{proposition}[LogDet divergence] \label{prop:logdet}
    Suppose that $\varphi \equiv \mathds D$, then for any $y \in \mc Y$, we have
    \[
        \tau_y^{\mathds D}(w) = \sqrt{-W_{-1}(-\exp(-\rho_y-1))} \sqrt{w^\top \covsa_y w},
    \]
    where $W_{-1}$ is the Lambert-W function for the branch $-1$.
\end{proposition}

\section{Asymptotic Surrogates and Recourse Robustness} \label{sec:asymptotic}

The preceding section shows that employing various divergences to prescribe the ambiguity sets leads to different regularizations of the vanilla CVAS. Yet, a fundamental question now arises: Are all these regularizations helpful for generating robust recourses under black-box model shifts? Unfortunately, the reformulations presented in the theorems of the preceding section do not provide any closed-form expression of the surrogates. This poses a critical challenge in comparing the robustness properties of these surrogates for recourse generation. Despite these obstacles, we manage to study the asymptotic surrogates obtained by inflating the ambiguity radii to infinity. Such an asymptotic analysis provides valuable insights into the impact of the covariance-robust surrogates on the recourse generation phase and guides the surrogate selection to promote robust recourse generation. The formal statement is presented in the following result.

\begin{proposition}[Asymptotic surrogates] \label{prop:asymptotic}
    Fix $y \in \mc Y$, let $-y$ be its opposite class, and let $a \Let \sum_{y' \in \mc Y} y' \msa_{y'}$. Suppose that $\rho_{-y}$ remains constant, then as $\rho_y \to \infty$, the optimal solution $\theta^\varphi_{\rho_y} = (w_{\rho_y}^\varphi, b_{\rho_y}^\varphi)$ of problem~\eqref{eq:dro}, parametrized by $\rho_y$, can be expressed as follows. 
    \begin{enumerate}[label=(\roman*)]
        \item \label{prop:asymptotic-1} If $\varphi$ is the Quadratic or Bures distance, then
        \[
            w_{\rho_y}^\varphi \to w_{\infty, y}^\varphi \Let \frac{a}{\| a\|_2^2} \quad \text{and} \quad b_{\rho_y}^\varphi \to b_{\infty, y}^\varphi \Let (w_{\infty, y}^\varphi)^\top \msa_{y} - y.
        \]
        \item \label{prop:asymptotic-2} If $\varphi$ is the Fisher-Rao or LogDet distance, then
        \[
            w_{\rho_y}^\varphi \to w_{\infty, y}^\varphi \Let \frac{\covsa_y^{-1} a}{a^\top \covsa_y^{-1} a} \quad \text{and} \quad b_{\rho_y}^\varphi \to b_{\infty, y}^\varphi \Let (w_{\infty, y}^\varphi)^\top \msa_{y} - y.
        \]
    \end{enumerate}
\end{proposition}
In all cases, the intercept $b_{\rho_y}^\varphi$ tends towards $b_{\infty, y}^\varphi = (w_{\infty, y}^\varphi)^\top \msa_{y} - y$; hence, the asymptotic hyperplane defined by ($w_{\infty, y}^\varphi, b_{\infty, y}^\varphi)$ is then characterized by the linear equation $(w_{\infty, y}^\varphi)^\top (x - \msa_{y}) + y = 0$. This equation identifies a hyperplane passing through the mean vector $\msa_{-y}$ because the slope satisfies the constraint $\sum_{y \in \mathcal{Y}} y w^\top \msa_y = 1$.

Proposition~\ref{prop:asymptotic}\ref{prop:asymptotic-1} shows that the Quadratic surrogate and the Bures surrogate are asymptotically equivalent even though they induce different regularizations of the vanilla surrogate~\eqref{eq:mpm-equi}. Moreover, for these two divergences, the asymptotic slope depends only on the aggregated quantity $\sum\nolimits_{y' \in \mc Y} y' \msa_{y'}$, but not on the specification of $y$. On the contrary, the asymptotic hyperplane of the Fisher-Rao and LogDet surrogate depends explicitly on the covariance matrix $\covsa_y$. 

Because there is no access to the analytical form of the surrogate, it is impossible to compare the worst-case perturbations of $\covsa_y$ across different divergences. To overcome this difficulty, we benchmark the coverage-validity trade-off using the \textit{nominal} values $\covsa_y$. The next result asserts the explicit trade-off between coverage and validity when calibrating the Fisher-Rao and LogDet radii.

\begin{proposition}[Coverage-validity trade-off]\label{prop:tradeoff}
Let $\rho = (\rho_{+1}, \rho_{-1})$ and $\rho' = (\rho'_{+1}, \rho'_{-1})$ be two sets of radii for the uncertainty sets, and $\theta_\rho = (w_\rho, b_\rho)$ and $\theta_{\rho'} = (w_{\rho'}, b_{\rho'})$ be the respective optimal hyperplanes obtained by solving the Fisher-Rao surrogate problem~\eqref{eq:FR} or the LogDet surrogate problem~\eqref{eq:logdet}. 
\begin{enumerate}[label=(\roman*), leftmargin = 5mm]
    \item\label{prop:tradeoff-i} If $\rho_{+1} = \rho_{+1}' = 0$ and $\rho_{-1} > \rho_{-1}' \ge 0$, then
    \[
    \mathrm{Co}_{\covsa_{+1}} (\theta_{\rho}) < \mathrm{Co}_{\covsa_{+1}}(\theta_{\rho'}) \quad \text{and} \quad \mathrm{Va}_{\covsa_{-1}} (\theta_{\rho}) > \mathrm{Va}_{\covsa_{-1}}(\theta_{\rho'}).
    \]

    \item\label{prop:tradeoff-ii} If $\rho_{-1} = \rho_{-1}' = 0$ and $\rho_{+1} > \rho_{+1}' \ge 0$, then
    \[
    \mathrm{Co}_{\covsa_{+1}} (\theta_{\rho}) > \mathrm{Co}_{\covsa_{+1}}(\theta_{\rho'}) \quad \text{and} \quad \mathrm{Va}_{\covsa_{-1}} (\theta_{\rho}) < \mathrm{Va}_{\covsa_{-1}}(\theta_{\rho'}).
    \]
\end{enumerate}
\end{proposition}

Proposition~\ref{prop:tradeoff} suggests that for the Fisher-Rao and LogDet surrogates, if we fix the radius of one of the two uncertainty sets to be zero, we can control the trade-off between coverage and validity by adjusting the radius of the other uncertainty set. Proposition~\ref{prop:tradeoff}\ref{prop:tradeoff-i} asserts that by ignoring the uncertainty in the positive clusters by setting $\rho_{+1} = 0$, then increasing the $\rho_{-1}$ increases the validity but decreases the coverage. Furthermore, by Proposition~\ref{prop:tradeoff}\ref{prop:tradeoff-ii}, if we neglect the covariance robustness of the validity by setting $\rho_{-1} = 0$, then increasing $\rho_{+1}$ also increases the coverage but decreases the validity. It is worth noting that the result does not hold universally for Quadratic and Bures divergences. To illustrate this, we consider the following counterexample.

\begin{counterexample} \label{counter}
Consider the following nominal mean vectors and covariance matrices:
\[
    \msa_{-1} = \begin{pmatrix}
        0 \\ 0
    \end{pmatrix},~\covsa_{-1} = \begin{bmatrix}
        5 & 2 \\
        2 & 1 
    \end{bmatrix} \quad \text{and} \quad \msa_{+1} = \begin{pmatrix}
        -10 \\
        0 \\ 
    \end{pmatrix},~\covsa_{+1} = \begin{bmatrix}
        5 & 2 \\
        2 & 1 
    \end{bmatrix}.
\]
The nominal surrogate obtained by Lemma~\ref{lemma:optimal} with $\rho_{+1}' = \rho_{-1}' = 0$ is given by $x_1 - 2 x_2 + 5 = 0$, which achieves $\mathrm{Va}_{\covsa_{-1}}(\theta_{\rho'}) = 5$. 
Fixing $\rho_{+1} = 0$ and letting $\rho_{-1} \to \infty$, Propositions~\ref{prop:asymptotic} asserts that both Quadratic and Bures surrogates converge to the hyperplane prescribed by $x_1 + 10 = 0$, which attains $\mathrm{Va}_{\covsa_{-1}}(\theta_{\rho}) = \frac{10}{\sqrt{5}}$. Therefore, $\mathrm{Va}_{\covsa_{-1}}(\theta_{\rho}) < \mathrm{Va}_{\covsa_{-1}}(\theta_{\rho'})$ even though $\rho_{-1} > \rho_{-1}'$, meaning that the Quadratic and Bures surrogates violate the validity inequality.

Meanwhile, if we use the Fisher-Rao surrogate, the optimal surrogate when $\rho_{-1} \to \infty$ is characterized by the hyperplane $x_1 - 2x_2 + 10 = 0$. This surrogate achieves $\mathrm{Va}_{\covsa_{-1}}(\theta_{\rho}) = 10$, consistent with Proposition~\ref{prop:tradeoff}. Figure~\ref{fig:asymptotic} illustrates and compares the asymptotic hyperplanes of three surrogates.
\end{counterexample}

\begin{figure}
    \centering
    \includegraphics[width=0.55\linewidth]{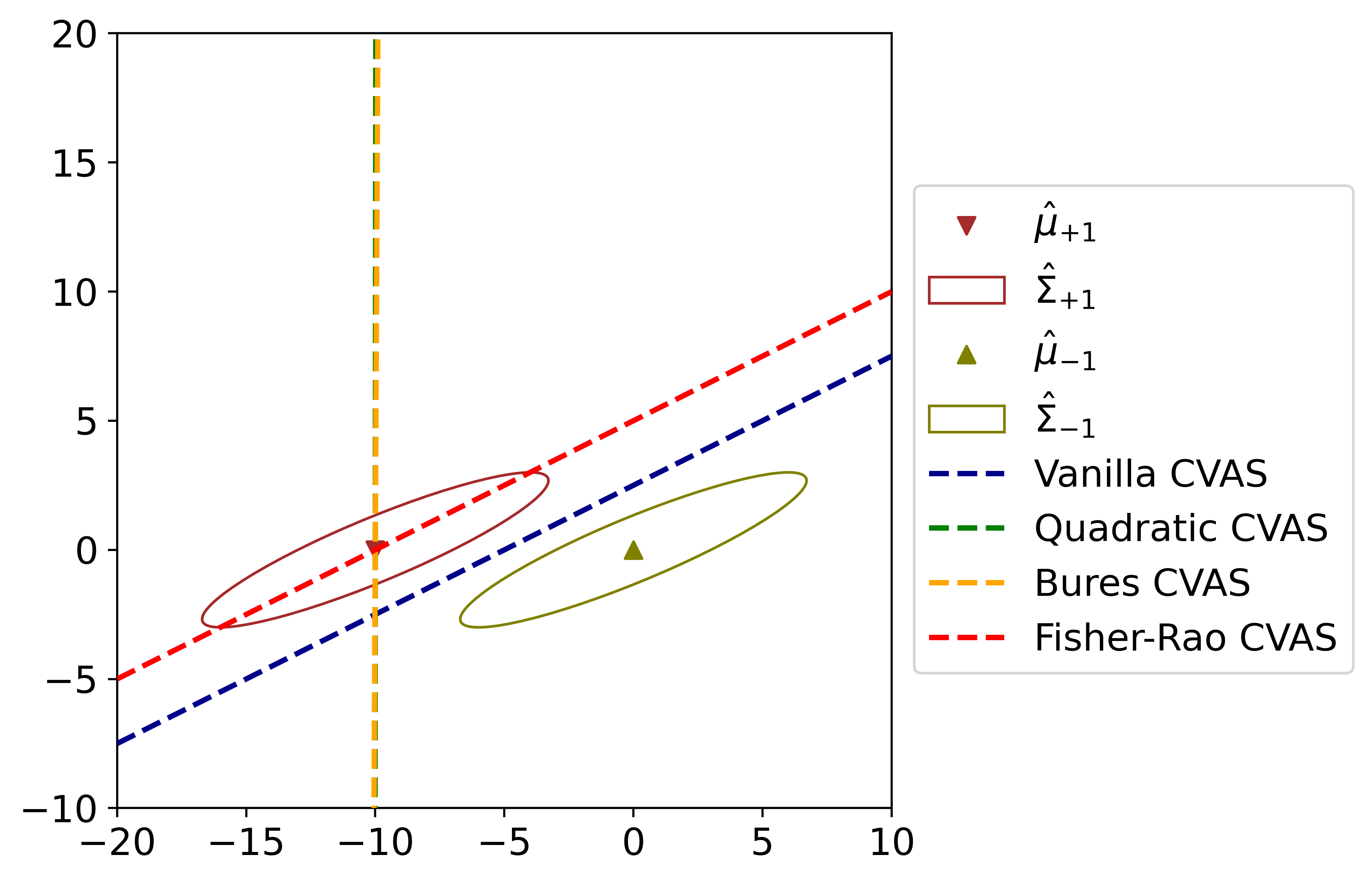}
    \caption{Comparison of the asymptotic hyperplanes of Quadratic, Bures, and Fisher-Rao surrogates as $\rho_{-1} \to \infty$.}
    \label{fig:asymptotic}
\end{figure}

Let us return to the graphical intuition presented in Figure~\ref{fig:robustify} to assess the impact of the surrogate on the recourse generation. To promote the robustness of the recourse, we could push the surrogate boundary to increase the distance from the (synthesized) negative samples to the boundary, and this is tantamount to increasing the validity $\mathrm{Va}_{\covsa_{-1}}$ defined in~\eqref{eq:validity}. Indeed, the validity measure is motivated by the precision metric in classification: it quantifies, over all samples predicted positive by the surrogate, how many of them are predicted positive by the black-box model. Hence, increasing the validity of a surrogate and employing the surrogate to guide the recourse generation can lead to an improvement in robustness. Blending Proposition~\ref{prop:tradeoff} and Counterexample~\ref{counter}, we observe that only the Fisher-Rao and the LogDet regularizations guarantee that the validity $\mathrm{Va}_{\covsa_{-1}}$ increases when we increase the ambiguity radius of the negative samples $\rho_{-1}$. One can, following the graphical intuition previously presented, expect that the Fisher-Rao and the LogDet surrogate can consistently promote robustness in the recourse generation phase. 
However, good things come at a price: increasing the validity $\mathrm{Va}_{\covsa_{-1}}$, or the distance from negative samples to the surrogate, means that the cost to implement the recourse will also increase. We will explore all these trade-offs empirically in the next section.

\section{Numerical Experiments} \label{sec:numerical}

The goals of the experiments are two-fold: First, we investigate the sensitivity and fidelity of our CVASes compared with LIME~\cite{ref:ribeiro2016why}, a widely-used surrogate in recourse literature. Second, we empirically demonstrate that CVAS can be integrated into the recourse generation problem to promote robustness. More extensive experiments, including comparisons with RBR~\cite{ref:nguyen2022robust} and DiRRAc~\cite{ref:nguyen2023distributionally}, are provided in Appendix~\ref{sec:app:exp:add}.

We provide information regarding the architecture we use for the black-box model, the dataset, and the data processing.

\textbf{Classifier.} To construct the black-box classifier, we use a three-layer MLP with 20, 50, and 20 nodes and ReLU activation in each consecutive layer. We use a sigmoid function in the last layer to produce probabilities. We use the binary cross-entropy to train this classifier, solved using the Adam optimizer and 1000 epochs.

\textbf{Dataset.} We evaluate our framework using popular real-world datasets for algorithmic recourse: \textit{German Credit} \cite{ref:dua2017uci, ref:groemping2019south}, \textit{Small Bussiness Administration (SBA)} \cite{ref:li2018should}, and \textit{Student performance} \cite{ref:cortez2008using}. Each dataset contains two sets of data (the present data $D_1$ and the shifted data $D_2$). The shifted dataset $D_2$ could capture the correction shift (for the German dataset), temporal shift (SBA), or geospatial shift (Student). For each dataset, we use $80\%$ of the instances in the present data $D_1$ to train an underlying classifier, and the remaining instances are used as input to generate recourses. The shifted data $D_2$ is used to train future classifiers to evaluate the validity of the recourse; see further details in Section~\ref{sec:expt:robust}. 

\textbf{Naming convention.} The Quadratic surrogate obtained in Theorem~\ref{thm:quadratic} is denoted QUAD-CVAS, the Bures surrogate obtained in Theorem~\ref{thm:bures} is denoted BW-CVAS, and the Fisher-Rao surrogate obtained in Theorem~\ref{thm:fr} is denoted FR-CVAS.

\textbf{Hyperparameters for the local sampler.} In all experiments, we choose $k=10$ examples in the favorable class with the smallest $\ell_1$ distance to the input instance $x_0$. The perturbation radius $r_p$ is set to 5\% of the maximum distance between instances in the available data.

\subsection{Sensitivity and Fidelity of CVAS}\label{sec:expt:fid-stab}

In the first set of experiments, we assess the quality of our covariance-robust CVASes in terms of their capacity as local surrogates to a black-box model. We compare the covariance-robust CVASes against LIME~\cite{ref:ribeiro2016why}, a well-known linear surrogate, in terms of the sensitivity with respect to the input and the fidelity with respect to the black-box model. Both sensitivity~\cite{ref:agarwal2021towards} and local fidelity~\citep{ref:laugel2018defining} are standard metrics in the literature to measure the quality of local surrogates.

\textbf{Sensitivity.} We use the procedure in~\cite{ref:agarwal2021towards} to measure the sensitivity of the surrogates with respect to small perturbations in the input instance. For a given instance $x$, we draw a set $\mc U_x$ of 10 neighbors of $x$ from $\mathcal{N}(x, 0.001 I)$ independently. We use the above-mentioned methods to find the linear surrogate $\theta_{x'} = (w_{x'}, b_{x'})$ for each $x' \in \mc U_x$. We report the maximum difference between the surrogates for $x'$ and that for $x$. Precisely, the sensitivity of a surrogate $\theta_x$ can be computed by 
\[
    \mathrm{Sensitivity}(\theta_x) = \max\nolimits_{x' \in \mc U_x} \| w_x -  w_{x'} \|_2.
\]
Ideally, the smaller the sensitivity, the better because it indicates that the surrogate is more stable to perturbations in the input instance $x_0$. 

\begin{figure}[!t]
    \centering
    \includegraphics[width=1.\linewidth]{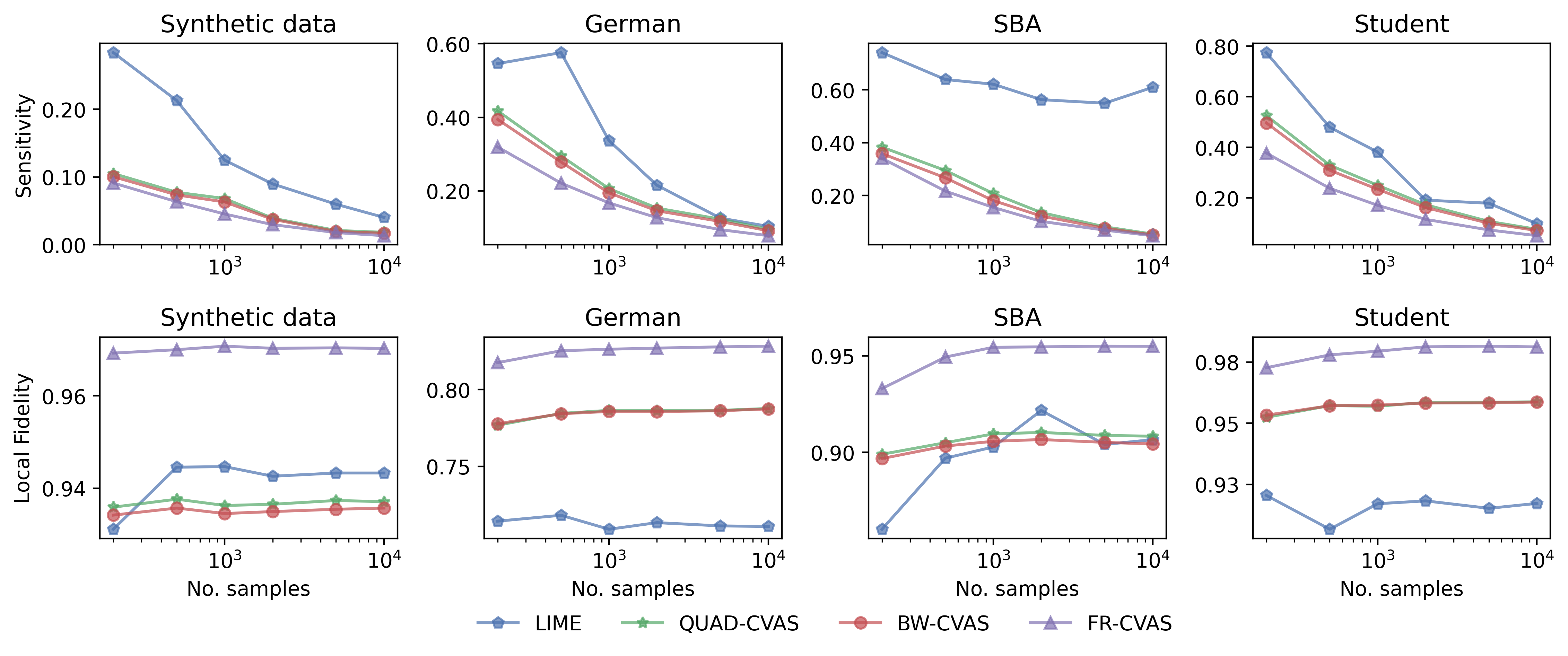}
    \caption{Benchmarks of sensitivity (top row) and local fidelity (bottom row) on four datasets. Lower sensitivity and higher local fidelity are better.}
    \label{fig:fid_stab_syge}
\end{figure}

\textbf{Local Fidelity.} We use the $\mathrm{LocalFid}$ criterion as in~\cite{ref:laugel2018defining} to measure the fidelity of a local surrogate model to the underlying model. For a given instance $x$ and a constructed linear surrogate $\mc C_{\theta_x}$, we draw a set $\mc V_x$ of 1000 instances uniformly from an $l_2$-ball of radius $r_{\mathrm{fid}}$ centered on $x$. The local fidelity of the surrogate $\theta_x$ is then measured as:
\[
    \mathrm{LocalFid}(\theta_x) =\frac{1}{|\mc V_x|} \sum_{x' \in \mathcal{V}_x} \mathbb{I}_{f(x') = \mc C_{\theta_x}(x')},
\]
where $f$ is the original black-box classifier and $\mathbb{I}$ is the indicator function. The metric $\mathrm{LocalFid}$ measures the fraction of instances where the output class of $f$ and $\mc C_{\theta_x}$ agree. The higher local fidelity value indicates that the linear surrogate $\mc C_{\theta_x}$ better approximates the local decision boundary of $f$. Here, we set $r_{\mathrm{fid}}$ to 10$\%$ of the maximum distance between instances in the available data. Note that $\mathcal{V}_x$ is for evaluation only, independent of the perturbation samples used to train the local surrogate.

To construct our CVASes, we set $\rho_{+1} = 0$, $\rho_{-1} = 1.0$. For LIME, we use the default parameters recommended in the LIME source code\footnote{https://github.com/marcotcr/lime} and return $\theta = (w, b - 0.5)$ as the LIME's surrogate, similar to~\cite{ref:laugel2018defining}. We vary the number of perturbation samples in a $[500, 10000]$ range to measure the fidelity and sensitivity of constructed surrogates under small sampling sizes. The results in Figure~\ref{fig:fid_stab_syge} show the superiority of CVASes to LIME in sensitivity and local fidelity metrics. Meanwhile, FR-CVAS provides higher-fidelity surrogates compared to QUAD-CVAS and BW-CVAS. These results assert that CVAS variants can serve as competitive linear surrogates to approximate the nonlinear decision boundary of black-box classifiers. The family of (covariance-robust) CVASes thus possesses a great opportunity to be an algorithmic explainer. However, we will focus on integrating CVASes into the recourse generation workflow for the remainder of the experiments.

\begin{figure}[!t]
    \centering    \includegraphics[width=\linewidth]{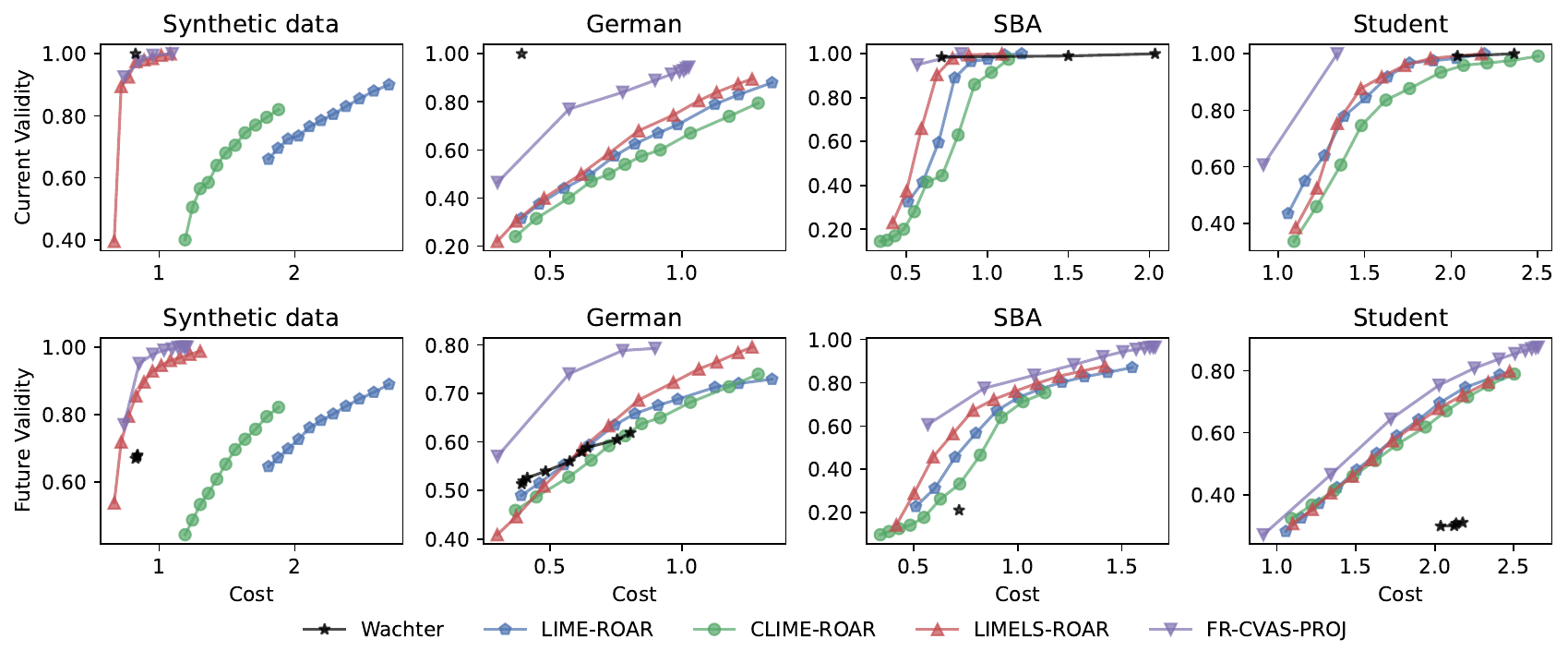}
    \caption{Pareto frontier of the cost-validity trade-off on four datasets. Each method has 11 configurations for the ambiguity size, but only non-dominated configurations are presented for clarity.}
    \label{fig:mlp_pareto}
\end{figure}
\subsection{CVAS for Robust Recourse Generation}\label{sec:expt:robust}

We now study the integration of the covariance-robust CVASes into the recourse generation scheme; we study the related robustness of our recourse against shifts of the black-box model and the cost-validity trade-off of the recourse. Proposition~\ref{prop:tradeoff} suggests that the Fisher-Rao model is a good candidate to form the surrogate for the recourse generation. In this section, we use FR-CVAS as the linear surrogate $\theta^\FR = (w^\FR, b^\FR)$, and we integrate $\theta^\FR$ with two methods of recourse generation: gradient-based recourse method (Section~\ref{sec:robust-proj}) and actionable recourse method (Section~\ref{sec:robust-AR}). Notice that $\theta^\FR$ inherently depends on the ambiguity size $\rho$, but this dependence is omitted to avoid clutter.

\textbf{Metrics.} Throughout the experiment, we use standard cost and validity metrics from~\cite{ref:ustun2019actionable, ref:upadhyay2021towards} to evaluate the quality of the generated recourses.
\begin{itemize}[leftmargin=5mm]
    \item \textit{Cost}: we use the $\ell_1$-distance between the constructed recourse $x_r$ and the input instance $x_0$, \ie, $\|x_r - x_0\|_1$ to measure the difficulty to implement recourse suggestions. The lower the cost, the better.
    \item \textit{Current validity \cite{ref:upadhyay2021towards, ref:rawal2020can}}: We define the \textit{current validity} as the percentage of recourses generated using the surrogate model remains valid with respect to the current black-box model $f$, which is known at the time the recourse is generated. The higher the current validity, the better.
    \item \textit{Future validity \cite{ref:upadhyay2021towards, ref:rawal2020can}}: To assess the robustness of recourses to model shifts, we leverage the shifted datasets $D_2$ to simulate future-shifted models. Specifically, we sample $80\%$ instances of the shifted data $D_2$ repeatedly 100 times to train 100 realizations of the `future' black-box models. Future validity is then computed as the average percentage of recourses remaining valid with respect to those future models. Note that this approach to computing the future validity metric differs from~\cite{ref:upadhyay2021towards}, which only trains one future model using the entire dataset $D_2$ and thus does not capture the uncertainty in future models. The higher the future validity, the better.
\end{itemize}

\subsubsection{Robust Projection-based Recourse}   \label{sec:robust-proj}

Because CVAS is a linear surrogate, the most simple recourse search is a projection onto the decision boundary of the CVAS surrogate. We thus consider the following recourse search mechanism 
\[
x_r = \arg\min \{ \| x - x_0 \|_1:  x^\top w^\FR + b^\FR \ge 0 \},
\]
which is the 1-norm projection onto the hyperplane $\{x: x^\top w^\FR + b^\FR \ge 0\}$.
This method is named FR-CVAS-PROJ. We compare this simple recourse against several baselines: 
\begin{enumerate}
\item [(i)] Wachter~\cite{ref:wachter2017counterfactual} (Wachter).  We suppose that the binary classifier $f$ takes the form
\[
    f(x) = \begin{cases}
        +1 & \text{if } g(x) \ge 0.5, \\
        -1 & \text{otherwise,}
    \end{cases}
\]
where $g(x)$ is the probability output of the model. Then Wachter solves
\be \label{eq:wachter}
 \Min{x}~\left\{(g(x) - 0.5)^2 + \lambda \| x - x_0 \|_1 \right\},
\ee
where the first term is a quadratic loss function between the probability output of $x$ and the threshold $0.5$, and the second term is the $\ell_1$-distance between $x$ and $x_0$. The parameter $\lambda > 0$ is the weight balancing the validity and the implementation cost.
We use a well-known CARLA's implementation\footnote{https://github.com/carla-recourse/CARLA} for Wachter~\cite{ref:wachter2017counterfactual}. This repository employs an adaptive scheme to adjust the hyperparameter $\lambda$ if no valid recourse is found.
\item [(ii)] a simple 1-norm projection onto the LIME surrogate (LIME-PROJ), 
\item [(iii-v)] a min-max formulation ROAR~\cite{ref:upadhyay2021towards} coupled with three local surrogates LIME~\cite{ref:ribeiro2016why}, CLIME~\cite{ref:agarwal2021towards}, and LIMELS~\cite{ref:laugel2018defining} as the nominal surrogate (LIME-ROAR, CLIME-ROAR, and LIMELS-ROAR, respectively). 
While LIME trains a weighted linear regression,  CLIME trains a linear regression, and LIMELS trains a ridge regression to find the surrogate. Because there is no publicly available implementation for CLIME~\cite{ref:agarwal2021towards} and LIMELS~\cite{ref:laugel2018defining}, we implement according to the original papers and based on LIME's source code. We use the CARLA's implementation for ROAR and set the initial $\lambda$ to $0.1$ as suggested in~\cite{ref:upadhyay2021towards}.
\end{enumerate}
Note, once again, that Wachter and LIME-PROJ are simple gradient-based methods that are not necessarily robust, while ROAR is a robust method using a min-max formulation.

\begin{table}[!t]
    \centering
    \caption{Performance of competing algorithms on the German, SBA, and Student datasets. For the current and future validity, higher is better. For the cost, lower is better. Bold indicates the best performance. FR-CVAS-PROJ has similar future validity to ROAR-related methods but has higher current validity and lower cost.}
    \label{tab:realdata}
    \resizebox{\linewidth}{!}{
    \begin{filecontents*}{data/ept3_mlp_hor_gesbst.csv}
method,cost-ge,cur-vald-ge,fut-vald-ge,cost-sb,cur-vald-sb,fut-vald-sb,cost-st,cur-vald-st,fut-vald-st
Wachter,\textbf{0.44} $\pm$ 0.03,\textbf{1.00} $\pm$ 0.00,0.53 $\pm$ 0.03,\underline{1.15} $\pm$ 0.15,\underline{0.97} $\pm$ 0.01,0.14 $\pm$ 0.06,\underline{2.30} $\pm$ 0.04,\textbf{0.95} $\pm$ 0.03,0.31 $\pm$ 0.04
LIME-PROJ,\underline{0.45} $\pm$ 0.10,0.35 $\pm$ 0.11,0.51 $\pm$ 0.08,\textbf{0.81} $\pm$ 0.06,\underline{0.97} $\pm$ 0.01,0.79 $\pm$ 0.03,\textbf{1.07} $\pm$ 0.13,0.78 $\pm$ 0.04,0.37 $\pm$ 0.07
LIME-ROAR,1.09 $\pm$ 0.19,0.66 $\pm$ 0.15,0.64 $\pm$ 0.08,1.83 $\pm$ 0.16,\textbf{1.00} $\pm$ 0.00,0.87 $\pm$ 0.04,2.41 $\pm$ 0.31,\textbf{0.95} $\pm$ 0.03,0.74 $\pm$ 0.11
CLIME-ROAR,1.40 $\pm$ 0.45,0.76 $\pm$ 0.17,\underline{0.72} $\pm$ 0.11,1.50 $\pm$ 0.53,0.97 $\pm$ 0.04,0.78 $\pm$ 0.13,3.07 $\pm$ 0.75,\underline{0.95} $\pm$ 0.04,\textbf{0.81} $\pm$ 0.11
LIMELS-ROAR,1.29 $\pm$ 0.07,0.88 $\pm$ 0.03,\textbf{0.79} $\pm$ 0.02,1.63 $\pm$ 0.23,\textbf{1.00} $\pm$ 0.00,\underline{0.88} $\pm$ 0.05,2.47 $\pm$ 0.29,\textbf{0.95} $\pm$ 0.03,0.75 $\pm$ 0.09
FR-CVAS-PROJ,1.04 $\pm$ 0.02,\underline{0.94} $\pm$ 0.03,\textbf{0.79} $\pm$ 0.01,1.72 $\pm$ 0.13,\textbf{1.00} $\pm$ 0.00,\textbf{0.97} $\pm$ 0.02,2.34 $\pm$ 0.20,\textbf{0.95} $\pm$ 0.03,\underline{0.80} $\pm$ 0.06
\end{filecontents*}
\pgfplotstabletypeset[
        col sep=comma,
        string type,
        every head row/.style={before row={%
            \toprule \multirow{2}{*}{Method}  &
            \multicolumn{3}{c}{German} & \multicolumn{3}{c}{SBA} & \multicolumn{3}{c}{Student} \\
            \cmidrule(r){2-4} \cmidrule(r){5-7} \cmidrule(r){8-10}
            },
            after row=\midrule},
        every last row/.style={after row=\bottomrule},
        columns/dataset/.style={column name=Dataset, column type={l}},
        columns/method/.style={column name={}, column type={l}},
        columns/cost-ge/.style={column name=\textit{Cost} $\downarrow$, column type={c}},
        columns/cur-vald-ge/.style={column name=\textit{Cur Validity} $\uparrow$, column type={c}},
        columns/fut-vald-ge/.style={column name=\textit{Fut Validity} $\uparrow$, column type={c}},
        columns/cost-sb/.style={column name=\textit{Cost} $\downarrow$, column type={c}},
        columns/cur-vald-sb/.style={column name=\textit{Cur Validity} $\uparrow$, column type={c}},
        columns/fut-vald-sb/.style={column name=\textit{Fut Validity} $\uparrow$, column type={c}},
        columns/cost-st/.style={column name=\textit{Cost} $\downarrow$, column type={c}},
        columns/cur-vald-st/.style={column name=\textit{Cur Validity} $\uparrow$, column type={c}},
        columns/fut-vald-st/.style={column name=\textit{Fut Validity} $\uparrow$, column type={c}},
    ]{data/ept3_mlp_hor_gesbst.csv}
    }
\end{table}
The main comparison herein is the cost-validity trade-off among different methods. We fix the number of perturbation samples to $1000$ and vary the ambiguity size with $\rho_{+1} = 0$, $\rho_{-1} \in [0, 10]$ with step size $0.1$ for FR-CVAS-PROJ, and $\delta_{\max} \in [0, 0.2]$ with step size $0.02$ for the uncertainty size of ROAR. We then plot the Pareto frontiers of the cost-validity trade-off in Figure~\ref{fig:mlp_pareto}. Generally, increasing the ambiguity size $\rho_{-1}$ will increase the current and future validity of the FR-CVAS-PROJ recourse but will also sacrifice in terms of the implementation cost. A similar observation applies to ROAR-related recourses when we increase the uncertainty size $\delta_{\max}$. These results are consistent with the analysis in~\cite{ref:rawal2020can}. However, the Pareto frontiers of FR-CVAS-PROJ dominate the frontiers of ROAR-related methods on all evaluated datasets. In other words, with the same cost (or validity), our method will provide recourses with a higher validity (or lower cost) compared to ROAR. 
For the experiment in Table~\ref{tab:realdata}, we choose $\rho_{+1} = 0, \rho_{-1}=10$ for FR-CVAS-PROJ and $\delta_{\max} = 0.2$ for ROAR with different surrogates. Table~\ref{tab:realdata} demonstrates that our method has similar validity but a much smaller cost than the best baseline LIMELS-ROAR on German datasets. Meanwhile, our method achieves higher validity with reasonable cost on SBA and Student datasets.

\subsubsection{Robust Actionable Recourse} \label{sec:robust-AR}

Because our proposed covariance-robust CVASes can explicitly hedge against model shifts, our method can be integrated with AR~\cite{ref:ustun2019actionable} to promote robust and actionable recourses. To ensure that recourses are actionable, AR restricts each feature to a \textit{discrete} set of feasible values predefined using the available data. We follow the specification in~\cite{ref:ustun2019actionable} and model the actionability using mixed-integer constraints. The robust actionable recourse is 
\[
x_{\text{ar}} = \arg\min \{ \| x - x_0 \|_1:  x^\top w^\FR + b^\FR \ge 0,~\delta = x - x_0,~\delta \text{~actionable} \}.
\]
Here, each feature $\delta_j$ is restricted to a grid of $m_j + 1$ feasible values $\delta_j \in \{0, \delta_{j1},\ldots, \delta_{jm_{j}}\}$ via the indicator variables. Following the same setup in~\cite{ref:pawelczyk2021carla}, we consider the actionability constraints such as immutable race, gender, or non-decrease age; see Appendix \ref{sec:app:exp:detail} for the specification of the actionability constraints. We use the original authors' implementations for AR.\footnote{https://github.com/ustunb/actionable-recourse}

For the baselines, we will compare our above robust actionable recourse against three different surrogates, where $(w^\FR, b^\FR)$ are replaced by the LIME, CLIME, and LIMELS, respectively. For FR-CVAS, we set $\rho_{+1} = 0$, $\rho_{-1} = 1.0$. Table~\ref{tab:mlp_realdata_ar:ge} demonstrates that the actionable recourse with FR-CVAS surrogate increases the current and future validity substantially compared to other surrogates. 

\begin{table}[!t]
    \centering
    \caption{Performance of AR using different local surrogates.}
    \label{tab:mlp_realdata_ar:ge}
    \resizebox{\linewidth}{!}{
    \begin{filecontents*}{data/ept3_mlp_hor_ar_gesbst.csv}
method,cost-ge,cur-vald-ge,fut-vald-ge,cost-sb,cur-vald-sb,fut-vald-sb,cost-st,cur-vald-st,fut-vald-st
LIME-AR,\underline{0.44} $\pm$ 0.08,0.14 $\pm$ 0.05,0.39 $\pm$ 0.07,\textbf{4.76} $\pm$ 3.06,0.09 $\pm$ 0.06,0.05 $\pm$ 0.03,\underline{3.38} $\pm$ 0.25,0.46 $\pm$ 0.07,0.42 $\pm$ 0.06
CLIME-AR,0.72 $\pm$ 0.52,\underline{0.27} $\pm$ 0.12,\underline{0.44} $\pm$ 0.10,5.73 $\pm$ 4.13,\underline{0.40} $\pm$ 0.38,\underline{0.17} $\pm$ 0.18,4.35 $\pm$ 1.24,\underline{0.63} $\pm$ 0.28,\underline{0.49} $\pm$ 0.15
LIMELS-AR,\textbf{0.36} $\pm$ 0.11,0.24 $\pm$ 0.09,0.44 $\pm$ 0.08,\underline{5.48} $\pm$ 3.17,0.19 $\pm$ 0.07,0.07 $\pm$ 0.04,\textbf{3.16} $\pm$ 0.16,0.49 $\pm$ 0.13,0.38 $\pm$ 0.03
FR-CVAS-AR,1.95 $\pm$ 0.11,\textbf{0.80} $\pm$ 0.05,\textbf{0.73} $\pm$ 0.03,7.59 $\pm$ 2.42,\textbf{0.83} $\pm$ 0.18,\textbf{0.50} $\pm$ 0.21,9.64 $\pm$ 0.38,\textbf{1.00} $\pm$ 0.00,\textbf{0.95} $\pm$ 0.03
\end{filecontents*}
\pgfplotstabletypeset[
        col sep=comma,
        string type,
        every head row/.style={before row={%
            \toprule \multirow{2}{*}{Method}  &
            \multicolumn{3}{c}{German} & \multicolumn{3}{c}{SBA} & \multicolumn{3}{c}{Student} \\
            \cmidrule(r){2-4} \cmidrule(r){5-7} \cmidrule(r){8-10}
            },
            after row=\midrule},
        every last row/.style={after row=\bottomrule},
        columns/dataset/.style={column name=Dataset, column type={l}},
        columns/method/.style={column name={}, column type={l}},
        columns/cost-ge/.style={column name=\textit{Cost} $\downarrow$, column type={c}},
        columns/cur-vald-ge/.style={column name=\textit{Cur Validity} $\uparrow$, column type={c}},
        columns/fut-vald-ge/.style={column name=\textit{Fut Validity} $\uparrow$, column type={c}},
        columns/cost-sb/.style={column name=\textit{Cost} $\downarrow$, column type={c}},
        columns/cur-vald-sb/.style={column name=\textit{Cur Validity} $\uparrow$, column type={c}},
        columns/fut-vald-sb/.style={column name=\textit{Fut Validity} $\uparrow$, column type={c}},
        columns/cost-st/.style={column name=\textit{Cost} $\downarrow$, column type={c}},
        columns/cur-vald-st/.style={column name=\textit{Cur Validity} $\uparrow$, column type={c}},
        columns/fut-vald-st/.style={column name=\textit{Fut Validity} $\uparrow$, column type={c}},
    ]{data/ept3_mlp_hor_ar_gesbst.csv}
      }
\end{table}

\section{Conclusions} \label{sec:conclusion}
This paper developed a new recourse design framework that is robust to future model shifts, with the flexibility to incorporate additional mixed-integer constraints for capturing various practical considerations. A notable feature of our framework is the novel linear surrogate approximating the nonlinear decision boundary of a black-box machine learning model called the coverage-validity-aware surrogate. This CVAS surrogate allows us to balance the two criteria, coverage and validity, in a geometrically intuitive manner. Theoretically, we established strong connections between this surrogate and the MPM classifier. Moreover, we studied multiple robust variants which can hedge against boundary shifts. We showed that these robust variants correspond to new forms of regularizations of the MPM classifier. We also investigated the asymptotic properties of the robust variants of the CVAS surrogate. Through extensive numerical results using real datasets, we demonstrated that our surrogates exhibit higher fidelity to the underlying model and lower sensitivity to the problem inputs than well-known baseline surrogates. Consequentially, our recourses achieved robustness with a better cost-validity trade-off while enabling actionability through mixed-integer constraints.

\noindent\textbf{Acknowledgments.} Viet Anh Nguyen gratefully acknowledges the generous support from the UGC ECS Grant 24210924, the CUHK’s Improvement on Competitiveness in Hiring New Faculties Funding Scheme and the CUHK's Direct Grant Project Number 4055191.

\bibliographystyle{siam}
\bibliography{bibliography}


\appendix

\section{Additional Experiments} \label{sec:app:exp}

\subsection{Experimental Details}\label{sec:app:exp:detail}

\textbf{Synthetic data generation.} For synthetic data, we generate 2-dimensional data by sampling instances uniformly in a rectangle $x = (x_1, x_2) \in [-2, 4] \times [-2, 7]$. Each sample is labeled using the following function: 
\[
    f(x) = \left\{
            \begin{array}{cl}
                1 & \mathrm{if} \quad x_2 \ge 1 + x_1 + 2 x_1^2 + x_1^3 - x_1^4 + \varepsilon, \\
                -1 & \mathrm{otherwise},
            \end{array}
        \right.
\]
where $\varepsilon$ is a random noise. We generate a present data set $D_1$ with $\varepsilon = 0$ and a shifted data set $D_2$ with $\varepsilon \sim \mc N(0, 1)$. The current decision boundary of the MLP classifier for the synthetic data is illustrated in Figure~\ref{fig:synthetic_example}. 

\begin{figure}[!ht]
    \centering
    \includegraphics[width=0.5\linewidth]{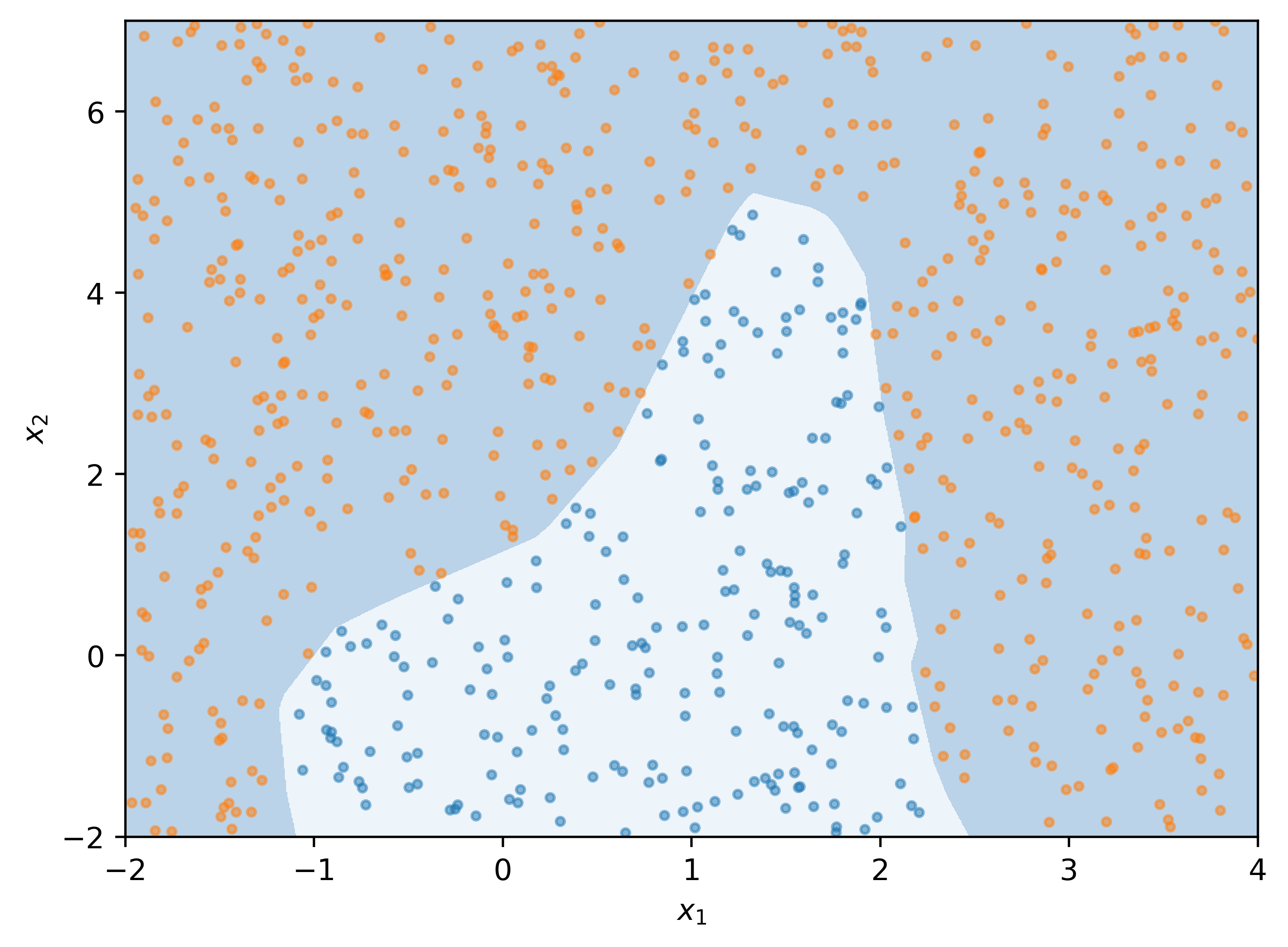}
    \caption{An illustration of MLP's decision boundary for the synthetic data.}
    \label{fig:synthetic_example}
\end{figure}

\textbf{Real-world datasets.} The details of three real-world datasets are listed below: 

\begin{enumerate}[label=(\roman*), leftmargin=5mm, partopsep=0pt,topsep=0pt]
    \item \textit{German Credit} \cite{ref:dua2017uci}. The dataset contains information (e.g., age, gender, financial status, etc.) on 1000 customers who took out bank loans. The classification task predicts an individual's risk (good or bad). There is another version of this dataset regarding corrections of coding error \cite{ref:groemping2019south}. We use the corrected version of this dataset as shifted data to capture the correction shift. The features we used in this dataset include `duration', `amount', `personal\_status\_sex', and `age'. When considering actionability constraints in Section \ref{sec:expt:robust}, we set `personal\_status\_sex' as immutable and `age' as non-decreasing.
    \item \textit{Small Bussiness Administration (SBA)} \cite{ref:li2018should}. This data includes 2,102 observations with historical data of small business loan approvals from 1987 to 2014. We divide this dataset into two datasets (one is instances from 1989 - 2006, and one is instances from 2006 - 2014) to capture temporal shifts. We use the following features: selected, `Term', `NoEmp', `CreateJob', `RetainedJob', `UrbanRural', `ChgOffPrinGr', `GrAppv', `SBA\_Appv', `New', `RealEstate', `Portion', `Recession'. When considering actionability constraints, we set `UrbanRural' as immutable.
    \item \textit{Student performance} \cite{ref:cortez2008using}. This data includes the performance records of 649 students in two schools: Gabriel Pereira (GP) and Mousinho da Silveira (MS). The classification task is to determine whether their final score is above average. We split this dataset into two sets in two schools to capture geospatial shifts. The features we used are: `age', `Medu', `Fedu', `studytime', `famsup', `higher', `internet', `romantic', `freetime', `goout', `health', `absences', `G1', `G2'.  When considering actionability constraints, we set `romantic' as immutable and `age' as non-decreasing.
\end{enumerate}

For categorical features, we convert them to binary features using the same one-hot encoding procedure proposed by~\cite{ref:mothilal2020explaining}. We also normalize continuous features to zero mean and unit variance before training the classifier. The classifier's performance on all datasets is reported in Table \ref{tab:clf_prfm}. 

\begin{table}[!ht]
    \centering
    \caption{Accuracy and AUC results of the classifiers on the synthetic and three real-world datasets.}
    \small
    \begin{tabular}{llcccc}
        \toprule
         \multirow{2}{*}{Classifier} & \multirow{2}{*}{Dataset} & \multicolumn{2}{c}{Present data $D_1$} & \multicolumn{2}{c}{Shift data $D_2$} \\
         \cmidrule(r){3-4} \cmidrule(r){5-6}
         && \textit{Accuracy $\uparrow$} & \textit{AUC $\uparrow$} & \textit{Accuracy $\uparrow$} & \textit{AUC $\uparrow$} \\
         \midrule
         \multirow{4}{*}{MLP}&Synthetic data & 0.99 $\pm$ 0.00 & 1.00 $\pm$ 0.00 & 0.94 $\pm$ 0.01 & 0.99 $\pm$ 0.01 \\
         &German credit & 0.67 $\pm$ 0.02 & 0.60 $\pm$ 0.03 & 0.66 $\pm$ 0.23 & 0.60 $\pm$ 0.04 \\
         &SBA & 0.96 $\pm$ 0.00 & 0.99 $\pm$ 0.00 & 0.98 $\pm$ 0.01 & 0.96 $\pm$ 0.01 \\
         &Student & 0.86 $\pm$ 0.02 & 0.93 $\pm$ 0.01 & 0.91 $\pm$ 0.04 & 0.97 $\pm$ 0.02 \\
         \bottomrule
    \end{tabular}
    \label{tab:clf_prfm}
\end{table}

\textbf{Reproducibility.} We release all source code and scripts to replicate our experimental results at \url{https://anonymous.4open.science/r/cvas}. The repository includes source code, datasets, configurations, and instructions. 

The hyperparameter configurations for our methods and other baseline are clearly stated in Section~\ref{sec:numerical}, Appendix~\ref{sec:app:exp:detail}, and Appendix~\ref{sec:app:exp:add} and also stored in the repository. The surrogates sharing the same local sampler have the same random seed and, therefore, have the same synthesized samples. The hyperparameters that affect the baselines' performance, such as $\lambda$ and the probabilistic threshold of Wachter and ROAR, will also be studied in Appendix~\ref{sec:app:exp:add}.

\subsection{Additional Experimental Results}\label{sec:app:exp:add}

\subsubsection{Local Fidelity and Sensitivity Comparisons} 
We run with a different setting for the sensitivity and local fidelity metric to assess the sensitivity of the results to the parameter choices.  Specifically, we sample 10 neighbors in the distribution $\mc N(x, 0.0001I)$ instead of $\mc N(x, 0.001I)$ to measure the sensitivity. Meanwhile, we set $r_{fid}$ to 20\% and the radius $r_p$ to 10\% of the maximum distance between available instances. The result in Figure~\ref{fig:fid_stab_sbst_other} is consistent with Figure~\ref{fig:fid_stab_syge}, showing that our evaluation is not sensitive to the choice of hyper-parameters used to compute the metrics.

\begin{figure}[!ht]
    \centering
    \includegraphics[width=\linewidth]{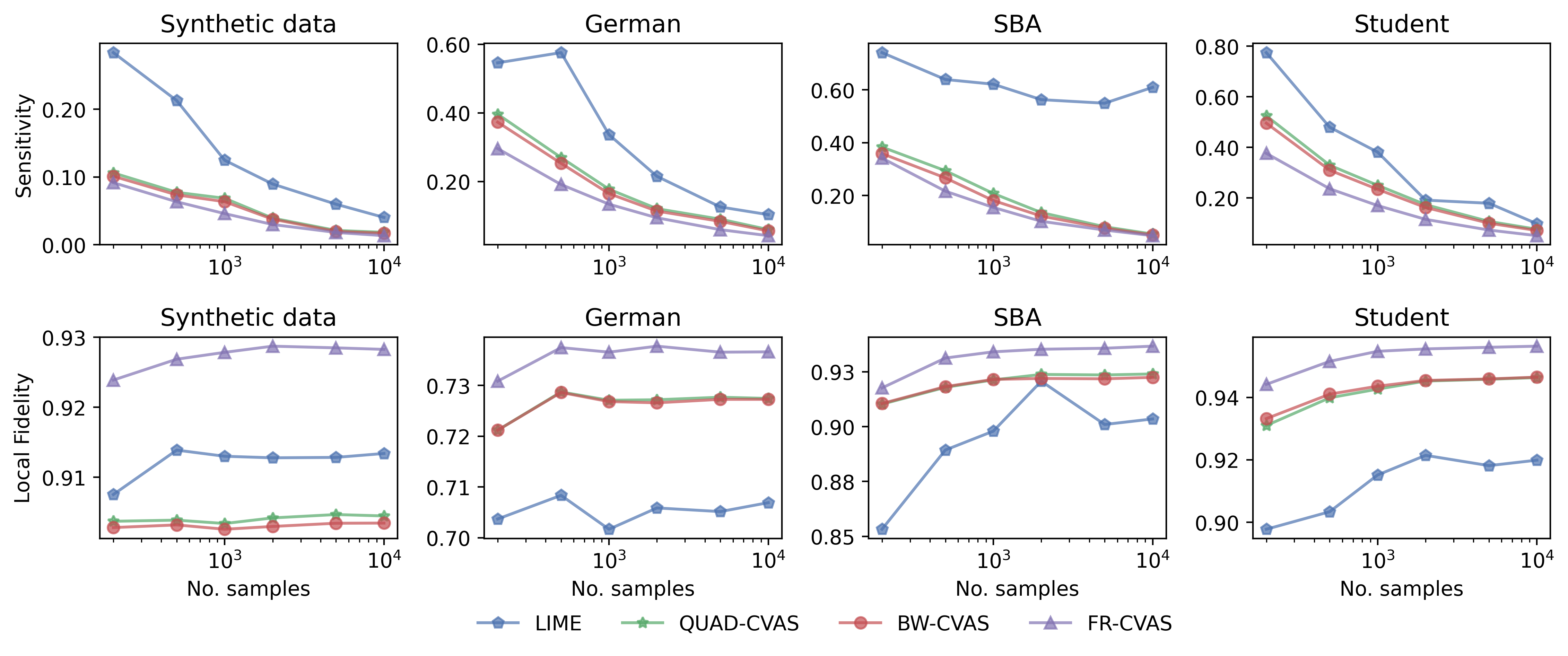}
    \caption{Benchmarks of sensitivity and local fidelity of LIME and CVAS variants on four datasets.}
    \label{fig:fid_stab_sbst_other}
\end{figure}

\begin{figure*}[!t]
    \centering
    \includegraphics[width=\linewidth]{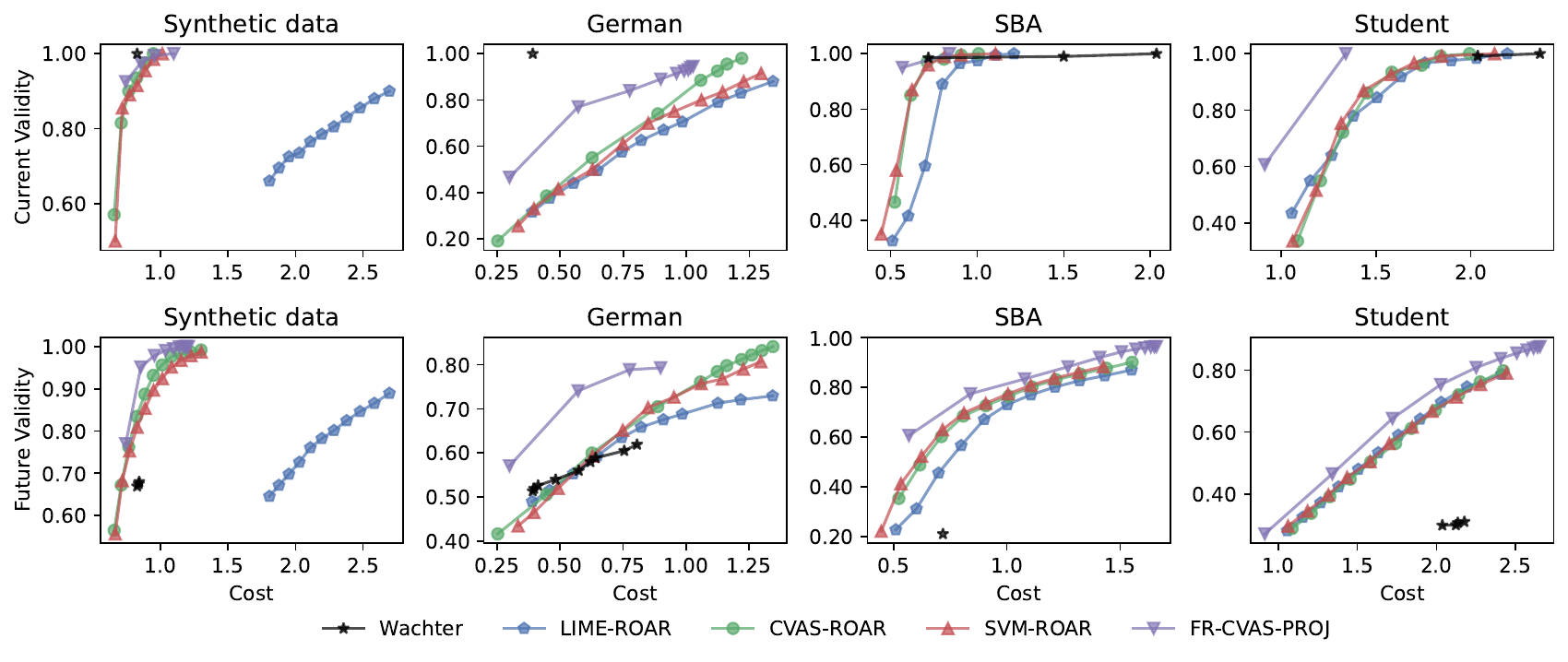}
    \caption{Pareto frontiers of our method compared with ROAR using LIME, non-robust CVAS, and SVM as the surrogate model. The recourses are generated with respect to the MLP classifier on synthetic and three real-world datasets.}
    \label{fig:vanillaexpt}
\end{figure*}

\subsubsection{Comparison with ROAR Using the Non-robust CVAS and SVM as the Surrogate Model.} Here, we compare the FR-CVAS-PROJ with ROAR using LIME, vanilla CVAS, and SVM~\cite{ref:hearst1998support} as the surrogate model. Both vanilla CVAS and SVM use the same boundary sample procedure (with the same seed number) as FR-CVAS. The settings are similar to the experiment in Section~\ref{sec:expt:robust}. Figure~\ref{fig:vanillaexpt} shows that the Pareto frontiers of FR-CVAS dominate the Pareto frontiers of both SVM-ROAR and CVAS-ROAR, demonstrating the merits of our covariance-robust surrogates in balancing cost-validity trade-off.

\begin{figure*}[!ht]
    \centering
    \includegraphics[width=\linewidth]{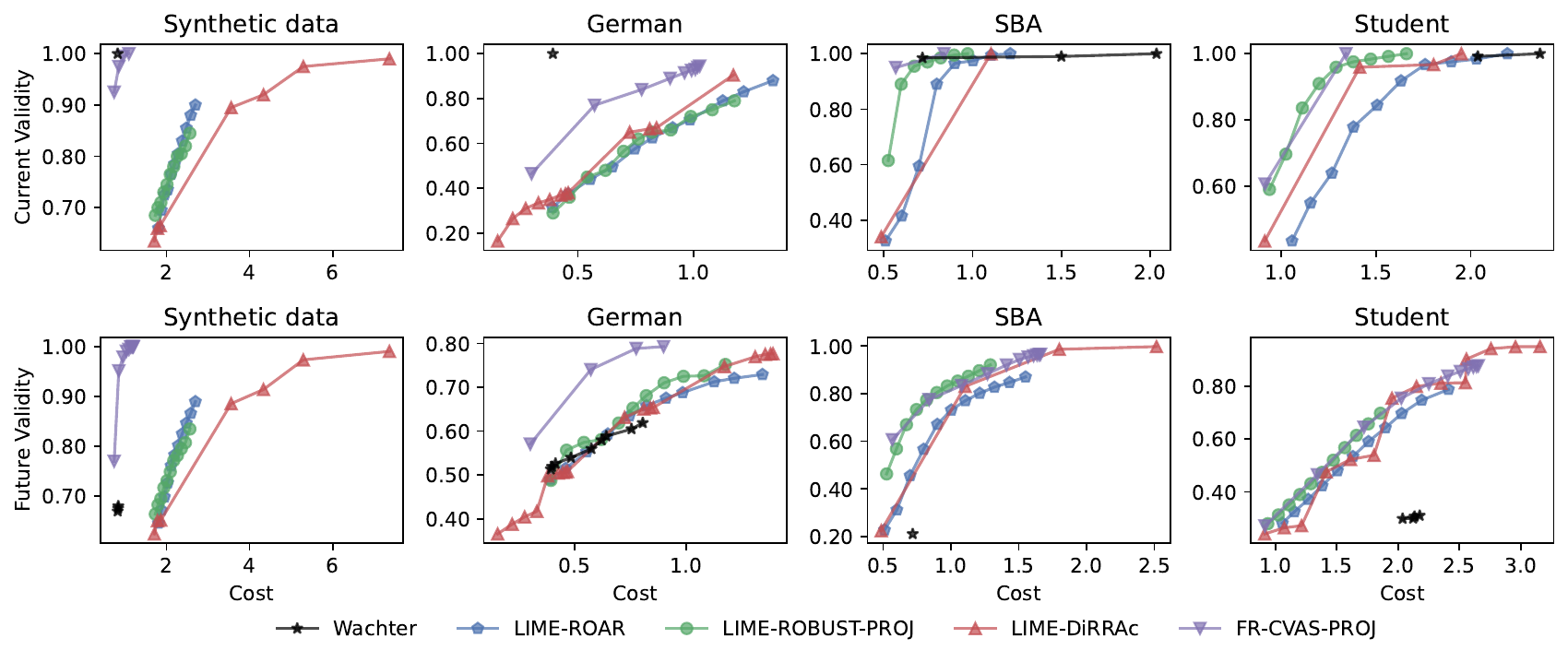}
    \caption{Ablation study: Pareto frontiers of FR-CVAS-PROJ compared to its ablations by alternating the FR-CVAS by LIME and alternating ROAR by DiRRAc and the robust projection. The recourses are generated with respect to the MLP classifier on synthetic and three real-world datasets.}
    \label{fig:ablationexpt}
\end{figure*}

\subsubsection{Ablation Study.} 
We conduct an ablation study to understand the contribution of local surrogate models and recourse-generation approaches in our method. Figure~\ref{fig:vanillaexpt} compares our method with ROAR using vanilla CVAS and SVM as the local surrogate. Figure~\ref{fig:ablationexpt} shows the Pareto frontiers of FR-CVAS-PROJ compared to its ablations by substituting the FR-CVAS with other surrogates (LIME, CVAS) or substituting ROAR~\cite{ref:upadhyay2021towards} by DiRRAc~\cite{ref:nguyen2023distributionally}. We also compare our method with LIME-ROBUST-PROJ, which uses LIME as the surrogate model and then solves the robustified projection: 
\[
    x_r = \arg\min \{ \| x - x_0 \|_1 ~:~ x \in \R^d,~x^\top w + b - \delta_{\mathrm{max}} \|x\|_2 \ge 0 \},
\]
where $(w, b)$ is the weight and bias of the LIME's surrogate, and $\delta_{\mathrm{max}}$ is similar to the uncertainty size of ROAR~\cite{ref:upadhyay2021towards}. 

The recourses are generated with respect to the MLP classifier on synthetic and three real-world datasets. This result demonstrates the usefulness of the FR-CVAS in promoting the generation of robust recourses. Note that, for $\rho_{-1} = 0$ and $\rho_{+1} = 0$, the hyperplane of the FR-CVAS classifier coincides with the vanilla CVAS's hyperplane.

\subsubsection{Comparison with the Probabilistic Threshold Shiftings.} 

A probabilistic classifier uses a threshold to convert the probability to a binary outcome: if the probability value is above the threshold value, the outcome is considered `favorable' (class $+1$). The standard threshold is 0.5, which is used as in~\eqref{eq:wachter}, and one possible approach to generate robust recourse is simply running the recourse generation problem with the same black-box classifier but with a higher threshold value. In doing so, the recourse will have a higher probability of being classified in the $+1$ class with respect to the \textit{current} classifier. This can hopefully be translated to a higher probability of being classified in the $+1$ class with respect to the \textit{future} classifier. While this method may not provide any guarantee, it is a simple and intuitive baseline.

In Figure~\ref{fig:linearshift}, we compare the proposed method with Wachter, LIME, and vanilla CVAS with various probabilistic thresholds in the range $[0.5, 0.9]$. For vanilla CVAS, we first obtain the solution $\wh \theta = (\wh w, \wh b)$ for the coverage-validity-aware surrogate problem~\eqref{eq:surrogate} and calibrate the intercept $\wh b$. Figure~\ref{fig:linearshift} demonstrates that FR-CVAS-PROJ consistently achieves the best performance compared to other baselines. Interestingly, Wachter improves its future validity significantly on the synthetic and German datasets as the threshold increases. However, our method still dominates Wachter in all datasets. These results also suggest that robustification via distributionally robust optimization is more effective than that via the calibration of linear surrogate thresholds.

\begin{figure*}[!ht]
    \centering
    \includegraphics[width=\linewidth]{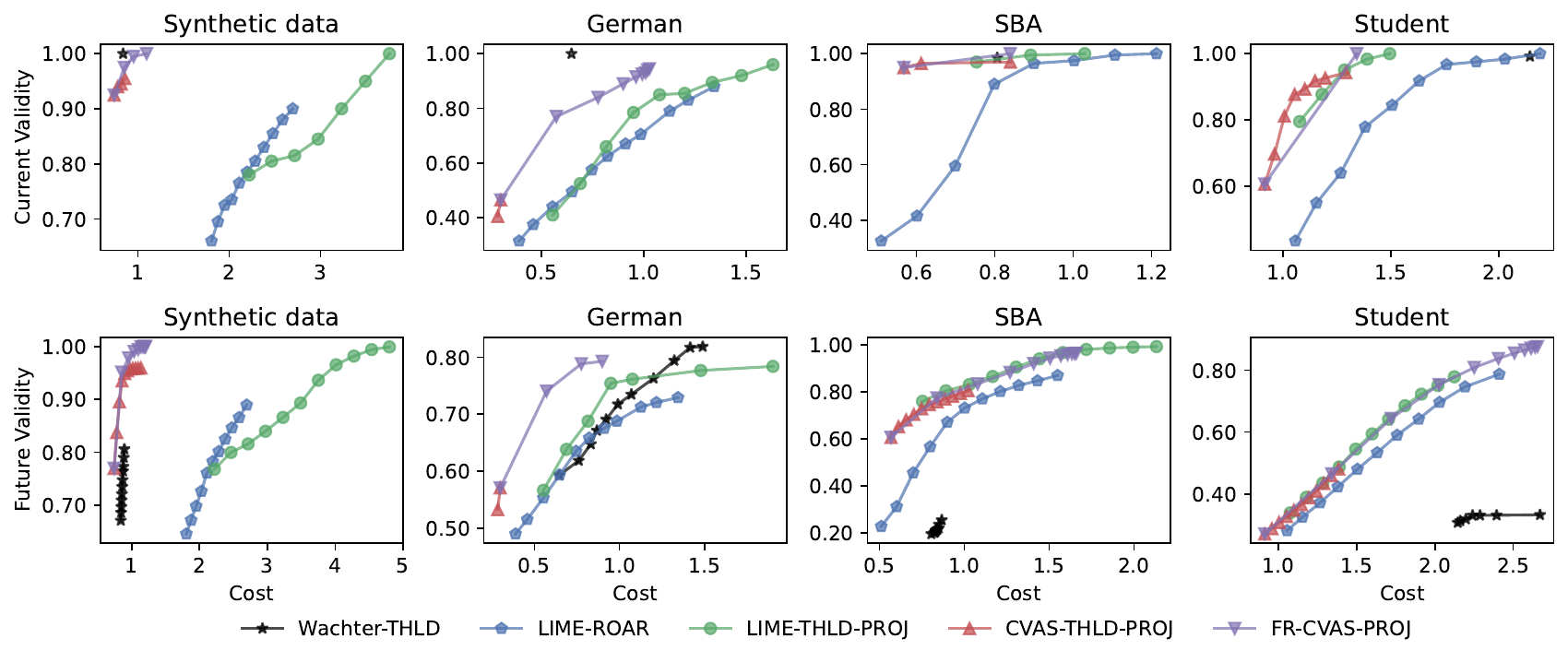}
    \caption{Comparison of FR-CVAS-PROJ, LIME-ROAR, and the probabilistic  shiftings (Wachter-THLD, LIME-THLD-PROJ and CVAS-THLD-PROJ).}
    \label{fig:linearshift}
\end{figure*}

\begin{figure}[ht]
    \centering
    \includegraphics[width=\linewidth]{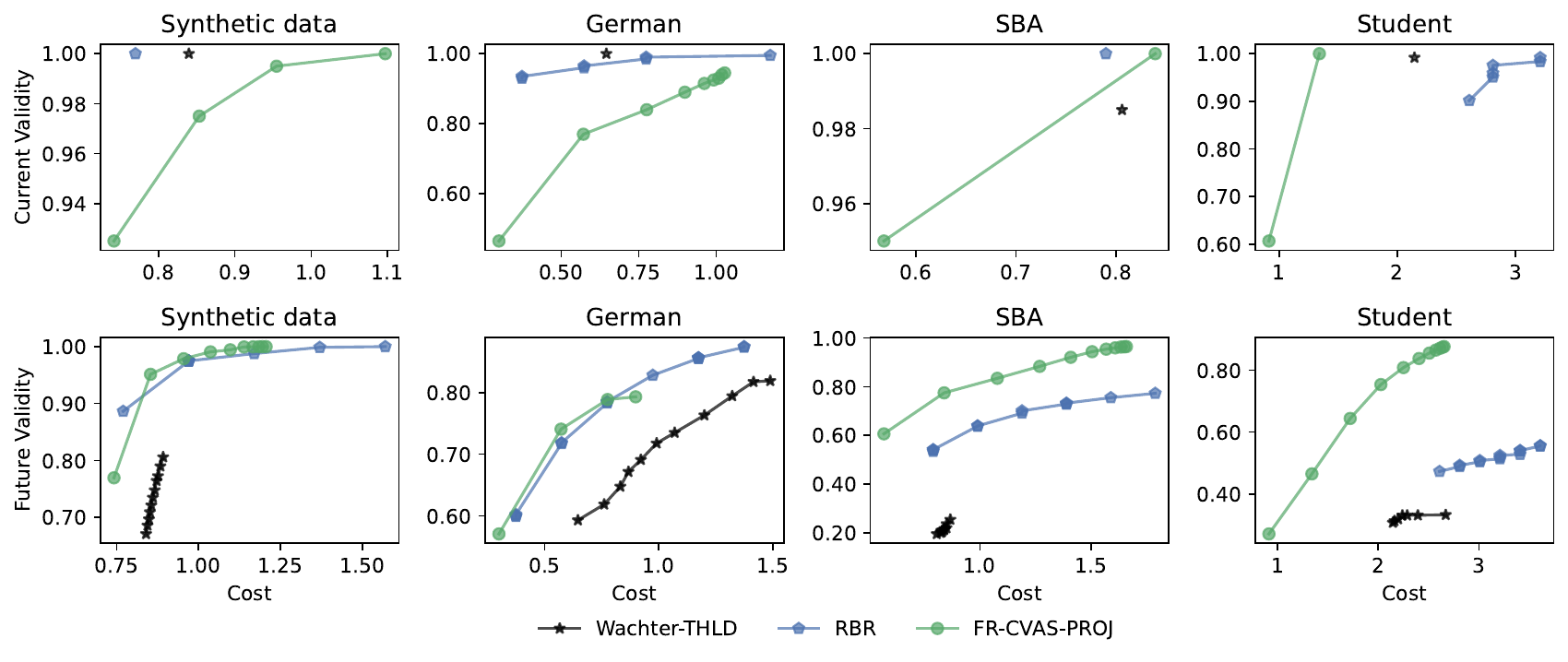}
    \caption{Comparison of our FR-CVAS-PROJ against RBR. We also plot Wachter with threshold shiftings as a comparison in the plot.}
    \label{fig:cost_validity_rbr}
\end{figure}

\subsubsection{Comparison with Model Agnostic Approaches} 
We also compare FR-CVAS-PROJ with Robust Bayesian Recourse (RBR), a model-agnostic approach~\cite{ref:nguyen2022robust}. RBR does not use a surrogate model; instead, it directly optimizes the \textit{quasi-likelihood} of the recourse using the synthesized samples. The frontiers for RBR are obtained by varying RBR's parameters with $\delta_{+} \in [0, 0.2]$ and $\varepsilon_{1} \in [0, 1.0]$. Figure~\ref{fig:cost_validity_rbr} shows that FR-CVAS-PROJ provides a better cost-validity trade-off than RBR, especially on SBA and Student datasets.

\subsubsection{Robust CVASes with Different Divergences} \label{sec:app:disc}

\begin{figure}[ht]
    \centering
    \begin{subfigure}{0.325\linewidth}
        \includegraphics[width=\linewidth]{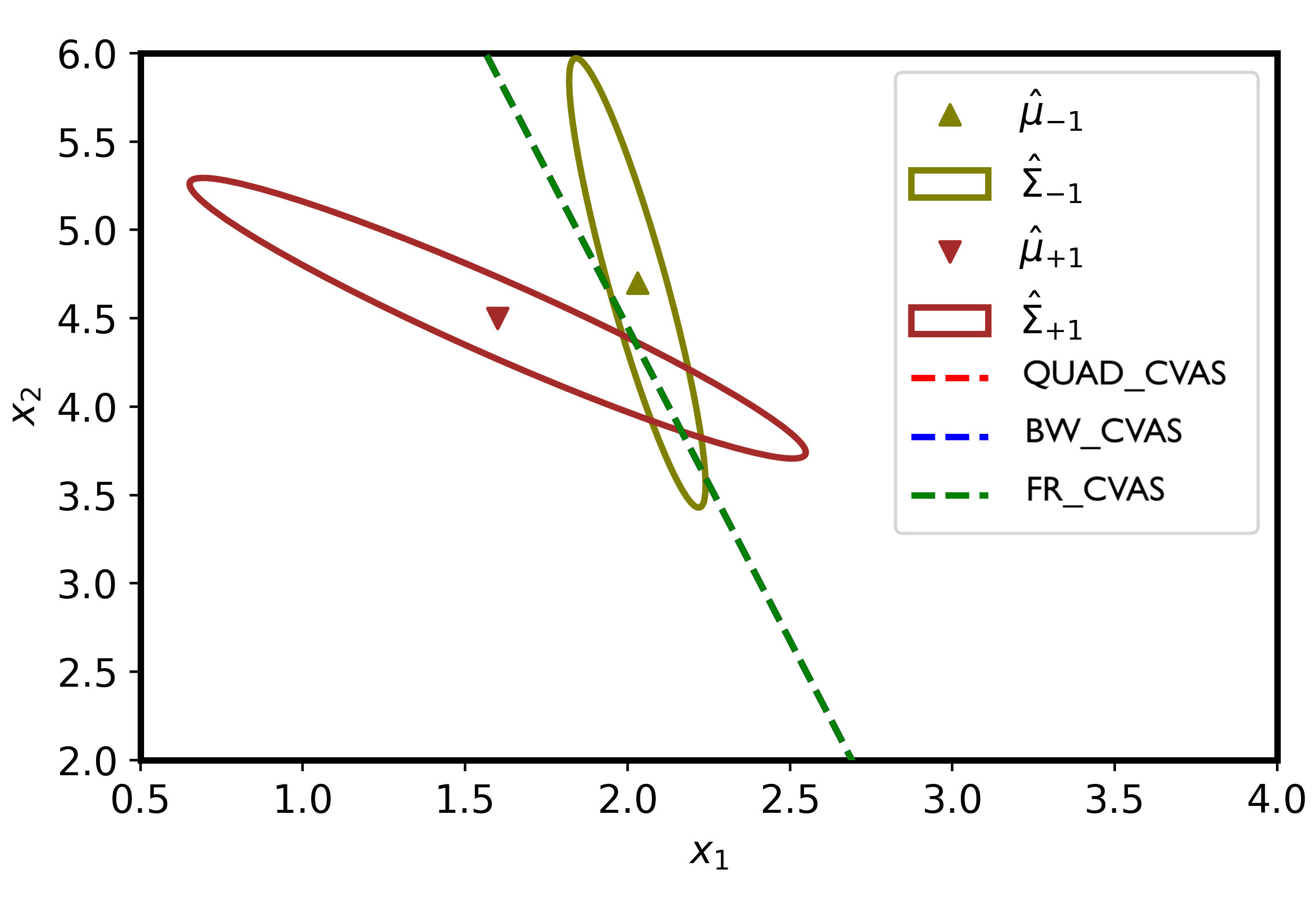}
        \caption{$\rho_{+1} = \rho_{-1} = 0$.}
        \label{fig:asym_ex_0}
    \end{subfigure}
    \begin{subfigure}{0.325\linewidth}
        \includegraphics[width=\linewidth]{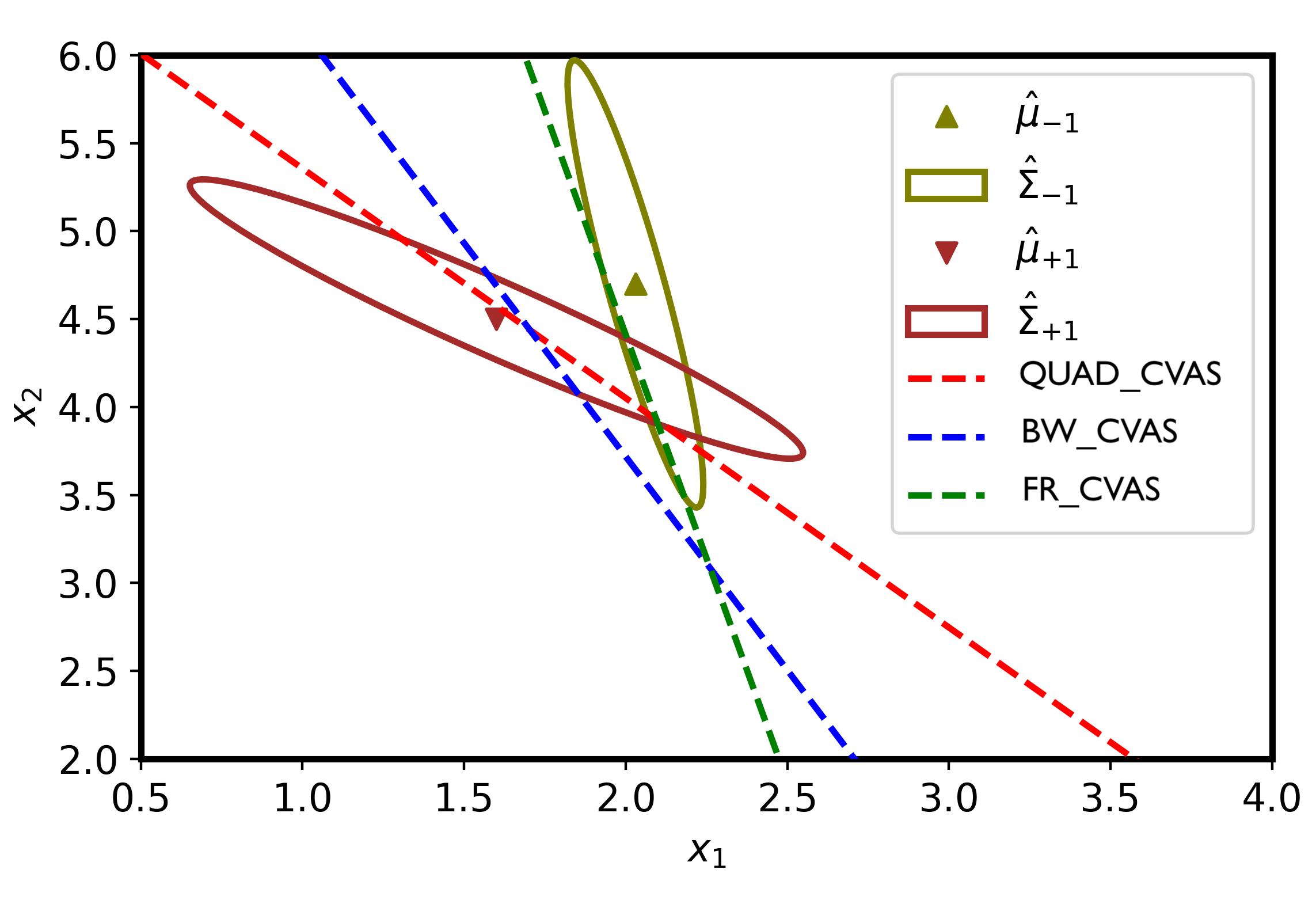}
        \caption{$\rho_{+1} = 0, \rho_{-1} = 1$.}
        \label{fig:asym_ex_1}
    \end{subfigure}
    \begin{subfigure}{0.325\linewidth}
        \includegraphics[width=\linewidth]{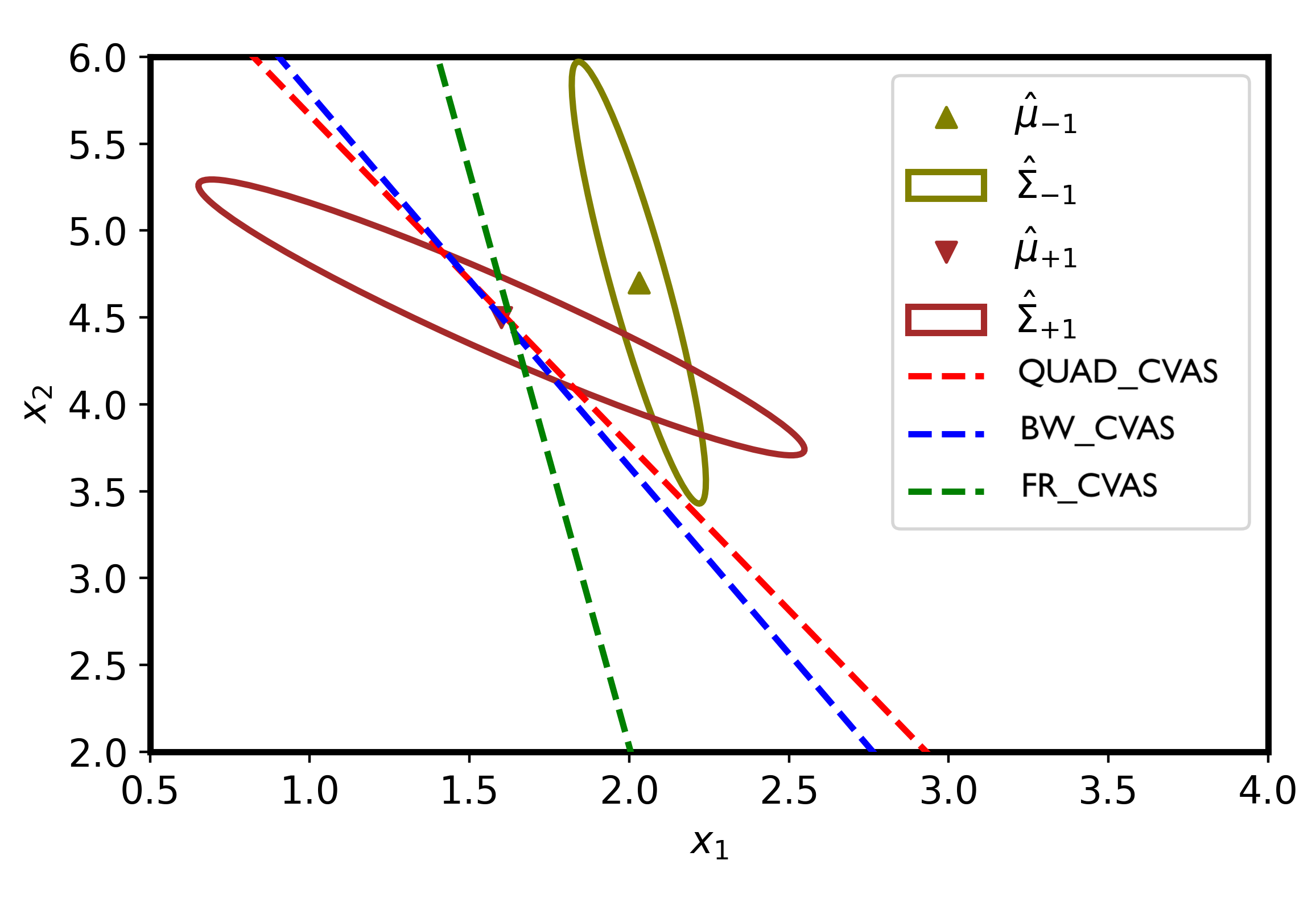}
        \caption{$\rho_{+1} = 0, \rho_{-1} = 10$.}
        \label{fig:asym_ex_2}
    \end{subfigure}
    \caption{Visualization of CVAS's hyperplanes with Quadratic, Bures, and Fisher-Rao distances.}
    \label{fig:asymptotic1}
\end{figure}

This section discusses the variants of covariance-robust CVASes with different divergences. This discussion aims to provide guidance for choosing the surrogate model in practice, especially at a low sample size. 

Proposition~\ref{prop:asymptotic} showed that Quadratic surrogate and Bures surrogate coincide when one of the radii $\rho_y$ grows to infinity, and they are independent of the covariance matrices $\covsa_y$. Meanwhile, the asymptotic hyperplane of the Fisher-Rao surrogate when $\rho_y \to \infty $ aligns with axes of the covariance matrices $\covsa_y$ (see Proposition~\ref{prop:asymptotic} and Figure~\ref{fig:asymptotic1}). We can observe that when $\rho_{+1} = \rho_{-1} = 0$, all hyperplanes coincide and recover the non-robust CVAS. All hyperplanes move towards the favorable class as the radius for the \textit{un}favorable class $\rho_{-1}$ increases. At $\rho_{-1} = 10$ in Subfigure~\ref{fig:asym_ex_2}, the hyperplanes of the Quadratic and Bures surrogates come close together, which is distinct from the Fisher-Rao hyperplane. Notice that the Fisher-Rao surrogate in Subfigure~\ref{fig:asym_ex_2} tends to position in parallel to the major axis of the \textit{un}favorable covariance matrix, which shows the dependence on $\covsa_{-1}$. The Bures and Quadratic hyperplanes in Subfigure~\ref{fig:asym_ex_2} do not depend on the covariance matrix, which aligns with the results in Proposition~\ref{prop:asymptotic}. Coupled with Proposition~\ref{prop:tradeoff}, we postulate that the Fisher-Rao surrogate is not a suitable surrogate at low sample sizes as it relies on the estimate of the covariance matrices. On the other hand, when the number of samples is sufficient to estimate the covariance matrices accurately, the Fisher-Rao surrogate would be better than the Quadratic and Bures surrogate as it considers the geometry of the data during robustification.

To demonstrate our claim above, we probe the performance of CVASes with different divergences at low sample sizes. 

\begin{figure}[ht]
    \centering
    \begin{subfigure}{\linewidth}
        \includegraphics[width=\linewidth]{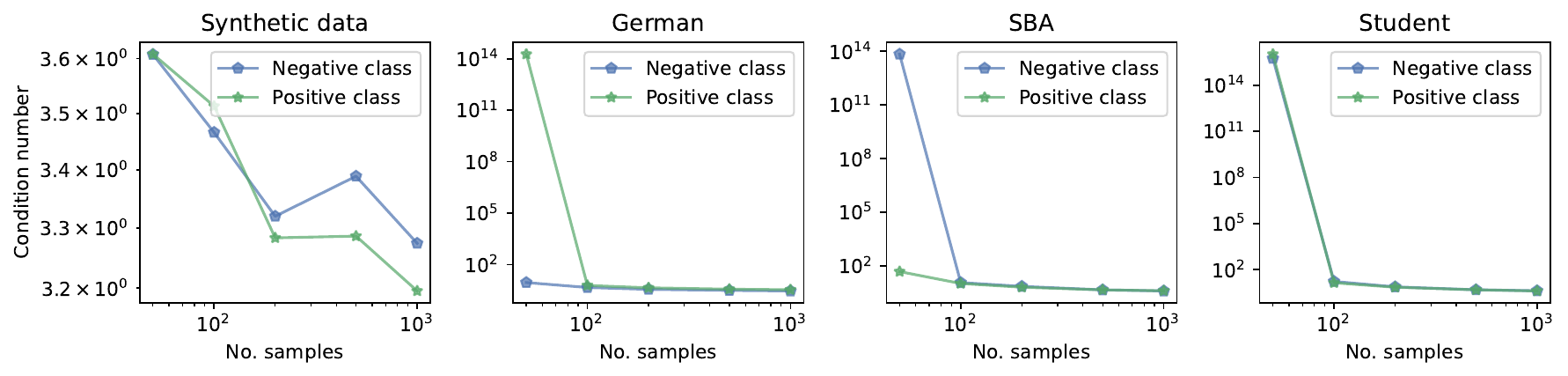}
        \caption{The average condition number of the generated covariance matrices for positive and negative classes.}
        \label{fig:cond_low}
    \end{subfigure}
    \begin{subfigure}{\linewidth}
        \includegraphics[width=\linewidth]{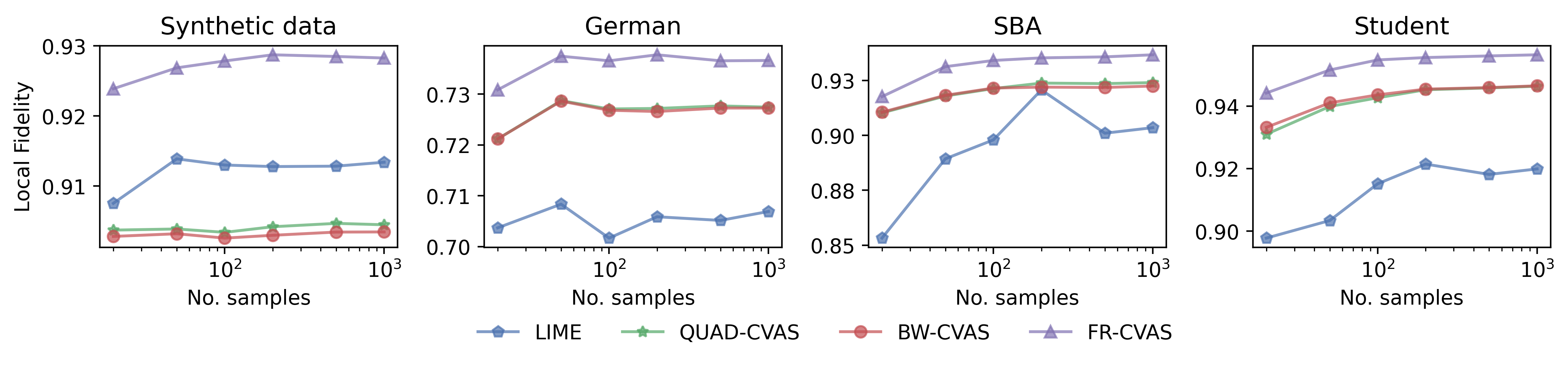}
        \caption{Local fidelity of CVAS variants at low sample sizes.}
        \label{fig:fid_low}
    \end{subfigure}
    \caption{The comparison among CVAS variants with different distances at low sample sizes.}
    \label{fig:cond_fid_low}
\end{figure}

\textbf{Local fidelity.} We probe the local fidelity at low sample sizes and plot the result in Figure~\ref{fig:cond_fid_low}. The experiment settings are similar to those in Section~\ref{sec:expt:fid-stab}. The number of samples is set in the range of $[50, 1000]$. We also measure the average condition number of estimated covariance matrices for both positive and negative classes in Figure~\ref{fig:cond_low}. This figure shows that the covariance matrices are ill-conditioned at 50 samples on SBA and Student datasets. The fidelity of FR-CVAS is slightly better than that of QUAD-CVAS and BW-CVAS. When the number of samples increased, FR-CVAS benefited the most, and the gap between FR-CVAS and QUAD-CVAS (or BW-CVAS) became more significant. It supports our claim that FR-CVAS would better approximate the decision boundary when the number of samples is sufficient for estimating the covariance matrices.

\begin{figure}[ht]
    \centering
    \begin{subfigure}{0.45\linewidth}
        \includegraphics[width=\linewidth]{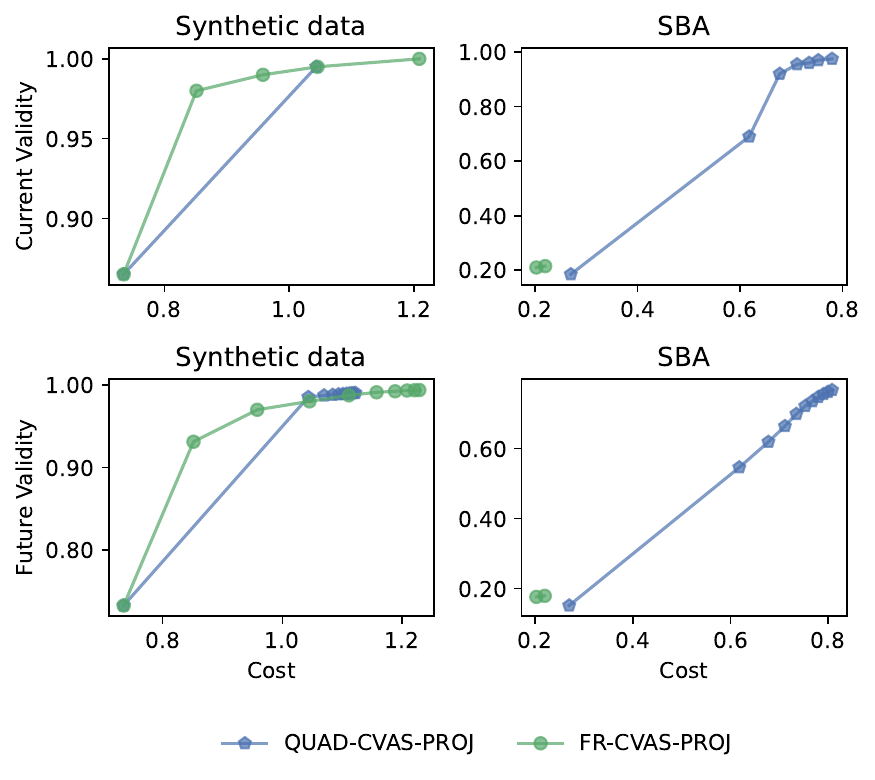}
        \caption{50 synthesized samples.}
        \label{fig:cost-val-50}
    \end{subfigure}
    \hfill
    \begin{subfigure}{0.45\linewidth}
        \includegraphics[width=\linewidth]{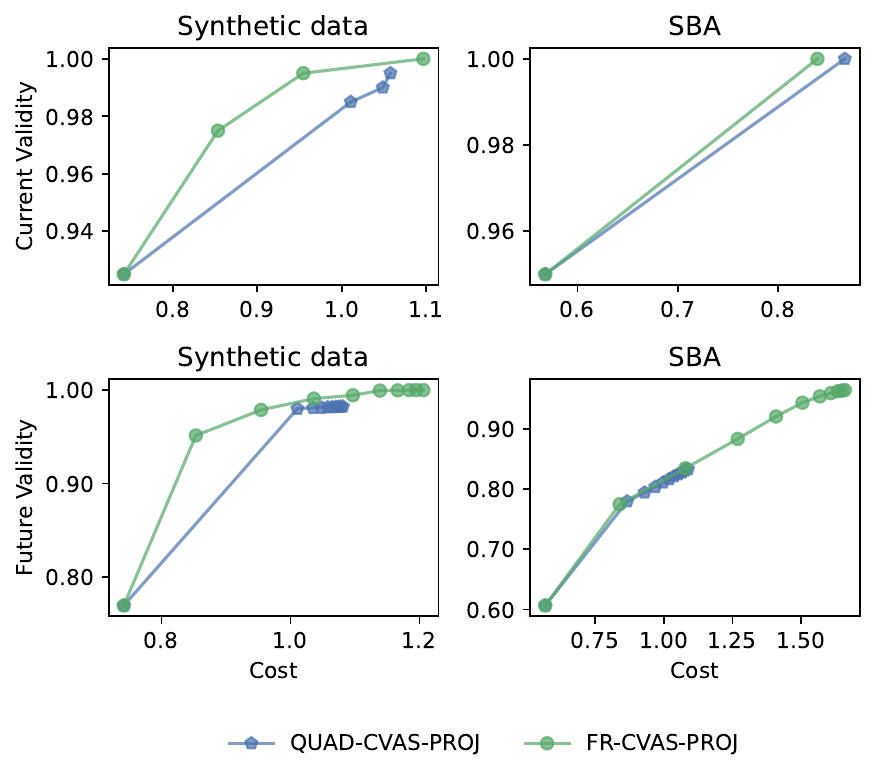}
        \caption{1000 synthesized samples.}
        \label{fig:cost-val-1000}
    \end{subfigure}
    \caption{The comparison of QUAD-CVAS-PROJ and FR-CVAS-PROJ at different sample sizes.}
    \label{fig:cost-val-low}
\end{figure}

\textbf{Robust recourses.} We revisit the recourse generation with covariance-robust CVASes using Quadratic distance (QUAD-CVAS-PROJ) and Fisher-Rao distance (FR-CVAS-PROJ), at which the surrogate is estimated with 50 and 1000 synthesized samples. We omit the comparison with the Bures distance to ease the presentation as it behaves asymptotically like the Quadratic surrogate. The results are shown in Figure~\ref{fig:cost-val-low}. The results showed that QUAD-CVAS-PROJ would be better at a low sample size. When increasing the number of samples, the recourses constructed with the Fisher-Rao surrogate exhibit a better cost-validity trade-off. This result is consistent with our previous observation in the local fidelity experiment.

\subsubsection{Sensitivity to Sampling Radius $r_p$}
To study the sensitivity of sampling radius hyperparameter $r_p$ (step (i)) to the recourse generation phase (step (iii)), we conduct an additional experiment by varying $r_p$. In this experiment, we fix $\rho_{+1} = 0$ and $\rho_{-1} = 10$, similar to the experiments described in Section~\ref{sec:expt:robust}. The results, illustrated in Figure~\ref{fig:r_p-recourse}, show a significant increase in cost for all evaluated datasets as $r_p$ increases. However, current and future validity exhibit different trends: validity decreases for the German dataset but increases for the SBA and Student datasets. 

\begin{figure}[h]
        \centering    
        \includegraphics[width=\linewidth]{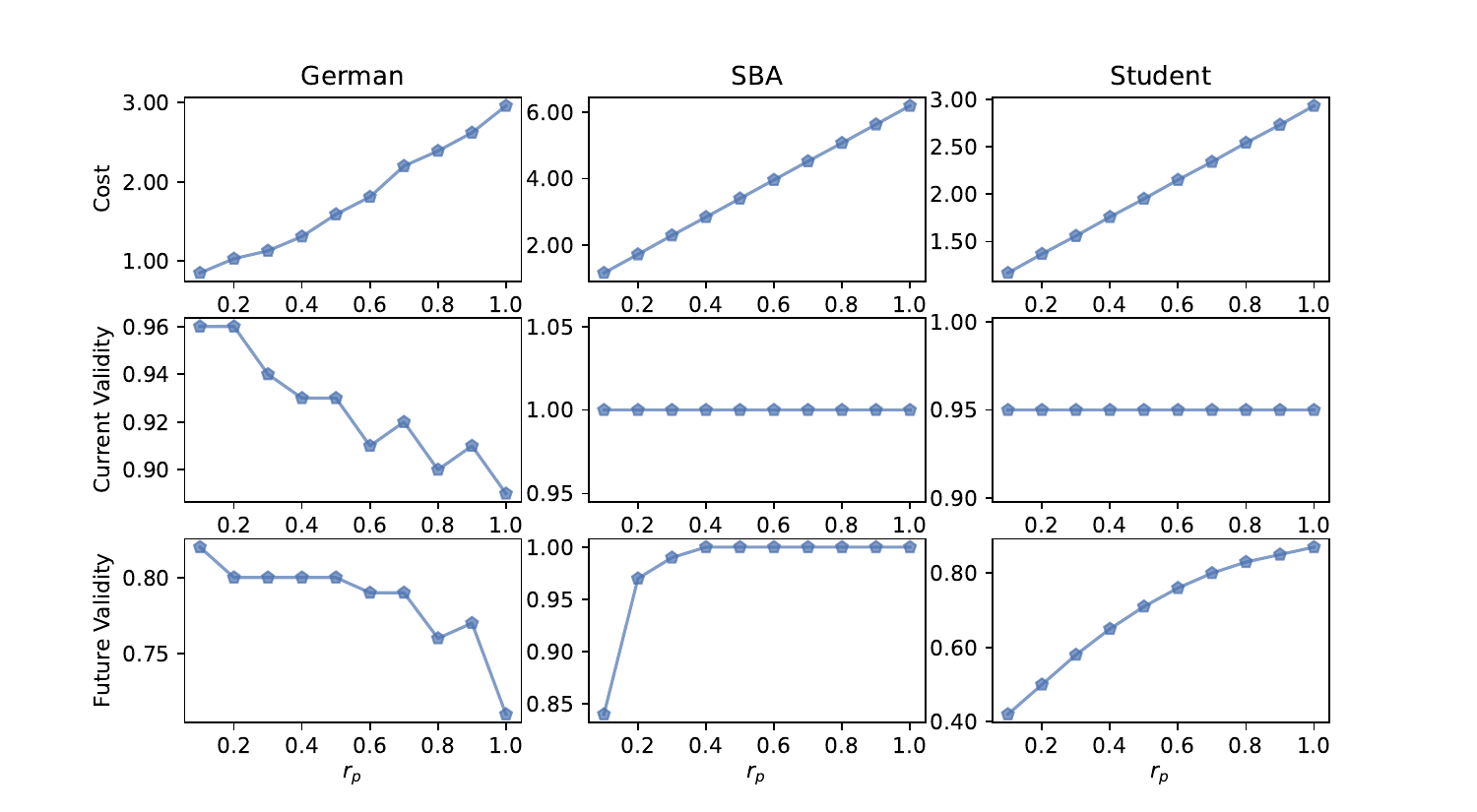}
        \caption{Impact of $r_p$ to cost, current validity, and future validity of recourse.}
        \label{fig:r_p-recourse}
\end{figure}

\section{Proofs}\label{sec:app:proofs}

\subsection{Proofs of Section~\ref{sec:MPM}}

\begin{proof}[Proof of Proposition~\ref{prop:MPM}]
    It follows from \cite[Theorem 10]{ref:bertsimas2000moment} that
    \[
        \Max{\PP_y \sim (\msa_y, \covsa_y)} \PP_y (X \in \mbb H_{\theta, y} ) = \frac{1}{1 + \nu_y^2},
    \]
    where $\nu_y^2 = \inf_{x \in \mbb H_{\theta, y}} (\msa_y - x)^\top \covsa_y^{-1} (\msa_y - x)$ is the squared distance from $\msa_y$ to the set $\mbb H_{\theta, y}$, under the Mahalanobis distance induced by the matrix $\covsa^{-1}_y$. Thus, the equivalence follows from the monotonicity of the square root and negative exponent functions:
    \begin{align*}
        \argmin_{\theta \in \Theta}~\Max{y \in \mc Y}~\Max{\Pnom_y\sim (\msa_y, \covsa_y)}\Pnom_y(\mc C_\theta(X) \neq y) 
        =& \argmin_{\theta \in \Theta} ~\max_{y \in \mc Y} ~\frac{1}{1 + \nu_y^2} \\
        =& \argmax_{\theta \in \Theta} ~\min_{y \in \mc Y} ~\inf_{x \in \mbb H_{\theta, y}} ~(\msa_y - x)^\top \covsa_y^{-1} (\msa_y - x) \\
        =& \argmax_{\theta \in \Theta} ~\min_{y \in \mc Y} ~\inf_{x \in \mbb H_{\theta, y}} ~\sqrt{(\msa_y - x)^\top \covsa_y^{-1} (\msa_y - x)} \\
        =& \argmax_{\theta \in \Theta} ~\min ~\{\mathrm{Va}_{\covsa_{-1}}(\theta),~ \mathrm{Co}_{\covsa_{+1}}(\theta)\}.
    \end{align*}
    Thus, the optimal solution of the CVAS obtained by solving~\eqref{eq:surrogate} coincides with the solution of the MPM problem~\eqref{eq:mpm}.
\end{proof}

\begin{proof}[Proof of Proposition~\ref{prop:refor}]
    Recall that $\mbb H_{\theta, y} = \{x\in\mathbb{R}^d: y (w^\top X - b) \ge 0 \} = \{x\in \mathbb{R}^d: \mc C_\theta (X) = y\}$ and that the ambiguity set is defined as
    \[ \mbb U_{y}^\varphi(\Pnom_y) =  \{ 
            \PP_y :~\PP_y \sim (\msa_y, \cov_y) \text{ for some }\cov_y \in \PSD^d \text{ with }\varphi(\cov_y \parallel \covsa_y) \le \rho_y \},
    \]
    where $\PP_y \sim (\msa_y, \cov_y)$ means that the the distribution $\PP_y$ has mean $\msa_y$ and covariance $\cov_y$. 
    In other words, each element $ \PP_y $ in the ambiguity $\mbb U_{y}^\varphi(\Pnom_y)$ is determined by first choosing a covariance matrix $\cov_y$ satisfying the divergence constraint $\varphi(\cov_y \parallel \covsa_y) \le \rho_y$ and then picking a distribution $\PP_y$ having mean $\msa_y$ and covariance $\cov_y$. 
    Therefore, the worst-case probability admits a two-layer decomposition
    \be\label{eq:max-decompose}
        \Max{\PP_y \in \mbb U_{y}^\varphi(\Pnom_y)}~\PP_y( \mc C_\theta(X) \neq y) =  \Max{\cov_y \in \PSD^d: \varphi(\cov_y \parallel \covsa_y) \le \rho_y}~\Max{\PP_y \sim (\msa_y, \cov_y)}~\PP_y(\mc C_\theta(X) \neq y).
    \ee
    Hence,
    \begin{align*}
        &\, \Max{y \in \mc Y}~\Max{\PP_y \in \mbb U_{y}^\varphi(\Pnom_y)}~\PP_y(\mc C_\theta(X) \neq y)\\
        = &\, \Max{y \in \mc Y}~\Max{\cov_y \in \PSD^d:\, \varphi(\cov_y \parallel \covsa_y) \le \rho_y}~\Max{\PP_y\sim (\msa_y, \cov_y)}~\PP_y(X \in \mbb H_{\theta, -y})\\
        = &\, \left(1 + \Min{y \in \mc Y}~\Min{\cov_y \in \PSD^d:\, \varphi(\cov_y \parallel \covsa_y) \le \rho_y}~\inf_{x \in \mbb H_{\theta, -y}} (\msa_y - x)^\top \cov_y^{-1} (\msa_y - x) \right)^{-1} \\
        = &\, \Bigg(1 + \min\Bigg\lbrace \Min{\cov_{+1} \in \PSD^d:\, \varphi(\cov_{+1} \parallel \covsa_{+1}) \le \rho_{+1}}~\inf_{x \in \mbb H_{\theta, {-1}}} (\msa_{+1} - x)^\top \cov_{+1}^{-1} (\msa_{+1} - x) ,\\
        &\,\qquad\qquad\qquad \Min{\cov_{-1} \in \PSD^d:\, \varphi(\cov_{-1} \parallel \covsa_{-1}) \le \rho_{-1}}~\inf_{x \in \mbb H_{\theta, {+1}}} (\msa_{-1} - x)^\top \cov_{-1}^{-1} (\msa_{-1} - x)\Bigg\rbrace  \Bigg)^{-1}\\
        = &\, \left( 1 + \min\left\{ \Min{\cov_{+1} \in \PSD^d:\, \varphi(\cov_{+1} \parallel \covsa_{+1}) \le \rho_{+1}}~\mathrm{Co}_{\cov_{+1}}(\theta), \Min{\cov_{-1} \in \PSD^d:\, \varphi(\cov_{-1} \parallel \covsa_{-1}) \le \rho_{-1}}~\mathrm{Va}_{\cov_{-1}}(\theta) \right\}\right)^{-1} \\
        = &\, \left( 1 + \Min{\cov_y \in \PSD^d:\, \varphi(\cov_y \parallel \covsa_y) \le \rho_y~\forall y\in \mc Y}~\min\left\{ \mathrm{Va}_{\cov_{-1}}(\theta),~\mathrm{Co}_{\cov_{+1}}(\theta)\right\}\right)^{-1},
    \end{align*}
    where the first equality follows the definition of $\mbb H_{\theta, y}$ and equality~\eqref{eq:max-decompose}, the second from~\cite[Theorem 10]{ref:bertsimas2000moment}, the third from the fact that $\mc Y = \{+1, -1\}$, the fourth from the definitions of $\mathrm{Co}_{\cov_{+1}}(\theta)$ and $\mathrm{Va}_{\cov_{-1}}(\theta)$, and the last from the fact that the minimization over $\cov_{+1}$ and $\cov_{-1}$ is separable. Thus, the optimization problems
    \[
        \Min{\theta \in \Theta}~\Max{y \in \mc Y}~\Max{\PP_y \in \mbb U_{y}^\varphi(\Pnom_y)}~\PP_y(\mc C_\theta(X) \neq y) \quad\text{and}\quad \Max{\theta \in \Theta}~\Min{\cov_y \in \PSD^d: \varphi(\cov_y \parallel \covsa_y) \le \rho_y~\forall y}~\min\left\{~\mathrm{Co}_{\cov_{+1}}(\theta), \mathrm{Va}_{\cov_{-1}}(\theta)\right\} ,
    \]
    share the same optimal solutions.
    In other words, the hyperplane obtained by solving~\eqref{eq:dro} coincides with the MPM under probability misspecification obtained by solving~\eqref{eq:dro_mpm}. 
    
    The remainder of this proof requires showing the optimal solution of the problem~\eqref{eq:dro_mpm}. 
    Towards that end, we note that
    \begin{equation}
        \begin{split}\label{eq:proof10}
            &\, \Min{\theta \in \Theta}~\Max{y \in \mc Y}~\Max{\PP_y \in \mbb U_{y}^\varphi(\Pnom_y)}~\PP_y(\mc C_\theta(X) \neq y) \\
            = &\, \Min{w\neq 0,\, b}~\Max{y\in \mc Y}~\Max{\PP_y \in \mbb U_{y}^\varphi(\Pnom_y)} \mbb P_y (y(w^\top X - b) \le 0) \\
            = &\, \Min{w\neq 0,\, b}~\Max{y\in \mc Y}~\Max{\cov_y \in \PSD^d: \varphi(\cov_y \parallel \covsa_y) \le \rho_y}~\Max{\PP_y\sim (\msa_y, \cov_y)}~\mbb P_y (y(w^\top X - b) \le 0) \\
            = &\, \Min{w\neq 0,\, b}~\Max{y\in \mc Y}~\Max{\cov_y \in \PSD^d: \varphi(\cov_y \parallel \covsa_y) \le \rho_y}~\left( 1 + \frac{(b - w^\top \msa_y )^2}{w^\top \cov_y w} \right)^{-1} \\
            = &\, \Min{w\neq 0,\, b}~\Max{y\in \mc Y}~\left( 1 + \frac{(b - w^\top \msa_y )^2}{\Max{\cov_y \in \PSD^d: \varphi(\cov_y \parallel \covsa_y) \le \rho_y}~w^\top \cov_y w} \right)^{-1} \\
            = &\, \Min{w\neq 0,\, b}~\Max{y\in \mc Y}~\left( 1 + \frac{(b - w^\top \msa_y )^2}{ (\tau_y^\varphi(w))^2 } \right)^{-1} \\
            = &\, \left( 1 + \left(\Max{w\neq 0,\, b}~\Min{y\in \mc Y}~\frac{\sign(b - w^\top \msa_y)(b - w^\top \msa_y) }{ \tau_y^\varphi(w) } \right)^2 \right)^{-1} ,
        \end{split}
    \end{equation}
    where the first equality follows from the definition of the classification rule $\mc C_\theta (X)$, the second from the decomposition~\eqref{eq:max-decompose}, the third from~\cite[Equation~(6)]{ref:lanckriet2001minimax}, the fourth from the fact that the map $t \mapsto (1+ t^{-1})^{-1}$ is monotonically increasing, the fifth from the definition of $\tau_y^\varphi(w)$, and the sixth from the fact that the map $t\mapsto (1 + t)^{-1}$ is monotonically decreasing. Using the same argument as in~\cite{ref:lanckriet2001minimax} (see equation~(4) and the discussions following it in~\cite{ref:lanckriet2001minimax}), we can show that the optimal $\theta = (w,b)$ must classify $\msa_y$ correctly, \ie, $y = \sign( w^\top \msa_y - b )$. Therefore, the max-min problem in the last line in the last display becomes
    \begin{equation}\label{eq:proof2}
        \Max{w\neq 0,\, b}~\Min{y\in \mc Y}~\frac{(w^\top \msa_y - b)y }{ \tau_y^\varphi(w) } ,
    \end{equation}
    which is equivalent to
    \be
    \label{opt:1}
        \begin{array}{cl}
            \max & \kappa \\
            \st & \kappa \in \R_+,~w \in \R^d\setminus\{0\},~b \in \R \\
            & y(w^\top \msa_y -b ) \ge \kappa \, \tau_y^\varphi(w) \qquad \forall y \in \mc Y.
        \end{array}
    \ee
    From the constraints, we get
    \begin{equation}\label{ineq:proof1}
        w^\top \msa_{+1} - \kappa \,\tau_{+1}^\varphi(w)  \ge b  \ge   w^\top\msa_{-1} + \kappa \,\tau_{-1}^\varphi(w).
    \end{equation}
    Since the objective value does not depend of $b$, we can eliminate the variable $b$ and reduce problem~\eqref{opt:1} to
    \be
    \label{opt:2}
        \begin{array}{cl}
            \max & \kappa \\
            \st & \kappa \in \R_+,~w \in \R^d\setminus\{0\} \\
            & w^\top \msa_{+1} - \kappa \,\tau_{+1}^\varphi(w)    \ge   w^\top\msa_{-1} + \kappa \,\tau_{-1}^\varphi(w).
        \end{array}
    \ee
    We claim that one can add the extra constraint $\sum_{y\in \mc Y } y\, w^\top\msa_y = 1$ to the problem without affecting the optimal value. Too see this, we first note that $\tau_y^\varphi(w)$ is positively homogeneous in $w$, \ie, $ \tau_y^\varphi(t\, w) = |t|\, \tau_y^\varphi( w) $ for any $t\in \mathbb{R}$. If $(\kappa^\star, w^\star)\in \mathbb{R}_+\times (\mathbb{R}^d\setminus \{0\})$ is an optimal solution, then the pair $ (\kappa^\star, t^\star w^\star )$ with $t^\star = (\sum_{y\in \mc Y } y\, {w^\star}^\top\msa_y)^{-1}$ has the same (optimal) objective value since the objective function is $\kappa$. Also, the pair $ (\kappa^\star, t^\star w^\star )$ satisfies the inequality constraint of \eqref{opt:2} since
    \begin{align*}
        &\, t^\star {w^\star}^\top \msa_{+1} - \kappa^\star \,\tau_{+1}^\varphi \left( t^\star w^\star  \right) \ge 
    t^\star\left( {w^\star}^\top \msa_{+1} - \kappa^\star \,\tau_{+1}^\varphi \left( w^\star  \right)\right) \\
    \ge &\,
    t^\star\left( {w^\star}^\top \msa_{-1} - \kappa^\star \,\tau_{-1}^\varphi \left( w^\star  \right)\right) \ge
    t^\star {w^\star}^\top\msa_{-1} + \kappa^\star \,\tau_{-1}^\varphi(t^\star w^\star) .
    \end{align*}
    So, $ (\kappa^\star, t^\star w^\star )$ is also an optimal solution to problem~\eqref{opt:2}. On the other hand, this optimal solution satisfies that
    \[
    \sum_{y\in \mc Y } y\, t^\star {w^\star}^\top\msa_y = (\sum_{y\in \mc Y } y\, {w^\star}^\top\msa_y)^{-1} \sum_{y\in \mc Y } y\,  {w^\star}^\top\msa_y = 1.
    \]
    This proves the claim. Hence, problem~\eqref{opt:2} is further equivalent to 
    \be
    \label{opt:3}
        \begin{array}{cl}
            \max & \kappa \\
            \st & \kappa \in \R_+,~w \in \R^d\setminus\{0\} \\
            & w^\top \msa_{+1} - \kappa \,\tau_{+1}^\varphi(w)    \ge   w^\top\msa_{-1} + \kappa \,\tau_{-1}^\varphi(w)\\
            & \displaystyle \sum_{y\in \mc Y } y\, w^\top\msa_y = 1.
        \end{array}
    \ee
    The inequality constraint in problem~\eqref{opt:3} is equivalent to
    \be\label{ineq:proof2}
        \kappa \le \frac{\sum_{y\in \mc Y } y\, w^\top \msa_y}{\sum_{y\in\mc Y} \tau_y^\varphi(w) }.
    \ee
    Thus, we can eliminate the variable $\kappa$ and rewrite problem~\eqref{opt:3} as
    \[
           \min~\left\{\displaystyle\sum_{y\in\mc Y} \tau_y^\varphi(w) ~:~ w \in \R^d,~\displaystyle\sum_{y\in \mc Y } y\, w^\top\msa_y = 1 \right\}.
    \]
    Finally, note that from \eqref{ineq:proof1} and \eqref{ineq:proof2}, at optimality, we have
    \[
    \kappa = \frac{ \sum_{y\in \mc Y } y\, w^\top \msa_y}{ \sum_{y\in\mc Y} \tau_y^\varphi(w) } = \frac{1}{\sum_{y\in\mc Y} \tau_y^\varphi(w)} ,
    \]
    and 
    \[ b = w^\top \msa_{+1} - \kappa \,\tau_{+1}^\varphi(w)  = w^\top\msa_{-1} + \kappa \,\tau_{-1}^\varphi(w). \]
    This completes the proof.
\end{proof}

\begin{proof}[Proof of Proposition~\ref{prop:gauss}]
    Let $\Phi$ be the cumulative distribution function of the standard Gaussian random variable, we have
    \begin{equation}\label{eq:proof1}
        \begin{split}
            &\, \Min{\theta \in \Theta}~\Max{y \in \mc Y}~\Max{\PP_y \in \mc U_{y}^{\mc N}(\Pnom_y)}~\PP_y(\mc C_\theta(X) \neq y) \\
            = &\, \Min{w\neq 0,\, b}~\Max{y\in \mc Y}~\Max{\PP_y \in \mc U_{y}^{\mc N}(\Pnom_y)}~\mbb P_y (y(w^\top X - b) \le 0) \\
            = &\, \Min{w\neq 0,\, b}~\Max{y\in \mc Y}~\Max{\substack{\cov_y \in \PSD^d:\, \varphi(\cov_y \parallel \covsa_y) \le \rho_y \\ \PP_y \sim \mathcal{N}(\msa_y, \cov_y)}}~\mbb P_y (y(w^\top X - b) \le 0) \\
            = &\, \Min{w\neq 0,\, b}~\Max{y\in \mc Y}~\Max{\cov_y \in \PSD^d:\, \varphi(\cov_y \parallel \covsa_y) \le \rho_y }~1 - \Phi\left( \frac{y(w^\top \msa_y -b)}{\sqrt{w^\top \cov_y w}} \right)  ,
        \end{split}
    \end{equation}
    where the first equality follows from the definition of the classification rule $\mc C_\theta (X)$, the second from the definition of the Gaussian ambiguity set $\mc U_{y}^{\mc N}(\Pnom_y)$, the third from the elementary fact that for Gaussian distribution $\mbb P_y\sim \mc N (\msa_u, \cov_y)$ the probability is given by
    \[
        \PP_y( y (w^\top X - b) \le 0) = 1 - \Phi\left( \frac{y(w^\top \msa_y -b)}{\sqrt{w^\top \cov_y w}} \right).
    \]
    We claim that $y(w^\top \msa_y -b) \ge 0$ at optimal $\theta^\star = (w^\star, b^\star)$. To see this, assume $y({w^\star}^\top \msa_y -b^\star) < 0$ for some $y\in \mc Y$. Then the optimal value is strictly bigger than $\frac{1}{2}$, but the pair $(-y \, \sign({w^\star}^\top (\msa_{-y} - \msa_{y})) w^\star, -y \, \sign({w^\star}^\top (\msa_{-y} - \msa_{y})) {w^\star}^\top \msa_y)$ would yield an objective value strictly smaller than $\frac{1}{2}$.
    Therefore,
    \begin{align*}
        &\,\Min{\theta \in \Theta}~\Max{y \in \mc Y}~\Max{\PP_y \in \mc U_{y}^{\mc N}(\Pnom_y)}~\PP_y(\mc C_\theta(X) \neq y) \\
        = &\, \Min{w\neq 0,\, b}~\Max{y\in \mc Y}~1 - \Phi\left( \frac{y(w^\top \msa_y -b)}{\sqrt{\Max{\cov_y \in \PSD^d:\, \varphi(\cov_y \parallel \covsa_y) \le \rho_y }~w^\top \cov_y w}} \right) \\
        = &\, \Min{w\neq 0,\, b}~\Max{y\in \mc Y}~1 - \Phi\left( \frac{y(w^\top \msa_y -b)}{\tau_y^\varphi(w)} \right) .
    \end{align*}
    where the first equality follows from~\eqref{eq:proof1} and the monotonicity of the map $t \mapsto 1 - \Phi (y(w^\top \msa_y -b)/\sqrt{t})$ established above and the second from the definition of $\tau_y^\varphi$. Since the cumulative distribution function $\Phi$ is monotonically increasing, the last min-max problem is equivalent to problem~\eqref{eq:proof2}, which from the proof of Proposition~\ref{prop:refor}, is equivalent to both problems~\eqref{eq:dro} and~\eqref{eq:dro_mpm}. Hence, problem~\eqref{eq:dro-Gauss} shares the same optimal solution as problem~\eqref{eq:dro}.
    This completes the proof.
\end{proof}

\subsection{Proofs of Section~\ref{sec:bures}}

We first prove Proposition~\ref{prop:bures} to lay the foundation for the proof of Theorem~\ref{thm:bures}.
\begin{proof}[Proof of Proposition~\ref{prop:bures}]
    By~\cite[Proposition~2.8]{ref:nguyen2018distributionally}, we have
    \[
        \tau_y^{\mathds B}(w)^2 = \Inf{ \dualvar I \succ ww^\top }~\dualvar (\rho_y - \Tr{\covsa_y}) + \dualvar^2 \inner{(\dualvar I - ww)^{-1}}{\covsa_y}.
    \]
    Using the Sherman-Morrison formula~\cite[Corollary~2.8.8]{ref:bernstein2009matrix}, we find
    \[
        (I - \frac{1}{\gamma}ww)^{-1} = I + \frac{ww^\top}{\gamma - \| w \|_2^2}.
    \]
    Notice that the constraint $\dualvar I \succ ww^\top$ is equivalent to $\dualvar > \| w\|_2^2$ by Schur complement. Thus, we have
    \[
        \tau_y^{\mathds B}(w)^2 =\Inf{\gamma > \| w \|_2^2}~ \gamma \rho_y + \gamma \frac{w^\top \covsa_y w}{\gamma - \| w\|_2^2}.
    \]
    Let $h(\gamma)$ be the objective function in the last display. Then,
    \begin{equation*}
        h' (\gamma) = \rho_y + \frac{w^\top \covsa_y w}{\gamma - \| w\|_2^2} - \gamma \frac{w^\top \covsa_y w}{(\gamma - \| w\|_2^2)^2} = \rho_y - \frac{ \| w\|_2^2 w^\top \covsa_y w }{(\gamma - \| w\|_2^2)^2},
    \end{equation*}
    and 
    \begin{equation*}
        h''(\gamma) = \frac{ 2 \| w\|_2^2 w^\top \covsa_y w }{(\gamma - \| w\|_2^2)^3} > 0 .
    \end{equation*}
    So, $h$ is a strictly convex function on $(\|w\|_2^2, +\infty)$. Since $h'(\gamma) < 0 $ as $\gamma \downarrow \|w\|_2^2$, the infimum is not attained at $\gamma = \|w\|_2^2$. The minimizer can be found by solving the first-order condition
    \[
        0 = h'(\gamma) = \rho_y - \frac{ \| w\|_2^2 w^\top \covsa_y w }{(\gamma - \| w\|_2^2)^2},
    \]
    which is equivalent to 
    \[
        \rho_y (\gamma - \| w\|_2^2)^2 = \| w\|_2^2 w^\top \covsa_y w.
    \]
    Thus, the optimal $\gamma$ is given by
    \[
        \gamma\opt = \| w\|_2^2 + \sqrt{\frac{w^\top \covsa_y w \|w \|_2^2}{\rho_y}},
    \]
    with the corresponding optimal value
    \[
        \tau_y^{\mathds B}(w)^2 = h(\gamma^\star) = \left(\rho_y \| w \|_2 + \sqrt{w^\top \covsa_y w} \right)^2.
    \]
    We thus have the desired result.
\end{proof}

We now prove Theorem~\ref{thm:bures}.
\begin{proof}[Proof of Theorem~\ref{thm:bures}]
    Using the Bures divergence $\mathds B$, the optimization problem
    \[
    \Min{w \in \mc W}~ \sum\nolimits_{y \in \mc Y} \tau_y^{\mathds B}(w)
    \]
    becomes problem~\eqref{eq:bures} by exploiting the analytical form of $\tau_y^{\mathds B}(w)$ in Proposition~\ref{prop:bures}. By invoking Proposition~\ref{prop:refor}, we obtain the postulated results on the optimal solution $\theta^{\mathds B}$  for the case of the Bures divergence.
\end{proof}

\subsection{Proofs of Section~\ref{sec:fr}}

We first provide the proof of Proposition~\ref{prop:FR}.

\begin{proof}[Proof of Proposition~\ref{prop:FR}]
Notice that
\[
    \tau_y^{\FR}(w)^2 = \max\left\{
    w^\top \cov_y w ~:~ \cov_y \in \PD^d,~\| \log (\covsa_y^{-\frac{1}{2}} \cov_y \covsa_y^{-\frac{1}{2}}) \|_F \le \rho_y
    \right\}.
\]
Using the transformation $Z_y \leftarrow \covsa_y^{-\frac{1}{2}} \cov_y \covsa_y^{-\frac{1}{2}}$,  we have 
\[
        \tau_y^{\FR}(w)^2  = \max \left\{ v^\top Z_y v~:~ \| \log Z_y \|_F \le \rho_y \right\} 
\]
with $v = \covsa_y^\half w$. We now show that the above optimization problem admits the maximizer
    \[
        Z_y\opt = UU^\top + \exp(\rho_y) \frac{ v v^\top }{ \| v \|_2^2 }, 
    \]
    where $U$ is an $d \times (d-1)$ orthonormal matrix whose columns are orthogonal to $v$. First, by~\cite[Lemma C.1]{ref:nguyen2019calculating}, the feasible region is compact. Since the objective function $v^\top Z_y v$ is continuous in $Z_y$, an optimal solution $Z_y\opt$ exists. Next, we first claim that the constraint holds with equality at optimality.\footnote{
    Alternatively, one can also prove this by using the theory of geodesic convexity. Indeed, it is well known that the set of positive definite matrices equipped with the Fisher-Rao distance $\mathbb{F}$ is a Hadamard manifold, see~\cite{yue2024geometric} for example. Although the set $\mc S =\{ \cov_y\in \PD^d: \| \log Z_y \|_F \le \rho_y \}$ is not convex in the Euclidean sense, it is a geodesically convex subset with respect to the Fisher-Rao distance~\cite{nguyen2019optimistic}. By~\cite[Lemma 13]{yue2024geometric}, the objective function $\cov_y \mapsto w^\top \cov_y w$ is geodesically convex. The Krein-Milman theorem for Hadamard manifolds~\cite{niculescu2007krein} then implies that the maximum is attained at the extreme points of $\mc S$~\cite{niculescu2007krein}. By the NPC inequality (see~\cite[Definition~1]{niculescu2007krein}), the extreme points of $\mc S$ are precisely $\mc S$ is $\{\cov_y\in \PD^d: \| \log Z_y \|_F = \rho_y \}$.
    Therefore, the maximizer $\cov^\star_y$ satisfies that $\| \log Z_y^\star \|_F = \rho_y$.} 
    Suppose that $ \| \log Z_y\opt \|_F < \rho_y$. Then, for some small $\delta >0$, the matrix $Z_y\opt + \delta\, vv^\top$ is feasible due to the continuity of the constraint function $\| \log Z_y \|_F$ and has a strictly better objective value than the optimal solution $Z_y\opt$. This violates the optimality of $Z_y\opt$. Hence, $ \| \log Z_y\opt \|_F = \rho_y $ for any optimal solution $Z_y\opt$, and the problem becomes
    \[
        \tau_y^{\FR}(w)^2  = \max \left\{ v^\top Z_y v~:~ \| \log Z_y \|_F = \rho_y \right\} ,
    \]
    which by eigenvalue decomposition $ Z_y = Q\text{Diag}(\lambda) Q^\top$, is equivalent to
    \[ 
        \begin{array}{cl}
        \max & v^\top Q \text{Diag}(\lambda) Q^\top v \\
        \st &  \sum_{i=1}^d (\log \lambda_i)^2 = \rho_y^2, \\
        & \lambda_1 \ge \cdots\ge \lambda_d >0,\ Q\in \mathcal{O}(d),
        \end{array} 
    \]
    where $\mathcal{O}(d)$ is the set of $d\times d$ orthogonal matrices.
    The objective function admits an upper bound $v^\top Q \text{Diag}(\lambda) Q^\top v\le \lambda_1 \|v\|_2^2$, 
    which can be attained by setting
    \begin{equation}
        \label{eq:proof3}
        Q = \left(\frac{v}{\|v\|_2}, \ U \right)\in \mathcal{O}(d).
    \end{equation}
    Hence, the optimal $Q$ is of the form~\eqref{eq:proof3} and the optimization problem is further equivalent to
    \[ 
        \begin{array}{cl}
        \max & \lambda_1 \|v\|_2^2 \\
        \st &  \sum_{i=1}^d (\log \lambda_i)^2 = \rho_y^2, \\
        & \lambda_1 \ge \cdots\ge \lambda_d >0.
        \end{array} 
    \]
    Since the objective function in the last display optimization problem depends only on $\lambda_1$ and is strictly increasing on $\lambda_1$, we thus want to have $\log \lambda_2 = \cdots = \log \lambda_d = 0$ and $\log \lambda_1$ to be as big as possible. The optimal $\lambda\in\R^d_{++}$ must satisfy $\lambda_2 = \cdots = \lambda_d = 1$ and $(\log \lambda_1)^2 = \rho_y^2$. Since $\lambda_1 \ge \lambda_2 = 1$, we have $\log \lambda_1 = \rho_y$ and hence $\lambda_1 = \exp(\rho_y)$. In other words,
    \begin{align*}
        Z_y\opt& = Q\text{Diag}(\lambda) Q^\top = \left(\frac{v}{\|v\|_2}, \ U \right)
        \begin{pmatrix}
        \exp(\rho_y) & & &\\
        & 1 &&\\
        && \ddots & \\
        && & 1
        \end{pmatrix}
        \left(\frac{v}{\|v\|_2}, \ U \right)^\top = UU^\top + \exp(\rho_y) \frac{ v v^\top }{ \| v \|_2^2 }.
    \end{align*}
    The corresponding optimal value is 
    \[\tau_y^{\FR}(w)^2 = v^\top Z_y\opt v = \exp(\rho_y)\, \|v \|_2^2  = \exp(\rho_y)\, w^\top \covsa_y w. \]
    This completes the proof.
\end{proof}

We are now ready to prove Theorem~\ref{thm:fr}.
\begin{proof}[Proof of Theorem~\ref{thm:fr}]
    Using the Fisher-Rao divergence $\mathds F$, the optimization problem
    \[
    \Min{w \in \mc W}~ \sum\nolimits_{y \in \mc Y} \tau_y^{\mathds F}(w)
    \]
    becomes problem~\eqref{eq:FR} by exploiting the analytical form of $\tau_y^{\mathds F}(w)$ in Proposition~\ref{prop:FR}. By invoking Proposition~\ref{prop:refor}, we obtain the postulated results on the optimal solution $\theta^{\mathds F}$ for the case of the Fisher-Rao divergence.
\end{proof}

\subsection{Proof of Section~\ref{sec:logdet}}

\begin{proof}[Proof of Proposition~\ref{prop:logdet}]
    By~\cite[Proposition~3.4]{ref:le2021adversarial}, we have
    \[
        \tau_y^{\mathds D}(w)^2 = \Inf{\substack{\gamma > 0 \\ \gamma \covsa_y^{-1} \succ ww^\top }}~ \gamma \rho_y - \gamma \log \det (I - \covsa_y^\half ww^\top \covsa_y^\half /\gamma).
    \]
    Using the matrix determinant formula~\cite{ref:bernstein2009matrix}, we have
    \[
        \det (I - \covsa_y^\half ww^\top \covsa_y^\half /\gamma) = (1 - w^\top \covsa_y w / \gamma).
    \]
    Notice that the constraint $\gamma \covsa_y^{-1} \succ ww^\top$ is equivalent to $\gamma > w^\top \covsa_y w$. Thus, the optimization problem is simplified to
    \[
        \tau_y^{\mathds D}(w)^2 =\Inf{\gamma > w^\top \covsa_y w}~ \gamma \rho_y - \gamma \log (1 - w^\top \covsa_y w / \gamma).
    \]
    The derivative of the objective function is
    \[
    \rho - \log \big(1 - \frac{w^\top \covsa_y w }{ \gamma} \big) - \frac{w^\top \covsa_y w} {\gamma - w^\top \covsa_y w},
    \]
    which is increasing for $\gamma > w^\top \covsa_y w$. The optimization problem is, therefore, a convex optimization. Noting that the objective value tends to $+\infty$ as $\gamma \to w^\top \covsa_y w$, the optimal solution must be given by the first-order optimality condition:
    \[
        \rho - \log \big(1 - \frac{w^\top \covsa_y w }{ \gamma} \big) - \frac{w^\top \covsa_y w} {\gamma - w^\top \covsa_y w} = 0,
    \]
    Solving this equation yields the optimal solution
    \[
    \gamma\opt = \frac{w^\top \covsa_y w}{1 + 1/W_{-1}(-\exp(-\rho_y-1))}.
    \]
    Replacing the value of $\gamma\opt$ into the objective function leads to the desired result.
\end{proof}

\subsection{Proof of Section~\ref{sec:asymptotic}}

\begin{proof}[Proof of Proposition~\ref{prop:asymptotic}]
First, consider the case that $\varphi$ is the Quadratic distance. Because the objective function of the problem~\eqref {eq:quad} is strictly convex and coercive in $w$, it has a unique optimal solution, and this solution coincides with the optimal solution $w\opt(\lambda)$ of the following second-order cone program
\[
    \Min{w \in \mc W}~ \ds \sqrt{\lambda w^\top \covsa_y w + w^\top w} + \sqrt{\lambda w^\top \covsa_{-y} w + \lambda \sqrt{\rho_{-y}} w^\top w},
\]
where $\lambda = 1/\sqrt{\rho_{y}}$. By a compactification of $\mc W$ and applying Berge's maximum theorem~\cite[pp.~115-116]{ref:berge1963topological}, the function $w\opt(\lambda)$ is continuous on a non-negative compact range of $\lambda$, and converges to $w\opt(0)$ as $\lambda \to 0$. The optimal solution $w\opt(0)$ coincides with the solution of
\be \label{eq:asymptotic}
    \Min{w \in \mc W}~\|w\|_2,
\ee
which is the Euclidean projection of the origin onto the hyperplane $\mc W$. 
Recall that the set of feasible slope is $\mc W = \{ w \in \R^d \backslash\{0\} : \sum_{y \in \mc Y} y w^\top \msa_y  = 1\}$.
An elementary argument via optimality condition confirms that
\[
    w\opt(0) = \frac{\sum_{y \in \mc Y} y \msa_y}{ \| \sum_{y \in \mc Y} y \msa_y\|_2^2 }.
\]
Letting $w_{\infty, y}^{\mathds Q} = w\opt(0)$, we have the asymptotic slope of the Quadratic surrogate. Since $\rho_{-y}$ remains constant and $\rho_y \to \infty$,
\begin{align*}
    \kappa \tau_{y}(w) &= \frac{\sqrt{w^\top \covsa_y w + \sqrt{\rho_y} \|w\|_2^2}}{\sqrt{w^\top \covsa_y w + \sqrt{\rho_y} \|w\|_2^2} + \sqrt{w^\top \covsa_{-y} w + \sqrt{\rho_{-y}} \|w\|_2^2}}  \\
    &= \frac{\sqrt{ \frac{w^\top \covsa_y w}{\sqrt{\rho_y}} +  \|w\|_2^2}}{\sqrt{ \frac{w^\top \covsa_y w}{\sqrt{\rho_y}} +  \|w\|_2^2} + \sqrt{ \frac{w^\top \covsa_{-y} w}{\sqrt{\rho_y}} + \frac{\sqrt{\rho_{-y}}}{\sqrt{\rho_y}} \|w\|_2^2}}  \to \frac{\|w\|_2}{\|w\|_2} = 1.
\end{align*}
By Proposition~\ref{prop:refor}, we have $b^\varphi = (w^\varphi)^\top \msa_{y} - y \kappa^\varphi \tau_{y}^\varphi(w^\varphi)$ for any $y\in \mc Y$. Therefore, $b_{\rho_y} \to b_{\infty, y} = w_{\infty, y}^\top \msa_{y} - y$.  

In case $\varphi$ is the Bures distance, the optimal solution of the problem~\eqref{eq:dro} coincides with the problem~\eqref{eq:bures}, which, by the same argument as for the Quadratic distance, admits a unique optimal solution $w\opt(\lambda)$ that coincides with the optimal solution of the second-order cone program
\[
    \Min{w \in \mc W}~ \ds \lambda \sum_{y \in \mc Y}\sqrt{w^\top \covsa_y w} + \|w\|_2,
\]
where $\lambda = 1/(\sum_{y \in \mc Y} \rho_y)$. The optimal solution $w\opt(0)$ also coincides with the solution of
\[ 
    \Min{w \in \mc W}~\|w\|_2.
\]
Hence, the Quadratic and Bures surrogates are asymptotically equivalent when one radius grows to infinity while the other is fixed.

Consider the case that $\varphi$ is Fisher-Rao or LogDet divergence. As the objective functions of the problem~\eqref{eq:FR} and~\eqref{eq:logdet} are both strictly convex and coercive, it has a unique solution. Thus, its optimal solution coincides with the optimal solution $w\opt(\lambda)$ of the following second-order cone program
\[
    \Min{w \in \mc W}~ \ds  \sqrt{w^\top \covsa_y w} + \lambda \sqrt{w^\top \covsa_{-y} w},
\]
where $\lambda = \exp \big( \frac{\rho_{-y} - \rho_{y}}{2} \big)$ if $\varphi$ is the Fisher-Rao distance and $\lambda = \sqrt{\frac{W_{-1}(-\exp(-\rho_{-y}-1))}{W_{-1}(-\exp(-\rho_y-1))}}$ if $\varphi$ is the LogDet distance. By a compactification of $\mc W$ and applying Berge's maximum theorem~\cite[pp.~115-116]{ref:berge1963topological}, the function $w\opt(\lambda)$ is continuous on a non-negative compact range of $\lambda$, and converges to $w\opt(0)$ as $\lambda \to 0$. The optimal solution $w\opt(0)$ coincides with the solution of
\[
    \Min{w \in \mc W}~\sqrt{w^\top \covsa_y w}.
\]
Because the square-root function is monotonically increasing, $w\opt(0)$ also solves
\[
    \Min{w \in \mc W}~w^\top \covsa_y w,
\]
which is a convex, quadratic program with a single linear constraint. Then a convex optimization argument implies
\[
w\opt(0) = \frac{1}{a^\top \covsa_y^{-1} a} \covsa_y^{-1} a,
\]
where $a = \sum_{y \in \mc Y} y \msa_{y}$ is as defined in the statement.
Thus, the asymptotic slope $w_{\infty, y}$ of the Fisher-Rao and LogDet surrogates converges to $w^\star(0)$. 

Using a similar calculation as in the case of Quadratic asymptotic surrogate, we observe $\kappa \tau_y(w) \to 1$ when $\rho_{-y}$ remains constant and $\rho_y \to \infty$. Consequently, the intercept $b_{\rho_y}$ also tends towards $b_{\infty, y} = w_{\infty, y}^\top \msa_{y} - y$. The asymptotic hyperplane defined by ($w_{\infty, y}, b_{\infty, y})$ is then characterized by the linear equation $w_{\infty, y}^\top x - w_{\infty, y}^\top \msa_{y} + y = 0$. This equation identifies a hyperplane passing through $\msa_{-y}$ as $\sum_{y \in \mathcal{Y}} y w^\top \msa_y = 1$.
\end{proof}

Before proving Proposition~\ref{prop:tradeoff}, we present Lemma~\ref{lemma:maha-proj} that computes the Mahalanobis distance from a vector to a set specified by a hyperplane. A short proof is provided for completeness.

\begin{lemma}[Projection distance]\label{lemma:maha-proj}
    Given a positive definite matrix $\covsa \in \PD^d$ and $(w, b) \in \R^{d+1}$ such that $w \neq 0$, the Mahalanobis distance from a vector $\msa \in \R^d$ to the set $\{x \in \R^d: w^\top x - b \geq 0\}$ is
    \[
        \min\left\{\sqrt{(\msa - x)^\top \covsa^{-1} (\msa - x)}: x \in \R^d,~w^\top x - b \geq 0 \right\}
        = \begin{dcases}
            \frac{|w^\top \msa - b |}{\sqrt{w^\top \covsa w}} & ~\text{if}~ w^\top \msa - b < 0, \\
            0 & ~\text{otherwise}.
        \end{dcases}
    \]
\end{lemma}
\begin{proof}[Proof of Lemma~\ref{lemma:maha-proj}]
    If $w^\top \msa - b \geq 0$, then clearly $x = \msa$ is the optimal solution to the minimization problem, and the optimal value is $0$. It suffices now to consider the case when $\w^\top \msa - b < 0$. Using a transformation $x' \leftarrow \covsa^{-\half}(\msa - x)$, $w'\leftarrow \covsa^{\half} w$, and $b' \leftarrow b - w^\top \msa$, we find
    \begin{align*}
        \min\left\{(\msa - x)^\top \covsa^{-1} (\msa - x): x \in \R^d,~w^\top x - b \geq 0 \right\} &= \min \left\{x'^\top x'~:~x' \in \R^d,~w'^\top x' - b' \geq 0\right\} = \frac{(b')^2}{\| w' \|_2^2},
    \end{align*}
        where the last equality follows from the geometric fact that the distance from the origin to a hyperplane $u^\top x - t=0$ defined by the unit-length normal vector $u$ and intercept $t$ is precisely $t$.
\end{proof}

We are now ready to prove Proposition~\ref{prop:tradeoff}.

\begin{proof}[Proof of Proposition~\ref{prop:tradeoff}]
    From the definitions of $\mathrm{Co}$ and Lemma~\ref{lemma:maha-proj}, we have
    \[
        \mathrm{Co}_{\covsa_{+1}}(\theta)
        = \begin{dcases}
            \frac{|w^\top \msa_{+1} - b |}{\sqrt{w^\top \covsa_{+1} w}} & ~\text{if}~ w^\top \msa_{+1} - b > 0, \\
            0 & ~\text{otherwise}.
        \end{dcases}
    \]
    As both $\theta_\rho = (w_\rho, b_\rho)$ and $\theta_{\rho'} = (w_{\rho'}, b_{\rho'})$ are the optimal solutions of the problem~\eqref{eq:dro}, we can deduce that the hyperplanes induced by $\theta_\rho$ and $\theta_{\rho'}$ classify $\msa_{+1}$ correctly, thus eliminating the case $\mathrm{Co}(\theta) = 0$, which would never be the optimal value. So, at the optimal $\theta = (w,b)$, we have
    \begin{equation}
        \label{eq:proof7}
        \mathrm{Co}_{\covsa_{+1}}(\theta) = \frac{|w^\top \msa_{+1} - b |}{\sqrt{w^\top \covsa_{+1} w}} .
    \end{equation}
    Similarly, at the optimal $\theta = (w,b)$, we have 
    \begin{equation}
        \label{eq:proof8}
         \mathrm{Va}_{\covsa_{-1}}(\theta) = \frac{|w^\top \msa_{-1} - b |}{\sqrt{w^\top \covsa_{-1} w}}.
    \end{equation}
    
    We will only prove~\ref{prop:tradeoff-i}, as \ref{prop:tradeoff-ii} can be proved almost verbatim.
    We first prove the coverage inequality in~\ref{prop:tradeoff-i}. Let $c(\rho_y) = \exp(\frac{\rho_y}{2})$ for the Fisher-Rao divergence and $c(\rho_y) = \sqrt{- W_{-1} (-\exp(-\rho_y - 1)) }$ for the LogDet divergence. Note that in both cases, $c(\rho_y)$ is a strictly increasing function on $[0,+\infty)$.
    Since $\rho_{+1} = \rho'_{+1} = 0$, 
    \begin{equation}
        \begin{split}\label{eq:proof4}
            \tau_{\rho', +1} (w) &= \Max{\cov_{+1} \in \PSD^d: \varphi(\cov_{+1} \parallel \covsa_{+1}) \le \rho'_{+1}} \sqrt{w^\top \cov_{+1} w} = \sqrt{w^\top \covsa_{+1} w} \\
            & = \Max{\cov_{+1} \in \PSD^d: \varphi(\cov_{+1} \parallel \covsa_{+1}) \le \rho_{+1}} \sqrt{w^\top \cov_y w} = \tau_{\rho, {+1}} (w) \quad \forall w \in \mc W.
        \end{split}
    \end{equation}
    Also, 
    \begin{equation}
        \label{eq:proof5}
        \begin{split}
            \tau_{\rho', -1} (w) &= \Max{\cov_{-1} \in \PSD^d: \varphi(\cov_{-1} \parallel \covsa_{-1}) \le \rho'_{-1}} \sqrt{w^\top \cov_{-1} w} = c(\rho'_{-1})\, \sqrt{w^\top \covsa_{-1} w} \\
            & < c(\rho_{-1})\, \sqrt{w^\top \covsa_{-1} w} = \Max{\cov_{-1} \in \PSD^d: \varphi(\cov_{-1} \parallel \covsa_{-1}) \le \rho_{-1}} \sqrt{w^\top \cov_{-1} w} = \tau_{\rho, {-1}} (w) \quad \forall w \in \mc W,
        \end{split}
    \end{equation}
    where the expressions for the maximum values follow from Proposition~\ref{prop:FR} and the inequality follows from the strict monotonicity of $c(\rho_y)$ and the fact that $\rho_{-1} > \rho'_{-1}$.
    Therefore,
    \begin{equation}
        \label{eq:proof6}
        (\kappa_{\rho'})^{-1} = \sum_{y \in \mc Y} \tau_{\rho', y}(w_{\rho'}) \leq \sum_{y \in \mc Y} \tau_{\rho', y}(w_\rho) < \sum_{y \in \mc Y} \tau_{\rho, y}(w_\rho) = \kappa_\rho^{-1},
    \end{equation}
    where first inequality follows from the fact that $(w_{\rho'}, b_{\rho'})$ is the optimum solution for problem~\eqref{eq:FR} with the radii $\rho'$ and the second from~\eqref{eq:proof4} and~\eqref{eq:proof5}. Hence, $\kappa_{\rho} < \kappa_{\rho'}$. 
    By Proposition~\ref{prop:refor}, at the optimal $\theta = (w, b)$, we have that $b = w^\top \msa_{+1} - \kappa \,\tau_{+1} (w)  = w^\top\msa_{-1} + \kappa \,\tau_{-1} (w)$. Therefore, for any $y\in \mc Y$,
    \begin{equation*}
        |w^\top \msa_{y} - b| = |w^\top \msa_{y} - w^\top \msa_{y} + y \kappa \tau_y (w)| = \kappa \tau_y (w),
    \end{equation*}
    which implies that
    \[
        \kappa = \frac{|w^\top \msa_{y} - b|}{\Max{\cov_y \in \PSD^d: \varphi(\cov_y \parallel \covsa_y) \le \rho_y}~\sqrt{  w^\top \cov_y w}} = \Min{\cov_y \in \PSD^d: \varphi(\cov_y \parallel \covsa_y) \le \rho_y}~\frac{|w^\top \msa_{y} - b|}{\sqrt{w^\top \cov_y w}}.
    \]
    Since $\rho_{+1} = \rho'_{+1} = 0$, inequality~\eqref{eq:proof6} and expression~\eqref{eq:proof7} together imply that
    \[
        \mathrm{Co}_{\covsa_{+1}}(\theta_\rho) = \frac{|w_\rho^\top \msa_{+1} - b_\rho|}{\sqrt{w_\rho^\top \covsa_{+1} w_\rho}} < \frac{|w_{\rho'}^\top \msa_{+1} - b_{\rho'}|}{\sqrt{w_{\rho'}^\top \covsa_{+1} w_{\rho'}}} = \mathrm{Co}_{\covsa_{+1}}(\theta_{\rho'}).
    \]
    Using the same arguments, we can prove the \emph{validity} inequality in~\ref{prop:tradeoff-ii}.
    
    We next show the validity inequality in~\ref{prop:tradeoff-i}. By Theorems~\ref{thm:fr} and~\ref{thm:logdet}, at the optimal solution $\theta = (w,b)$,
    \[
        \kappa = \left( \sum_{y \in \mc Y} c(\rho_y) \sqrt{ w^\top \covsa_y w}\right)^{-1} \quad\text{and}\quad b = w^\top \msa_{-1} + \kappa \, c(\rho_{-1}) \sqrt{ w^\top \covsa_{-1} w},
    \]
    Combining with the expression~\eqref{eq:proof8} of $\mathrm{Va}$, we have
    \be \label{eq:ro_dist}
        \mathrm{Va}_{\covsa_{-1}}(\theta_\rho) = \frac{|w^\top \msa_{-1} - b |}{\sqrt{w^\top \covsa_{-1} w}} = \frac{\kappa c(\rho_{-1}) \sqrt{ w^\top \covsa_{-1} w}}{\sqrt{w^\top \covsa_{-1} w}} = c(\rho_{-1}) \kappa.
    \ee
    Furthermore, for both FR and LogDet divergences
    \begin{align*}
        c(\rho_{-1}) \kappa_\rho  &= \frac{c(\rho_{-1})}{ c(\rho_{-1}) \sqrt{w_\rho^\top \covsa_{-1} w_\rho} + c(\rho_{+1}) \sqrt{w_\rho^\top \covsa_{+1} w_\rho}} \\
        &\geq \frac{c(\rho_{-1})}{c(\rho_{-1}) \sqrt{w_{\rho'}^\top \covsa_{-1} w_{\rho'}} + c(\rho_{+1}) \sqrt{w_{\rho'}^\top \covsa_{+1} w_{\rho'}}} \\
        &= \frac{c(\rho'_{-1}) c(\rho_{-1})}{c(\rho'_{-1}) c(\rho_{-1}) \sqrt{w_{\rho'}^\top \covsa_{-1} w_{\rho'}} + c(\rho'_{-1})c(\rho_{+1}) \sqrt{w_{\rho'}^\top \covsa_{+1} w_{\rho'}}} \\
        &> \frac{c(\rho'_{-1})c(\rho_{-1})}{c(\rho'_{-1})c(\rho_{-1}) \sqrt{w_{\rho'}^\top \covsa_{-1} w_{\rho'}} + c(\rho_{-1})c(\rho'_{+1}) \sqrt{w_{\rho'}^\top \covsa_{+1} w_{\rho'}}} \\
        &= \frac{c(\rho'_{-1})}{c(\rho'_{-1}) \sqrt{w_{\rho'}^\top \covsa_{-1} w_{\rho'}} + c(\rho'_{+1}) \sqrt{w_{\rho'}^\top \covsa_{+1} w_{\rho'}}} = c(\rho'_{-1})\kappa_{\rho'},
    \end{align*}
    where the first inequality follows from the fact that $(w_\rho, b_\rho)$ is the optimal solution for the problem~\eqref{eq:FR} or problem~\eqref{eq:logdet} with parameter set $\rho$ and the second from the strict monotonicity of $c(\rho_y)$ and the fact that $\rho_{+1} = \rho'_{+1} = 0$ and $\rho_{-1} > \rho'_{-1}$. Combining the last display inequality with~\eqref{eq:ro_dist}, we thus have
    \[
        \mathrm{Va}_{\covsa_{-1}}(\theta_\rho) > \mathrm{Va}_{\covsa_{-1}}(\theta_{\rho'}).
    \]
    Using the same arguments, we can prove the \emph{coverage} inequality in~\ref{prop:tradeoff-ii}. This completes the proof.
\end{proof}

\section{Surrogates with Mean Ambiguity}

In this section, we will argue that under two very natural assumptions, further robustification with respect to the mean $\m_y$ on top of the covariance robustification will not affect the final recourse generated by our framework. To begin, we consider the CVAS problem with both mean  and covariance uncertainty:
\begin{equation}\label{opt:mean-cov-robust}
    \Max{\theta \in \Theta}~\Min{\cov_y \in \PSD^d: \varphi(\cov_y \parallel \covsa_y) \le \rho_y~\forall y}~\Min{\m_y \in \R^d: (\m_y - \msa_y)^\top \cov_y^{-1} (\m_y - \msa_y) \le \nu_y^2}~\min\left\{~\mathrm{Co}_{\cov_{+1}}(\theta), \mathrm{Va}_{\cov_{-1}}(\theta)\right\}.
\end{equation}
Similarly to the $\mbb U_y^{\varphi} (\Pnom_y)$ in the main paper,  for each class $y \in \mc Y$, we define ambiguity set
\begin{equation*}
    \mc U_{y}^\varphi(\Pnom_y) = \{ 
            \PP_y :~\PP_y \sim (\m_y, \cov_y) \text{ for some }\cov_y \in \PSD^d \text{ with }\varphi(\cov_y \parallel \covsa_y) \le \rho_y,~(\m_y - \msa_y)^\top \cov_y^{-1} (\m_y - \msa_y) \le \nu_y^2
        \}.
\end{equation*}
Compared with the ambiguity sets we studied in the main paper, these ambiguity sets have an additional variable $\m_y$ that is constrained to reside in an ellipsoid defined by the covariance $\cov_y$ and the radius parameter $\nu_y \ge 0$. We assume that there exists some $\theta = (w,b) \in \Theta$ such that
\begin{equation}\label{assu:small_radii}
    \nu_y < \frac{|w^\top \msa_y - b|}{\sqrt{\displaystyle\max_{ \varphi(\cov_y \parallel \covsa_y) \le \rho_y }  w^\top \cov_y w}} = \frac{|w^\top \msa_y - b|}{\tau^{\varphi}_y(w)} \quad\forall y\in \mc Y.
\end{equation}
Denote by $\mathrm{dist}_\cov$ the Mahalanobis distance induced by a positive definite matrix $\cov$. Then, by Lemma~\ref{lemma:maha-proj}, \eqref{assu:small_radii} is equivalent to 
\[
\nu_y < \Min{ \varphi(\cov_y \parallel \covsa_y) \le \rho_y } \mathrm{dist}_{\cov_y} ( \msa_y, \{x:w^\top x = b\} ) \quad\forall y\in\mc Y.
\]
Therefore, geometrically, this assumption requires that the two mean uncertainty ellipsoids do not overlap (even under the worst-case covariance matrices) so that they can be separated by at least one hyperplane.
If such an assumption does not hold, as we will see below, the optimal value of problem~\eqref{eq:dro} is always $0$, which is the uninteresting case since the objective value is non-negative.

With the mean-covariance ambiguity set $\mc U_{y}^\varphi(\Pnom_y)$, following a similar derivation as in the proof of Proposition~\ref{prop:refor}, we can show that the mean-covariance robust variant~\eqref{opt:mean-cov-robust} is equivalent to
\begin{align*}
    &\, \Min{\theta \in \Theta}~\Max{y \in \mc Y}~\Max{\PP_y \in \mc U_{y}^\varphi(\Pnom_y)}~\PP_y(\mc C_\theta(X) \neq y) \\
    =&\,  \Min{w\neq 0,\, b}~\Max{y\in \mc Y}~\Max{\cov_y \in \PSD^d: \varphi(\cov_y \parallel \covsa_y) \le \rho_y}~\Max{\m_y \in \R^d: (\m_y - \msa_y)^\top \cov_y^{-1} (\m_y - \msa_y) \le \nu_y^2}~\left( 1 + \frac{(b - w^\top \m_y )^2}{w^\top \cov_y w} \right)^{-1}. 
\end{align*}
By Lemma~\ref{lemma:quad} presented below and the assumption~\eqref{assu:small_radii} about the radii $\nu_y$, we have that 
\begin{align*}
    \Max{\m_y \in \R^d: (\m_y - \msa_y)^\top \cov_y^{-1} (\m_y - \msa_y) \le \nu_y^2}~\left( 1 + \frac{(b - w^\top \m_y )^2}{w^\top \cov_y w} \right)^{-1} = \left( 1 + \left(\frac{|b - w^\top \msa_y| }{\sqrt{w^\top \cov_y w}} - \nu_y \right)^2\right)^{-1}.
\end{align*}
So, mean-covariance robust variant~\eqref{opt:mean-cov-robust} is further equivalent to
\begin{align*}
    &\, \Min{w\neq 0,\, b}~\Max{y\in \mc Y}~\Max{\cov_y \in \PSD^d: \varphi(\cov_y \parallel \covsa_y) \le \rho_y}~\left( 1 + \left(\frac{|b - w^\top \msa_y| }{\sqrt{w^\top \cov_y w}} - \nu_y \right)^2\right)^{-1} \\
    = &\, \Min{w\neq 0,\, b}~\Max{y\in \mc Y}~\left( 1 + \left(\frac{|b - w^\top \msa_y| }{ \tau^\varphi_y(w) } - \nu_y \right)^2 \right)^{-1} ,
\end{align*}
where we have again used assumption~\eqref{assu:small_radii}. If we choose the same radii for both classes (\ie, $\nu_y = \nu > 0$ for all $y\in\mc Y$), which is natural when there is no additional information about the mean uncertainty, then the minimax problem in the last display becomes 
\begin{align*}
    \Min{w\neq 0,\, b}~\Max{y\in \mc Y}~\left( 1 + \left(\frac{|b - w^\top \msa_y| }{ \tau^\varphi_y(w) } - \nu_y \right)^2 \right)^{-1} &\, = \left( 1 + \left(\Min{w\neq 0,\, b}~\Max{y\in \mc Y}~ \left\lbrace\frac{|b - w^\top \msa_y| }{ \tau^\varphi_y(w) } - \nu_y \right\rbrace\right)^2 \right)^{-1} \\
    &\, = \left( 1 + \left(\Min{w\neq 0,\, b}~\Max{y\in \mc Y}~ \left\lbrace\frac{|b - w^\top \msa_y| }{ \tau^\varphi_y(w) }  \right\rbrace - \nu \right)^2 \right)^{-1},
\end{align*}
leading us back to the same problem as in Proposition~\ref{prop:refor} (see~\eqref{eq:proof10}). Therefore, we conclude that under the two natural assumptions~\eqref{assu:small_radii} and $\nu_{+1} = \nu_{-1}$, further robustifying with respect to the mean will not affect the surrogate. Consequentially, it will not affect the recourse generation.

The following lemma studies the optimization problem arising from the mean robustification and is used in our analysis above.

\begin{lemma}[Optimal mean] \label{lemma:quad}
    Fix any $(w, b) \in \R^{d+1}$, $w \neq 0$, $\msa \in \R^d$, $\cov \in \PD^d$ and $\nu \in \R_+$. If $|b - w^\top \msa| \le \nu \sqrt{w^\top \cov w}$, we have that
    \[
    \Min{\m \in \R^d: (\m - \msa)^\top \cov^{-1} (\m - \msa) \le \nu^2}~ (b - w^\top \m)^2 = 0,
    \]
    and that the minimum is attained at 
    \[
        \m\opt =  \frac{(b - w^\top \msa )}{w^\top \cov w} \cov w + \msa. 
    \]
    If $|b - w^\top \msa| > \nu \sqrt{w^\top \cov w}$, we have that
    \[
    \Min{\m \in \R^d: (\m - \msa)^\top \cov^{-1} (\m - \msa) \le \nu^2}~ (b - w^\top \m)^2 = (|b - w^\top \msa| - \nu \sqrt{w^\top \cov w})^2,
    \]
    and that the minimum is attained uniquely at 
    \[
        \m\opt  = \frac{\sign(b- w^\top \msa) \nu}{\sqrt{w^\top \cov w}} \cov w + \msa.
    \]
\end{lemma}
\begin{proof}[Proof of Lemma~\ref{lemma:quad}] In the following, given any positive definite matrix $\cov$, we will use $\mathrm{dist}_\cov$ to denote the Mahalanobis distance induced by $\cov$. We first consider the case where $|b - w^\top \msa| \le \nu \sqrt{w^\top \cov w}$. By Lemma~\ref{lemma:maha-proj}, the minimum can be written as
\begin{equation*}
    \Min{\m \in \R^d: (\m - \msa)^\top \cov^{-1} (\m - \msa) \le \nu^2}~ (b - w^\top \m)^2 = \sqrt{w^\top \cov w} \Min{\mathrm{dist}_\cov(\m, \msa) \le \nu } \mathrm{dist}_\cov^2 (\mu, \{x: w^\top x = b\}).
\end{equation*}
Therefore, the minimization aims at finding the $\mu$ within the ellipsoid $\{\mu: \mathrm{dist}_\cov(\m, \msa) \le \nu \}$ that has the smallest distance to the hyperplane $\{x: w^\top x = b\}$.
On the other hand, the condition $|b - w^\top \msa| \le \nu \sqrt{w^\top \cov w}$ can be translated into $ \mathrm{dist}_\cov (\msa, \{x: w^\top x = b\}) \le \nu$, implying that the ellipsoid $\{\mu: \mathrm{dist}_\cov(\m, \msa) \le \nu \}$ has a non-empty intersection with the hyperplane $\{x: w^\top \mu = b\}$. Therefore, the minimum value in this case is $0$, and the minimum is attained at
\[
\m\opt =  \frac{(b - w^\top \msa )}{w^\top \cov w} \cov w + \msa. 
\]
Indeed, it can be easily checked that such a $\m\opt$ lies in both the ellipsoid and the hyperplane.

We next consider the case where $|b - w^\top \msa| > \nu \sqrt{w^\top \cov w}$. From the above analysis for the case  $|b - w^\top \msa| \le \nu \sqrt{w^\top \cov w}$, we know that the optimal value must be positive. By the Lagrange duality,
\begin{align}
\Min{\m \in \R^d: (\m - \msa)^\top \cov^{-1} (\m - \msa) \le \nu^2}~ (b - w^\top \m)^2    =&\Min{\m \in \R^d}~\Max{\lambda \ge 0} ~ (b - w^\top \m)^2 + \lambda ( (\m - \msa)^\top \cov^{-1} (\m - \msa) - \nu^2) \notag\\
=&\Max{\lambda \ge 0}~\Min{\m \in \R^d}~ (b - w^\top \m)^2 + \lambda ( (\m - \msa)^\top \cov^{-1} (\m - \msa) - \nu^2) .\label{eq:QP_dual}
\end{align}
If $\lambda = 0$, the optimal solution for the inner minimization is $\m = b\, \frac{w}{\|w\|_2^2}$ and the optimal value is $0$, which is a contradiction. Therefore, $\lambda > 0$. 
Fix any $\lambda > 0$ and consider the inner minimization. The first-order optimality condition is
\[
2(w^{\top} \m\opt-b) w + 2 \lambda \cov^{-1}(\m\opt -\msa) = 0,
\]
solving which yields
\begin{equation}
    \label{eq:proof9}
    \begin{split}
        \m\opt &= (w w^\top + \lambda \cov^{-1})^{-1} (b w + \lambda \cov^{-1} \msa) \\
        &= \left( \frac{1}{\lambda} \cov - \frac{\frac{1}{\lambda^2} \cov w w^\top \cov}{1 + \frac{1}{\lambda} w^\top \cov w} \right)(b w + \lambda \cov^{-1} \msa) \\
        &= \frac{1}{\lambda} \left(\cov - \frac{\cov w w^\top \cov}{\lambda + w^\top \cov w}\right)(b w + \lambda \cov^{-1} \msa) \\
        &= \frac{1}{\lambda} \left(b \cov w - \frac{b w^\top \cov w}{\lambda + w^\top \cov w} \cov w + \lambda \msa - \frac{\lambda w^\top \msa}{\lambda + w^\top \cov w} \cov w\right) \\
        &= \frac{1}{\lambda} \left(b \cov w - \frac{b w^\top \cov w}{\lambda + w^\top \cov w} \cov w - \frac{\lambda w^\top \msa}{\lambda + w^\top \cov w} \cov w\right) + \msa \\
        &= \frac{b \lambda + b w^\top \cov w - b w^\top \cov w - \lambda w^\top \msa}{\lambda(\lambda + w^\top \cov w)} \cov w + \msa \\
        &= \frac{b - w^\top \msa}{\lambda + w^\top \cov w} \cov w + \msa,
    \end{split}
\end{equation}
where the second equality follows from the Sherman-Morrison formula.
Substituting the above equation into~\eqref{eq:QP_dual} and the dual problem becomes
\begin{align*}
&\Max{\lambda \ge 0} \Min{\m \in \R^d}~ (b - w^\top \m)^2 + \lambda ( (\m - \msa)^\top \cov^{-1} (\m - \msa) - \nu^2) \\
=&\Max{\lambda \ge 0}~ \left(b - \frac{b - w^\top \msa}{\lambda + w^\top \cov w} w^\top \cov w - w^\top \msa \right) ^ 2 + \lambda \left(\frac{b - w^\top \msa}{\lambda + w^\top \cov w} \right)^2 w^\top \cov w - \lambda \nu^2 \\
=&\Max{\lambda \ge 0}~ \left(\frac{\lambda(b - w^\top \msa)}{\lambda + w^\top \cov w}\right)^2 + \lambda \left(\frac{b - w^\top \msa}{\lambda + w^\top \cov w} \right)^2 w^\top \cov w - \lambda \nu^2 \\
=&\Max{\lambda \ge 0}~ \frac{\lambda(b - w^\top \msa)^2}{(\lambda + w^\top \cov w)^2} (\lambda + w^\top \cov w) - \lambda \nu^2 \\
=&\Max{\lambda \ge 0}~ \lambda \left(\frac{(b - w^\top \msa)^2}{\lambda + w^\top \cov w} - \nu^2 \right).
\end{align*}
The first-order condition with respect to $\lambda$ asserts that the optimizes $\lambda\opt$ satisfies
\[
\frac{(b - w^\top \msa)^2}{\lambda\opt + w^\top \cov w} - \nu^2 - \frac{\lambda\opt(b - w^\top \msa)^2}{(\lambda\opt + w^\top \cov w)^2}  = 0,
\]
which is equivalent to
\[
(\lambda\opt + w^\top \cov w)(b - w^\top \msa)^2 - \nu^2 (\lambda\opt + w^\top \cov w)^2 - \lambda\opt (b - w^\top \msa)^2 = 0.
\]
The last display equation is quadratic in $\lambda\opt$:
\[
w^\top \cov w (b - w^\top \msa)^2 - \nu^2 (w^\top \cov w)^2 - 2 \lambda\opt \nu^2 w^\top \cov w - \nu^2 {\lambda\opt}^2 = 0.
\]
Since $\lambda\opt > 0$, using the assumption that $|b - w^\top \msa| > \nu \sqrt{w^\top \cov w}$, we have
\begin{align*}
\lambda\opt &= \frac{2 \nu^2 w^\top \cov w \pm \sqrt{4 \nu^4 (w^\top \cov w)^2 + 4 \nu^4(w^\top \cov w (b - w^\top \msa)^2 - \nu^2 (w^\top \cov w)^2)}}{-2 \nu^2} \\
&= \frac{-\nu^2 w^\top \cov w \pm \sqrt{\nu^2 w^\top \cov w (b - w^\top \msa)^2}}{\nu^2} = -w^\top \cov w + \frac{|b - w^\top \msa| \sqrt{w^\top \cov w}}{\nu} .
\end{align*}
Using this expression for $\lambda\opt$ and \eqref{eq:proof9}, we get unique solution
\[
\m\opt = \frac{b - w^\top \msa}{\lambda\opt + w^\top \cov w} \cov w + \msa = \frac{\sign(b- w^\top \msa) \nu}{\sqrt{w^\top \cov w}} \cov w + \msa.
\]
Finally, the optimal value is given by
\begin{align*}
\Max{\lambda \ge 0}~ \lambda \left(\frac{(b - w^\top \msa)^2}{\lambda + w^\top \cov w} - \nu^2 \right) 
=& \left(\frac{|b - w^\top \msa| \sqrt{w^\top \cov w}}{\nu} - w^\top \cov w \right) \left(\frac{(b - w^\top \msa)^2}{\frac{|b - w^\top \msa| \sqrt{w^\top \cov w}}{\nu}} - \nu^2 \right) \\
=& (b - w^\top \msa)^2 - 2 \nu |b - w^\top \msa| \sqrt{w^\top \cov w} + \nu^2 w^\top \cov w\\
=& (|b - w^\top \msa| - \nu \sqrt{w^\top \cov w})^2. 
\end{align*}
This completes the proof.
\end{proof}

\end{document}